\documentclass[11pt]{report}
 \pdfoutput=1 
\usepackage{scrextend}
\usepackage[online]{suthesis-2e}

 \usepackage{algpseudocode}
 \usepackage{algorithm}
\usepackage{amsthm,amsmath}
\usepackage{amsfonts, amsopn, amssymb}
\usepackage[authoryear]{natbib}
\usepackage[colorlinks,citecolor=blue,urlcolor=blue]{hyperref}
\usepackage{color}
\usepackage{graphicx}
\usepackage{epstopdf}
\usepackage{ subcaption}
\usepackage{ bbold }

\numberwithin{equation}{section}
\theoremstyle{plain}
\newtheorem{theorem}{Theorem}[section]

\newtheorem{corollary}[theorem]{Corollary}
\newtheorem{definition}[theorem]{Definition}
\newtheorem{lemma}[theorem]{Lemma}
\newtheorem{proposition}[theorem]{Proposition}
\newtheorem{remark}[theorem]{Remark}
\newtheorem{conjecture}[theorem]{Conjecture}
\newcommand{\Pind}{P^\perp_{\Sigma,\eta}}

\newcommand{\cA}{\mathcal{A}}

\newcommand{\cH}{\mathcal{H}}

\newcommand{\cN}{\mathcal{N}}

\DeclareMathOperator{\unif}{Unif}

\DeclareMathOperator{\fwer}{FWER}

\renewcommand{\dagger}{+}

\newcommand{\R}{\mathbb{R}}
\newcommand{\V}{{\cal V}}
\newcommand{\Ee}{\mathbb{E}}
\newcommand{\Pp}{\mathbb{P}}

\newcommand{\E}{M}

\newcommand{\hbeta}{\hat{\beta}}

\def\bar{\overline}

\newcommand{\BALD}{\begin{aligned}}
\newcommand{\EALD}{\end{aligned}}
\newcommand{\BALDS}{\begin{aligned*}}
\newcommand{\EALDS}{\end{aligned*}}
\newcommand{\BCAS}{\begin{cases}}
\newcommand{\ECAS}{\end{cases}}
\newcommand{\BEAS}{\begin{eqnarray*}}
\newcommand{\EEAS}{\end{eqnarray*}}
\newcommand{\BEQ}{\begin{equation}}
\newcommand{\EEQ}{\end{equation}}
\newcommand{\BIT}{\begin{itemize}}
\newcommand{\EIT}{\end{itemize}}
\newcommand{\BMAT}{\begin{bmatrix}}
\newcommand{\EMAT}{\end{bmatrix}}
\newcommand{\BNUM}{\begin{enumerate}}
\newcommand{\ENUM}{\end{enumerate}}
\newcommand{\eg}{{\it e.g.}}

\newcommand{\ie}{{\it i.e.}}

\newcommand{\BA}{\begin{array}}
\newcommand{\EA}{\end{array}}

\newcommand{\ones}{\mathbf 1}
\newcommand{\zeros}{\mathbf 0}

\newcommand{\reals}{\mathbf{R}}

\DeclareMathOperator*{\argmax}{\arg\max}
\DeclareMathOperator*{\argmin}{\arg\min}
\DeclareMathOperator{\Expect}{\mathbf{E}}

\DeclareMathOperator*{\minimize}{minimize}

\DeclareMathOperator{\sign}{sign}

\newcommand{\pc}{\hspace{1pc}}

\newcommand{\norm}[1]{\left\| #1 \right\|}

\DeclareMathOperator{\diag}{diag}

\newcommand{\indicator}[1]{\mathbb{1}{\left[ {#1} \right] }}

\newcommand{\No}{\mathcal{N}}
\newcommand{\bs}{\backslash}

    \title{Selective Inference and Learning Mixed Graphical Models}
    \author{Jason Dean Lee}
    \dept{Computational Math and Engineering}
    \principaladviser{Trevor J. Hastie}
     \coprincipaladvisor{Jonathan E. Taylor}
    \firstreader{Lester Mackey}
   
\begin{document}


    \beforepreface

    \prefacesection{Abstract}
This thesis studies two problems in modern statistics. First, we study selective inference, or inference for hypothesis that are chosen after looking at the data. The motiving application is inference for regression coefficients selected by the lasso. We present the Condition-on-Selection method that allows for valid selective inference, and study its application to the lasso, and several other selection algorithms. 

In the second part, we consider the problem of learning the structure of a pairwise graphical model over continuous and discrete variables. We present a new pairwise model for graphical models with both continuous and discrete variables that is amenable to structure learning. In previous work, authors have considered structure learning of Gaussian graphical models and structure learning of discrete models. Our approach is a natural generalization of these two lines of work to the mixed case. The penalization scheme involves a novel symmetric use of the group-lasso norm and follows naturally from a particular parametrization of the model. We provide conditions under which our estimator is model selection consistent in the high-dimensional regime.

    \prefacesection{Acknowledgements}
       \begin{itemize}
       	\item I would like to thank my advisors Trevor Hastie and Jonathan Taylor. Trevor has given me the perfect amount of guidance and encouragement through my PhD. I greatly benefited from his numerous statistical and algorithmic insights and intuitions and input on all my projects. Jonathan has been a great mentor throughout my PhD. He has been generous with sharing his ideas and time. With Yuekai Sun, we have spent countless afternoons trying to understand Jonathan's newest ideas. 
       	\item I would like to thank Lester Mackey for serving on my oral exam committee and reading committee. The Stats-ML reading group discussions also introduced me to several new areas of research. I would also like to thank Andrea Montanari and Percy Liang for their spectacular courses and serving on my oral exam committee.
       	\item Emmanuel Candes, John Duchi, and Rob Tibshirani have always been available to provide advice, guiadance, and fantastic courses.
       	\item Yuekai Sun and I have collaborated on many projects. I've benefited from his numerous insights, and countless discussions. I am lucky to have a collaborator like Yuekai. I would also like to thank my classmates in ICME for their friendship.
       	\item ICME has been a great place to spend the past 5 years. Margot, Indira, Emily, and Antoinette have kept ICME running smoothly, for which I am very grateful.
       	\item Microsoft Research and Technicolor for hosting me in the summers. In particular, I want to thank my mentors Ran Gilad-Bachrach, Emre Kiciman, Nadia Fawaz, and Brano Kveton for making my summers enjoyable. I would also like to thank the numerous friends I met at Microsoft Research and Technicolor.
       	\item I would like to thank my parents for their unconditional love and support. They are the best parents one could hope for. 
       \end{itemize}

\newpage
~\\~\\~\\~\\~\\~\\~\\~\\~\\~\\~\\~\\
\begin{center}	
	\it Dedicated to my parents Jerry Lee and Tien-Tien Chou.
\end{center}
  \afterpreface

\chapter{Introduction}
This thesis is split into two parts: selective inference and learning mixed graphical models. The contributions are summarized below:
\begin{itemize}
	\item Selective Inference:
	\begin{itemize}
		\item Chapter \ref{chap:sel-lasso}: This chapter studies selective inference for the lasso-selected model. We show how to construct confidence intervals for regression coefficients corresponding to variables selected by the lasso, and how to test the significance of a lasso-selected model by conditioning on the selection event of the lasso. The results of this chapter appear in \cite{lee2013exact} and is joint work with Dennis Sun, Yuekai Sun, and Jonathan Taylor.
		\item Chapter \ref{chap:selection-additional}: This chapter shows how the Condition-on-Selection method developed in Chapter \ref{chap:sel-lasso} is not specific to the lasso. In Chapter \ref{sec:formalism}, we show that controlling the conditional type 1 error implies control of the selective type 1 error, which motivates the use of the Condition-on-Selection method to control conditional type 1 error. Chapter \ref{sec:other-affine-selection} studies several other variable selection methods including marginal screening, orthogonal matching pursuit, and non-negative least squares with affine selection events, so we can apply the results of Chapter \ref{chap:sel-lasso}. Motivated by more complicated selection algorithms that do not simple selection events,such as the knockoff filter, SCAD/MCP regularizers, and $\ell_1$-logistic regression, we develop a general algorithm that only requires a blackbox evaluation of the selection algorithm in Chapter \ref{sec:general-method}. Finally in Chapter \ref{sec:full-model} we study inference for the full model regression coefficients. We show a method for FDR control, and the asymptotic coverage of selective confidence intervals in the high-dimensional regime. This chapter is joint work with Jonathan Taylor and will appear in a future publication.
	\end{itemize}
	
	\item Learning Mixed Graphical Models:
	\begin{itemize}
		\item We propose a new pairwise Markov random field that generalizes the Gaussian graphical model to include categorical variables.
		\item We design a new regularizer that promotes edge sparsity in the mixed graphical model.
		\item Three methods for parameter estimation are proposed: pseudoliklihood, node-wise regression, and maximum likelihood.
		\item The resulting optimization problem is solved using the proximal Newton method \cite{lee2012proximal}.
		\item We use the framework of \cite{lee2013model} to establish edge selection consistency results for the MLE and pseudolikelihood estimation methods.
		\item The results of this chapter originally appeared in \cite{Lee2014learning} and is joint work with Trevor Hastie.
	\end{itemize}
\end{itemize}
\part{Selective Inference}
   \chapter{Selective Inference for the Lasso}
\label{chap:sel-lasso}
\section{Introduction}
\label{sec:intro}

As a statistical technique, linear regression is both simple and powerful. Not only does it provide estimates of the ``effect'' of each variable, but it also quantifies the uncertainty in those estimates, paving the way for intervals and tests of the effect size. However, in many applications, a practitioner starts with a large pool of candidate variables, such as genes or demographic features, and does not know \emph{a priori} which are relevant. The problem is especially acute if there are more variables than observations, when it is impossible to even fit linear regression.

A practitioner might wish to use the data to select the relevant variables and then make inference on the selected variables. As an example, one might fit a linear model, observe which coefficients are significant at level $\alpha$, and report $(1-\alpha)$-confidence intervals for only the significant coefficients. However, these intervals fail to take into account the randomness in the selection procedure. In particular, the intervals do not have the stated coverage once one marginalizes over the selected model.

To see this formally, assume the usual linear model 
\begin{equation} 
y = \mu + \epsilon,\ \mu = X\beta^0,\ \epsilon \sim N(0, \sigma^2I),
\label{eq:model}
\end{equation}
where $X \in \R^{n \times p}$ is the design matrix and $\beta^0 \in \R^p$. Let $\hat\E \subset \{ 1, ..., p \}$ denote a (random) set of selected variables. Suppose the goal is inference about $\beta^0_j$. Then, we do not even form intervals for $\beta^0_j$ when $j \notin \hat\E$, so the first issue is to define an interval when $j\notin \hat\E$ in order to evaluate the coverage of this procedure. There is no obvious way to do this so that the marginal coverage is $1-\alpha$. Furthermore, as $\hat\E$ varies, the target of the ordinary least-squares (OLS) estimator $\hat\beta_{\hat\E}^{OLS}$ is not $\beta^0$, but rather 
\[ \beta^\star_{\hat\E} := X_{\hat\E}^\dagger\mu , \] 
where $X_{\hat\E}^\dagger$ denotes the Moore-Penrose pseudoinverse of $X_{\hat\E}$. We see that $X_{\hat\E}\beta^\star_{\hat\E} = P_{\hat\E}\mu$, the projection of $\mu$ onto the columns of $X_{\hat\E}$, so $\beta^\star_{\hat\E}$ represents the coefficients in the best linear model using only the variables in $\hat\E$. In general, $\beta^\star_{\hat\E, j} \neq \beta^0_j$ unless $\hat\E$ contains the support set of $\beta^0$, i.e., $\hat\E \supset S := \{ j: \beta_j^0 \neq 0 \}$. Since $\hat\beta_{\hat\E, j}^{OLS}$ may not be estimating $\beta^0_j$ at all, there is no reason to expect a confidence interval based on it to cover $\beta^0_j$. \citet{berk2013posi} provide an explicit example of the non-normality of $\hat\beta_{\hat\E, j}^{OLS}$ in the post-selection context. In short, inference in the linear model has traditionally been incompatible with model selection.

\subsection{The Lasso}

In this paper, we focus on a particular model selection procedure, the lasso \citep{tibshirani:lasso}, which achieves model selection by setting coefficients to zero exactly. This is accomplished by adding an $\ell_1$ penalty term to the usual least-squares objective:
\begin{equation}
\label{eq:lasso}
\hbeta \in \argmin_{\beta \in \R^p} \,
\frac{1}{2} \|y-X\beta\|^2_2+ \lambda \|\beta\|_1, 
\end{equation}
where $\lambda \geq 0$ is a penalty parameter that controls the tradeoff between fit to the data and sparsity of the coefficients. However, the distribution of the lasso estimator $\hbeta$ is known only in the less interesting $n \gg p$ case \citep{knight2000lasso}, and even then, only asymptotically. Inference based on the lasso estimator is still an open question.

We apply our framework for post-selection inference about $\eta_{\hat\E}^T \mu$ to form confidence intervals for $\beta^\star_{\hat\E, j}$ and to test whether the the fitted model captures all relevant signal variables.

\subsection{Related Work}

Most of the theoretical work on fitting high-dimen\-sional linear models focuses on \emph{consistency}. The flavor of these results is that under certain assumptions on $X$, the lasso fit $\hat{\beta}$ is close to the unknown $\beta^0$  \citep{negahban2012unified} and selects the correct model \citep{zhao2006model,wainwright2009sharp}. A comprehensive survey of the literature can be found in \citet{buhlmann2011statistics}.

There is also some recent work on obtaining confidence intervals and significance testing for penalized M-estimators such as the lasso. One class of methods uses sample splitting or subsampling to obtain confidence intervals and p-values. Recently, \cite{meinshausen2010stability} proposed \emph{stability selection} as a general technique designed to improve the performance of a variable selection algorithm. The basic idea is, instead of performing variable selection on the whole data set, to perform variable selection on random subsamples of the data of size $\frac n2$ and include the variables that are selected most often on the subsamples. 

A separate line of work establishes the asymptotic normality of a corrected estimator obtained by ``inverting'' the KKT conditions \citep{van2013asymptotically,zhang2011confidence,javanmard2013confidence}. The corrected estimator $\hat{b}$ usually has the form
$$
\hat{b} = \hat{\beta} + \lambda\Theta\hat{z},
$$
where $\hat{z}$ is a subgradient of the penalty at $\hat{\beta}$ and $\Theta$ is an approximate inverse to the Gram matrix $X^TX$. This approach is very general and easily handles M-estimators that minimize the sum of a smooth convex loss and a convex penalty. The two main drawbacks to this approach are:
\BNUM
\item the confidence intervals are valid only when the M-estimator is consistent
\item obtaining $\Theta$ is usually much more expensive than obtaining $\hat{\beta}$.
\ENUM

Most closely related to our work is the pathwise signficance testing framework laid out in \cite{lockhart2012significance}. They establish a test for whether a newly added coefficient is a relevant variable. This method only allows for testing at $\lambda$ that are LARS knot values. This is a considerable restriction, since the lasso is often not solved with the LARS algorithm. Furthermore, the test is asymptotic, makes strong assumptions on $X$, and the weak convergence assumes that all relevant variables are already included in the model. They do not discuss forming confidence intervals for the selected variables. Section \ref{sec:goodness} establishes a nonasymptotic test for the same null hypothesis, while only assuming $X$ is in general position. 

In contrast, we provide a test that is exact, allows for arbitrary $\lambda$, and arbitrary design matrix $X$. 
By extension, we do not make any assumptions on $n$ and $p$, and do not require the lasso to be a consistent estimator of $\beta^0$. Furthermore, the computational expense to conduct our test is negligible compared to the cost of obtaining the lasso solution.


Like all of the preceding works, our test assumes that the noise variance $\sigma^2$ is known or can be estimated. In the low-dimensional setting $p \ll n$, $\sigma^2$ can be estimated from the residual sum-of-squares of the saturated model. Strategies in high dimensions are discussed in \citet{fan2012variance} and \citet{reid2013variance}. In Section \ref{sec:extensions}, we also provide a strategy for estimating $\sigma^2$ based on the framework we develop.

\subsection{Outline of Chapter}

We begin by defining several important quantities related to the lasso in Section \ref{sec:prelim}; most notably, we define the selected model $\hat\E$ in terms of the active set of the lasso solution. Section \ref{sec:lasso-selection} provides an alternative characterization of the selection procedure for the lasso in terms of affine constraints on $y$, i.e., $Ay \leq b$. Therefore, the distribution of $y$ conditional on the selected model is the distribution of a Gaussian vector conditional on its being in a polytope. In Section \ref{sec:truncated-gaussian-test}, we generalize and show that for $y \sim N(\mu, \Sigma)$, the distribution of $\eta^T y\ |\ Ay \leq b$ is roughly a truncated Gaussian random variable, and derive a pivot for $\eta^T \mu$. In Section \ref{sec:lasso}, we specialize again to the lasso, deriving confidence intervals for $\beta^\star_{\hat\E, j}$ and hypothesis tests of the selected model as special cases of $\eta^T \mu$. Section \ref{sec:examples} presents an example of these methods applied to a dataset. 

In Section \ref{sec:minimal}, we consider a refinement that produces narrower confidence intervals. Finally, Section \ref{sec:extensions} collects a number extensions of the framework. In particular, we demonstrate:
\begin{itemize}
\item modifications needed for the elastic net \citep{zou2005regularization}.
\item different norms as test statistics for the ``goodness of fit'' test discussed in Section \ref{sec:lasso}.
\item estimation of $\sigma^2$ based on fitting the lasso with a sufficiently small $\lambda$.
\item composite null hypotheses.
\item fitting the lasso for a sequence of $\lambda$ values and its effect on our basic tests and intervals.
\end{itemize}

\section{Preliminaries}
\label{sec:prelim}

Necessary and sufficient conditions for $(\hat\beta, \hat z)$ to be solutions to the lasso problem \eqref{eq:lasso} are the Karush-Kuhn-Tucker (KKT) conditions:
\begin{gather}
X^T(X\hbeta - y) + \lambda\hat{z} = 0, 
\label{eq:lasso-kkt-1} \\
\hat{z}_i \in \BCAS \sign(\hat{\beta}_i) & \text{if }\hat{\beta}_i \ne 0 \\ [-1, 1] & \text{if }\hat{\beta}_i = 0 \ECAS.
\label{eq:lasso-kkt-2}
\end{gather}
where $\hat z := \partial ||\cdot||_1 (\hat\beta)$ denotes the subgradient of the $\ell_1$ norm at $\hat\beta$. We consider the \emph{active set} \citep{tibshirani2013lasso} 
\begin{equation}
\label{eq:equicor}
\hat\E = \left\{i\in\{1,\dots,p\} : |\hat{z}_i| = 1 \right\},
\end{equation}
so-named because by examining only the rows corresponding to $\hat\E$ in \eqref{eq:lasso-kkt-1}, we obtain the relation 
$$
X_{\hat\E}^T(y - X\hbeta) = -\lambda\hat{z}_{\hat\E},
$$
where $X_{\hat\E}$ is the submatrix of $X$ consisting of the columns in $\hat\E$. Hence 
$$
|X_{\hat\E}^T(y - X\hbeta)| = \lambda,
$$
i.e.\ the variables in this set have equal (absolute) correlation with the residual $y - X\hbeta$. Since $\hat z_i \in \{-1, 1\}$ for any $\hbeta_i \neq 0$, all variables with non-zero coefficients are contained in the active set. 


Recall that we are interested in inference for $\eta^T \mu$ in the model \eqref{eq:model} for some direction $\eta = \eta_{\hat\E} \in \R^n$, which is allowed to depend on the selected variables $\hat\E$. In most applications, we will assume $\mu = X\beta^0$, although our results hold even if the linear model is not correctly specified. 

A natural estimate for $\eta^T \mu$ is $\eta^T y$. As mentioned previously, we allow $\eta = \eta_{\hat\E}$ to depend on the random selection procedure, so our goal is post-selection inference based on 
$$\eta^T y\ |\ \{ \hat\E = \E \}.$$
For reasons that will become clear, a more tractable quantity is the distribution conditional on both the selected variables and their signs
$$ \eta^T y\ |\ \{ (\hat\E, \hat z_{\hat\E}) = (\E, z_\E) \}.$$
Note that confidence intervals and hypothesis tests that are valid conditional on the finer partition $\{ (\hat\E, \hat z_{\hat\E}) = (\E, z_\E) \}$ will also be valid for $\{ \hat\E = \E \}$, by summing over the possible signs $z_\E$:
$$ \Pp(\ \cdot\ \big|\ \hat\E = \E) = \sum_{z_\E} \Pp(\ \cdot\ \big|\ (\hat\E, \hat z_{\hat\E}) = (\E, z_\E))\ \Pp(\hat z_{\hat\E} = z_\E\ \big|\ \hat\E = \E).$$
From this, it is clear that controlling $\Pp(\ \cdot\ \big|\ (\hat\E, \hat z_{\hat\E}) = (\E, z_\E))$ to be, say, less than $\alpha$ (as in the case of hypothesis testing) will ensure $\Pp(\ \cdot\ \big|\ \hat\E = \E) \leq \alpha$. 

It may not be obvious yet why we condition on $\{ (\hat\E, \hat z_{\hat\E}) = (\E, z_\E)\}$ instead of $\{ \hat\E = \E \}$. In the next section, we show that the former can be restated in terms of affine constraints on $y$, i.e., $\{Ay \leq b\}$. We revisit the problem of conditioning only on $\{ \hat\E = \E \}$ in Section \ref{sec:minimal}.

\section{Characterizing Selection for the Lasso}
\label{sec:lasso-selection}

Recall from the previous section that our goal is inference conditional on $\{ (\hat\E, \hat z_{\hat\E}) = (\E, z_\E) \}$. In this section, we show that this \textit{selection event} can be rewritten in terms of affine constraints on $y$, i.e.,
$$ \{ (\hat\E, \hat z_{\hat\E}) = (\E, z_\E) \} = \{ A(\E, z_\E)y \leq b(\E, z_\E) \}$$
for a suitable matrix $A(\E, z_\E)$ and vector $b(\E, z_\E)$. Therefore, the conditional distribution $y\ |\ \{ (\hat\E, \hat z_{\hat\E}) = (\E, z_\E) \}$ is simply $y\ \big|\ \{A(\E, z_\E)y \leq b(\E, z_\E)\}$. This key theorem follows from two intermediate results.

\begin{lemma}
\label{lem:equiv_sets}
Without loss of generality, assume the columns of $X$ are in general position. Let $\E \subset \{1, \dots, p\}$ and $z_\E \in \{-1,1\}^{|\E|}$ be a candidate set of variables and signs, respectively. Define 
\begin{align}
U = U(\E, z_\E) &:= (X_\E^T X_\E)^{-1} (X_\E^T y - \lambda z_\E) \label{eq:beta-active} \\
W = W(\E, z_\E) &:=  X_{-\E}^T(X_\E^T)^\dagger z_\E + \frac{1}{\lambda} X_{-\E}^T (I - P_{\E}) y \label{eq:inactive_subgradient}.
\end{align}
Then the selection procedure can be rewritten in terms of $U$ and $W$ as:
\begin{align}
\{ (\hat\E, \hat z_{\hat\E}) = (\E, z_\E)\} = \{ \sign(U(\E, z_\E)) = z_\E, \norm{W(\E, z_\E)}_\infty < 1 \} 
\label{eq:equiv_sets}
\end{align}

\begin{proof}
First, we rewrite the KKT conditions \eqref{eq:lasso-kkt-1} and \eqref{eq:lasso-kkt-2} by partitioning them according to the active set $\hat\E$:
\begin{align*}
X_{\hat\E}^T (X_{\hat\E}\hat{\beta}_{\hat\E} - y) + \lambda\hat{z}_{\hat\E} = 0 \\
X_{-\hat\E}^T (X_{\hat\E}\hat{\beta}_{\hat\E} - y) + \lambda\hat{z}_{-\hat\E} = 0 \\
\sign(\hat\beta_{\hat\E}) = \hat{z}_{\hat\E},\, \hat{z}_{-\hat\E} \in (-1, 1) .
\end{align*}
Since the KKT conditions are necessary and sufficient for a solution, we obtain that $\{ (\hat\E, \hat z_{\hat\E}) = (\E, z_\E) \}$ if and only if there exist $U$ and $W$ satisfying:
\begin{gather}
X_\E^T (X_\E U - y) + \lambda z_\E = 0 \label{eq:lasso-E-fixed}\\
X_{-\E}^T (X_\E U - y) + \lambda W = 0 \label{eq:lasso-nE-fixed}\\
\sign(U) = z_\E,\, W \in (-1, 1) \label{eq:signs}.
\end{gather}
Solving \eqref{eq:lasso-E-fixed} and \eqref{eq:lasso-nE-fixed} for $U$ and $W$ yields the formulas \eqref{eq:beta-active} and \eqref{eq:inactive_subgradient}. Finally, the requirement that $U$ and $W$ satisfy \eqref{eq:signs} yields \eqref{eq:equiv_sets}.
\end{proof}

\end{lemma}

Lemma \ref{lem:equiv_sets} is remarkable because it says that the selection event $\{ (\hat\E, \hat z_{\hat\E}) = (\E, z_\E) \}$ is equivalent to affine constraints on $y$. To see this, note that both $U$ and $W$ are affine functions of $y$, so $\{  \sign(U) = z_\E, \norm{W}_\infty < 1 \}$ can be written as affine constraints $\{ A(\E, z_\E)y \leq b(\E, z_\E) \}$. The following proposition provides explicit formulas for $A$ and $b$.

\begin{proposition}
\label{prop:A_b}
Let $U$ and $W$ be defined as in \eqref{eq:beta-active} and \eqref{eq:inactive_subgradient}. Then:
\begin{align}
\{ \sign(U) = z_\E , \norm{W}_\infty < 1 \}  = \left\{ \begin{pmatrix} A_0(\E, z_\E) \\ A_1(\E, z_\E) \end{pmatrix} y < \begin{pmatrix} b_0(\E, z_\E) \\ b_1(\E, z_\E) \end{pmatrix}\right\}
\label{eq:polytope}
\end{align}
where $A_0, b_0$ encode the ``inactive'' constraints $\{\norm{W}_\infty < 1\}$, and $A_1, b_1$ encode the ``active'' constraints $\{\sign(U) = z_\E\}$. These matrices have the explicit forms:
\begin{align*}
A_0(\E, z_\E) &= \frac{1}{\lambda}\begin{pmatrix}X_{-\E}^T (I - P_\E) \\- X_{-\E}^T (I - P_\E) \end{pmatrix} & b_0(\E, z_\E) &= \begin{pmatrix}\ones - X_{-\E}^T (X_\E^T)^\dagger z_\E \\ \ones + X_{-\E}^T (X_\E^T)^\dagger z_\E  \end{pmatrix} \\
A_1(\E, z_\E) &= -\diag(z_\E) (X_\E^T X_\E)^{-1} X_\E^T  & b_1(\E, z_\E) &= -\lambda \diag(z_\E)(X_\E^T X_\E)^{-1} z_\E
\end{align*}
\end{proposition}

\begin{proof}
First, we write 
$$ \{ \sign(U) = z_\E \} = \{ \diag(z_\E) U > 0 \}. $$
From here, it is straightforward to derive the above expressions from the definitions of $U$ and $W$ given in \eqref{eq:beta-active} and \eqref{eq:inactive_subgradient}. 
\end{proof}

Combining Lemma \ref{lem:equiv_sets} with Proposition \ref{prop:A_b}, we obtain the following.

\begin{theorem}
The selection procedure can be rewritten in terms of affine constraints on $y$:
\[
\{ (\hat\E, \hat z_{\hat\E}) = (\E, z_\E)\} = \{ A(\E, z_\E)y \leq b(\E, z_\E) \}.
\]
\end{theorem}

To summarize, we have shown that in order to understand the distribution of $y \sim N(\mu, \Sigma)$ conditional on the selection procedure $\{ (\hat\E, \hat z_{\hat\E}) = (\E, z_\E) \}$, it suffices to study the distribution of $y$ conditional on being in the polytope $\{ Ay \leq b\}$. The next section derives a pivot for $\eta^T \mu$ for such distributions, which will be useful for constructing confidence intervals and hypothesis tests in Section \ref{sec:lasso}.

\section{A Pivot for Gaussian Vectors Subject to Affine Constraints}
\label{sec:truncated-gaussian-test}

The distribution of a Gaussian vector $y \sim N(\mu, \Sigma)$ conditional on affine constraints $\{Ay \leq b\}$, while explicit, still involves the intractable normalizing constant $\Pp(Ay \leq b)$. In this section, we show that one dimensional projections of $\mu$ (i.e., $\eta^T \mu$) are univariate truncated normal, which will allow us to form tests and intervals for $\eta^T \mu$. 

The key to deriving this pivot is the following lemma:
\begin{lemma}
\label{lem:conditional}
The conditioning set can be rewritten in terms of $\eta^T y$ as follows:
\[  \{Ay \leq b\} = \{\V^-(y) \leq \eta^T y \leq \V^+(y), \V^0(y) \geq 0 \} \]
where 
\begin{align}
a &= \frac{A\Sigma\eta}{\eta^T\Sigma\eta} \label{eq:alpha} \\
\V^- = \V^-(y) &= \max_{j:\ a_j < 0} \frac{b_j - (Ay)_j +a_j\eta^T y} {a_j} \label{eq:v_minus} \\
\V^+ = \V^+(y) &= \min_{j:\ a_j > 0} \frac{b_j - (Ay)_j + a_j\eta^T y}{a_j}. \label{eq:v_plus} \\
\V^0 = \V^0(y) &= \min_{j:\ a_j = 0} b_j - (Ay)_j \label{eq:v_zero}
\end{align}
Furthermore, $(\V^+,\V^-,\V^0)$ is independent of $\eta^T y$.
Then, $\eta^Ty$ conditioned on $Ay \le b$ and $(\V^+(y), \V^-(y)) = (v^+,v^-)$, has a truncated normal distribution, \ie{}
\[\textstyle
\eta^Ty\,\big| \left\{\,Ay \le b,\V^+(y)=v^+, \V^-(y)=v^-\right\} \sim TN(\eta^T\mu,\eta^T\Sigma\eta,v^-,v^+).
\]

\end{lemma}

However, before stating the proof of this lemma, we show how it is used to obtain our main result. 

\begin{theorem}
\label{thm:truncated-gaussian-pivot}
Let $F_{\mu, \sigma^2}^{[a, b]}$ denote the CDF of a $N(\mu, \sigma^2)$ random variable truncated to the interval $[a, b]$, i.e.:
\begin{equation}
F_{\mu, \sigma^2}^{[a, b]}(x) = \frac{\Phi((x-\mu)/\sigma) - \Phi((a-\mu)/\sigma)}{\Phi((b-\mu)/\sigma) - \Phi((a-\mu)/\sigma)}
\label{eq:U}
\end{equation}
where $\Phi$ is the CDF of a $N(0, 1)$ random variable. Then $F_{\eta^T\mu,\ \eta^T \Sigma \eta}^{[\V^-, \V^+]}(\eta^T y)$ is a pivotal quantity, conditional on $\{Ay \leq b\}$:
\begin{equation}
F_{\eta^T\mu,\ \eta^T \Sigma \eta}^{[\V^-, \V^+]}(\eta^T y)\ \big|\ \{Ay \leq b\} \sim \unif(0,1) 
\label{eq:pivot}
\end{equation}
where $\V^-$ and $\V^+$ are defined in \eqref{eq:v_minus} and \eqref{eq:v_plus}.
\end{theorem}

\begin{proof}
By Lemma \ref{lem:conditional}, $\eta^Ty\,\big| \left\{\,Ay \le b,\V^+(y)=v^+, \V^-(y)=v^-\right\} \sim TN(\eta^T\mu,\eta^T\Sigma\eta,v^-,v^+)$. We apply the CDF transform to deduce
\[
F_{\eta^T\mu,\eta^T\Sigma\eta}^{[v^-,v^+]}(\eta^Ty)\,\big| \left\{\,Ay \le b,\V^+(y)=v^+, \V^-(y)=v^-\right\}
\]
is uniformly distributed. 
By integrating over $(\V^+(y)=v^+, \V^-(y)=v^-)$, we conclude  $F_{\eta^T\mu,\ \eta^T \Sigma \eta}^{[\V^-, \V^+]}(\eta^T y)\ \big|\ \{Ay \leq b\} \sim \unif(0,1) $. Let $G(v^+,v^-)= \Pp(\V^+ \le v^+,\V^- \le v^- \mid Ay\le b)$.
\begin{align*}
&\Pp\left( F_{\eta^T\mu,\ \eta^T \Sigma \eta}^{[\V^-, \V^+] }(\eta^T y)\le s\ \big|\ Ay \leq b\right)\\
\pc&=\int  \Pp \left( F_{\eta^T\mu,\ \eta^T \Sigma \eta}^{[v^-, v^+] }(\eta^T y )\le s\ \big|\ Ay \leq b,\V^+(y)=v^+, \V^-(y)=v^-\right)\\
 &\qquad \ dG(v^+,v^- ) \\
 \pc&= \int s\ dG(v^+,v^- )
\\
\pc&=s.
\end{align*}
\end{proof}

We now prove Lemma \ref{lem:conditional}.
\begin{proof}
The linear constraints $Ay \le b$ are equivalent to
\BEQ
Ay - \Expect[Ay\mid\eta^Ty] \le b - \Expect[Ay\mid\eta^Ty].
\label{eq:truncated-normal-1}
\EEQ
Since conditional expectation has the form
\[
\Expect[Ay\mid\eta^Ty] = A\mu + a(\eta^Ty - \eta^T\mu),\,a = \frac{A\Sigma\eta}{\eta^T\Sigma\eta},
\]
\eqref{eq:truncated-normal-1} simplifies to $Ay - b - a\eta^Ty \le -a\eta^Ty$. Rearranging, we obtain
\begin{align*}
\eta^Ty &\ge \frac{1}{a_j}(b_j - (Ay)_j + a_j\eta^Ty) & a_j < 0 \\
\eta^Ty &\le \frac{1}{a_j}(b_j - (Ay)_j + a_j\eta^Ty) & a_j > 0 \\
0 &\le b_j - (Ay)_j & a_j = 0.
\end{align*}
We take the max of the lower bounds and min of the upper bounds to deduce
\[
\underbrace{\max_{j:a_j < 0}\frac{1}{a_j}(b_j - (Ay)_j + a_j\eta^Ty)}_{\V^-(y)} \le \eta^Ty \le \underbrace{\min_{j:a_j > 0}\frac{1}{a_j}(b_j - (Ay)_j + a_j\eta^Ty)}_{\V^+(y)}.
\]
Since $y$ is normal, $b_j - (Ay)_j + a_j\eta^Ty,\,j=1,\dots,m$ are independent of $\eta^Ty$. Hence $(\V^+(y),\,\V^-(y),\,\V^0(y))$ are also independent of $\eta^Ty$.

To complete the proof, we must show $\eta^Ty$ given $Ay \le b$, $(\V^+(y), \V^-(y)) = (v^+,v^-)$ is truncated normal. 
\begin{align*}
&\Pp\left(\eta^Ty\le s \ \big| Ay \le b,\V^+(y)=v^+, \V^-(y)=v^-,  \right) \\
&\pc= \Pp\left( \eta^Ty\le s \ \big| v^-\le \eta^Ty\le v^+,\V^+(y)=v^+, \V^-(y)=v^-, \V^0(y) \ge 0\right)\\
&\pc= \frac{\Pp\left( \eta^Ty \le s , v^- \le \eta^Ty \le v^+ \big| \V^+(y)=v^+, \V^-(y)=v^-, \V^0(y) \ge 0 \right)}{\Pp\left(v^- \le \eta^Ty \le v^+ \big| \V^+(y)=v^+, \V^-(y)=v^-, \V^0(y) \ge 0 \right)}\\
&\pc= \frac{\Pp\left( \eta^Ty \le s, v^- \le \eta^Ty \le v^+\right)}{\Pp\left(v^- \le \eta^Ty \le v^+\right)} = \Pp\left( \eta^Ty \le s\mid v^- \le \eta^Ty \le v^+\right)
\end{align*}
where the second to last equality follows from the independence of $(\V^+,\V^-,\V^0)$ and $\eta^Ty$. This is the CDF of a truncated normal.
\end{proof}

Although the proof of Lemma \ref{lem:conditional} is elementary, the geometric picture gives more intuition as to why $\V^+$ and $\V^-$ are independent of $\eta^T y$. Without loss of generality, we assume $||\eta||_2 = 1$ and $y \sim N(\mu, I)$ (since otherwise we could replace $y$ by $\Sigma^{-\frac{1}{2}}y$). Now we can decompose $y$ into two independent components, a 1-dimensional component $\eta^T y$ and an $(n-1)$-dimensional component orthogonal to $\eta$:
\[ y = \eta^T y + P_{\eta^\perp} y. \]
The case of $n=2$ is illustrated in Figure \ref{fig:polytope}. $\V^-$ and $\V^+$ are independent of $\eta^T y$, since they are functions of $P_{\eta^\perp}$ only, which is independent of $\eta^T y$.



\begin{figure}[!h]
	\centering
\includegraphics[width = .5\textwidth]{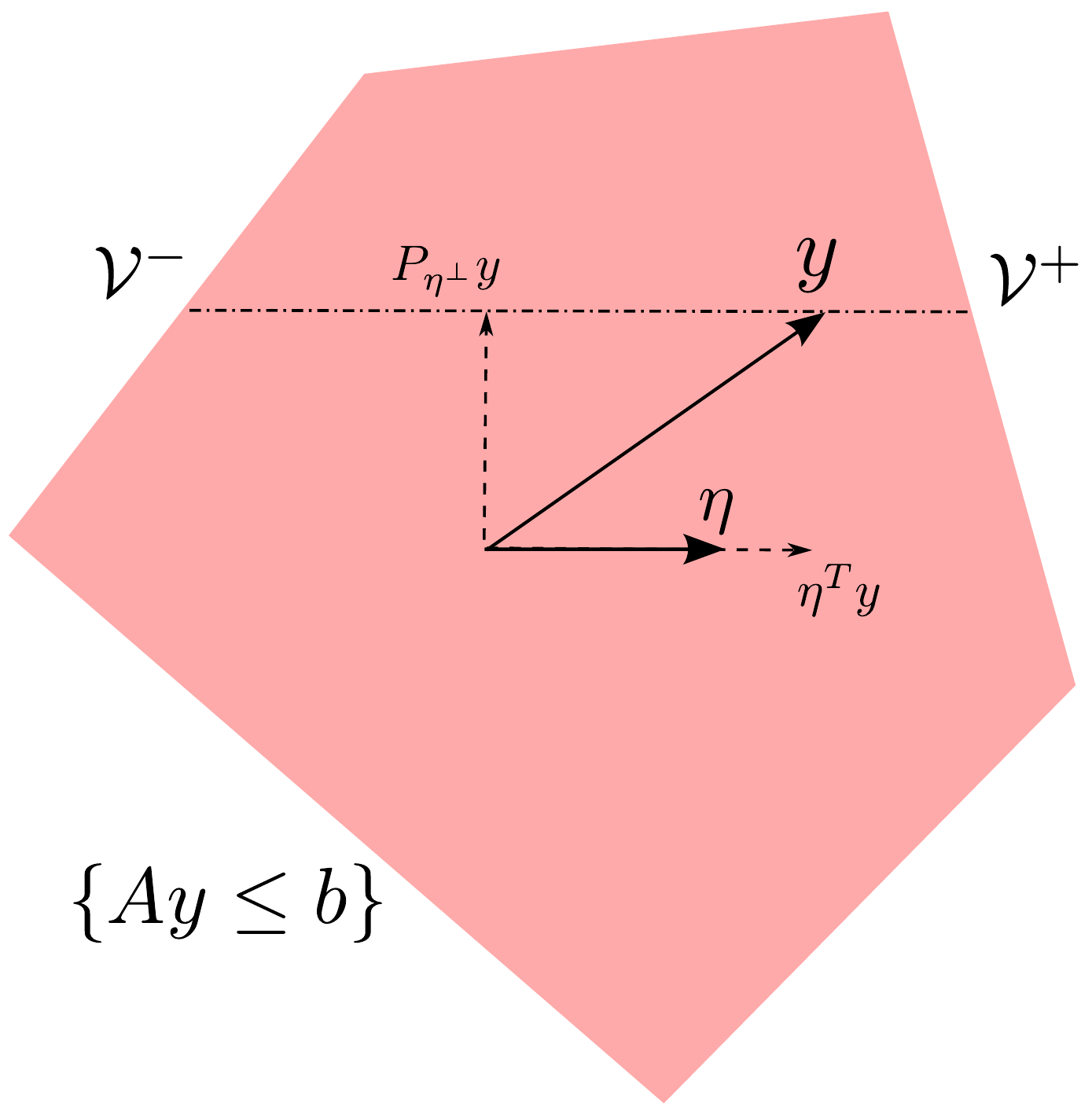}
\caption{A picture demonstrating that the set $\left\{Ay \leq b \right\}$ can be characterized by $\{ \V^- \leq \eta^T y \leq \V^+\}$. Assuming $\Sigma = I$ and $||\eta||_2 = 1$, $\V^-$ and $\V^+$ are functions of $P_{\eta^\perp}y$ only, which is independent of $\eta^T y$.}
\label{fig:polytope}
\end{figure}

In Figure \ref{fig:truncated}, we plot the density of the truncated Gaussian, noting that its shape depends on the
 location of $\mu$ relative to $[a,b]$ as well as the width relative to $\sigma$.

\begin{figure}[!h]
	\centering
\includegraphics[width = .5\textwidth]{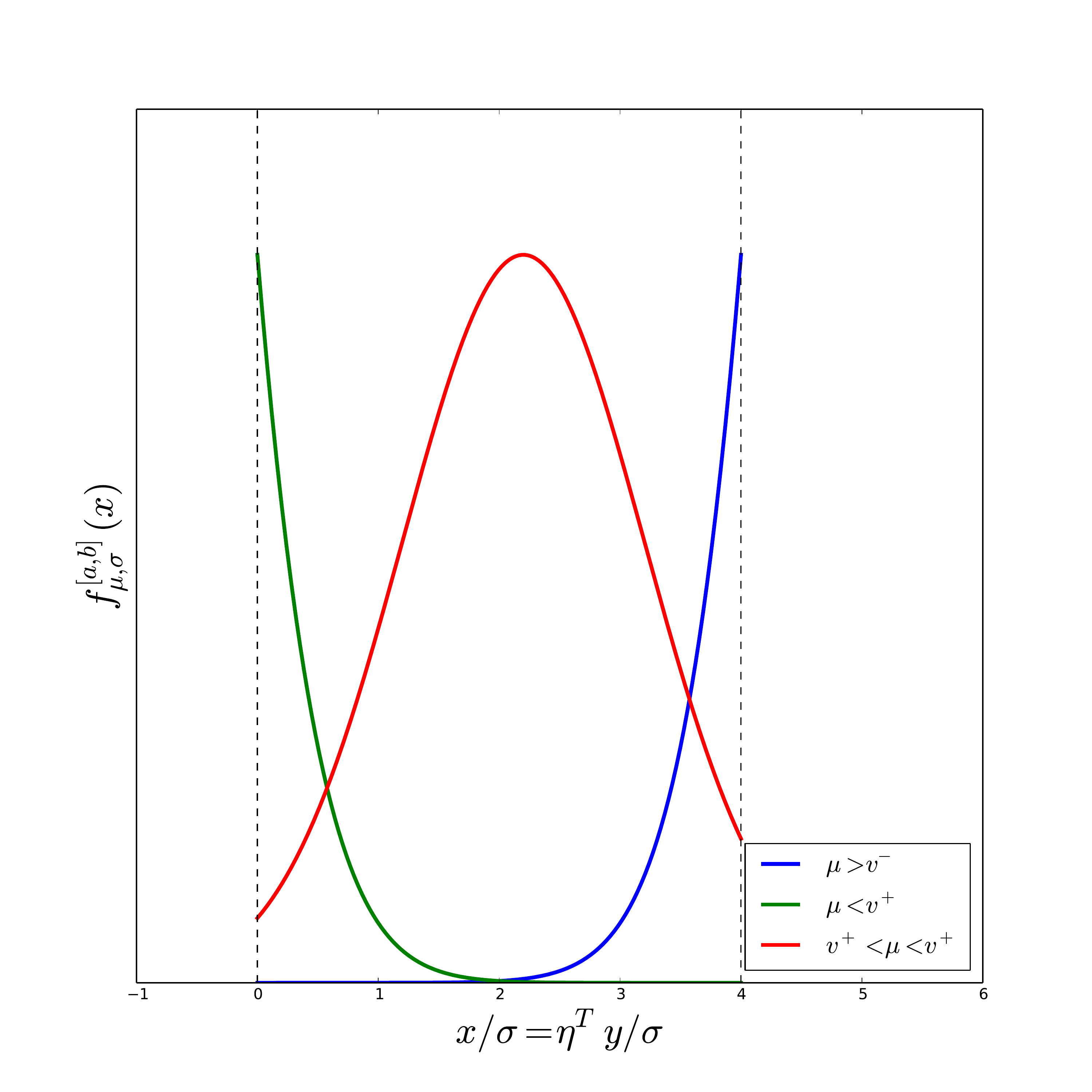}
\caption{The density of the truncated Gaussian with distribution $F^{[v^-,v^+]}_{\mu,\sigma^2}$ depends on the width of $[v^-,v^+]$ relative to $\sigma$ as well as the location of $\mu$ relative to $[v^-,v^+]$. When
$\mu$ is firmly inside the interval, the distribution resembles a Gaussian. As $\mu$ varies outside $[v^-,v^+]$, the density begins to converge to an exponential distribution with mean inversely proportional to the distance between $\mu$ and its projection onto $[v^-,v^+]$.}
\label{fig:truncated}
\end{figure}

\begin{figure}[h]
\includegraphics[width = .48\textwidth]{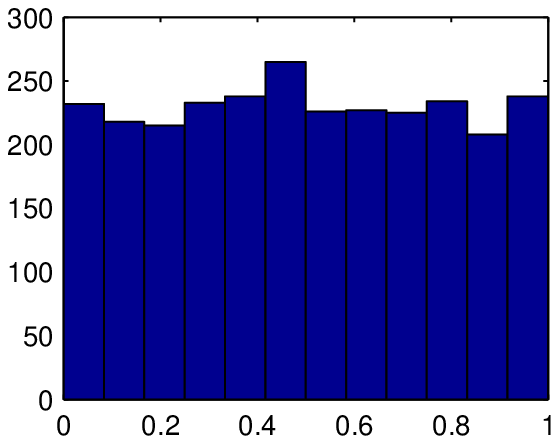}
\includegraphics[width = .48\textwidth]{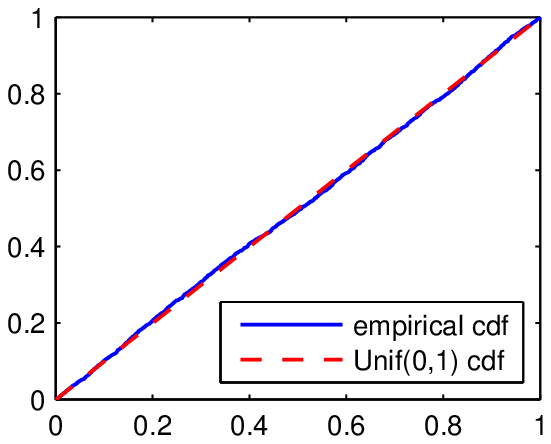}
\caption{Histogram and empirical distribution of $F_{\eta^T\mu,\ \eta^T \Sigma \eta}^{[\V^-, \V^+]}(\eta^T y)$ obtained by sampling $y \sim N(\mu,\Sigma)$ constrained to $\{ Ay \le b\}$. The distribution is very close to $\unif(0,1)$ as shown in Theorem \ref{thm:truncated-gaussian-pivot}.}
\label{fig:null-qq-plot}
\end{figure}
\subsection{Adaptive choice of $\eta$}
For the applications to forming confidence intervals and significance testing, we will need choices of $\eta$ that are adaptive, or dependent on $y$. We will restrict ourselves to functions $\eta$ that are functions of the partition, $\eta(y) = f(\hat \E(y))$. This choice of functions includes $\eta(y) = X_{\hat \E(y)} ^{T\dagger} e_j$ which is used for forming confidence intervals in Section \ref{sec:lasso}.

\begin{theorem}
Let $\eta: \reals^n \to \reals^n$ be a function of the form $\eta(y) =f(\hat \E(y))$, then 
$$
F_{\eta(y)^T\mu,\ \eta(y)^T \Sigma \eta(y)}^{[\V^-(y), \V^+(y)]}\left(\eta(y)^T y\right) \sim \unif(0,1).
$$
\label{thm:unconditional-pivot}
\end{theorem}
\begin{proof}
We can expand $F$ with respect to the partition,
\begin{align*}
\Pp& \left(F_{\eta(y)^T\mu,\ \eta(y)^T \Sigma \eta(y)}^{[\V^-(y), \V^+(y)]}\left(\eta(y)^T y\right)\le t\right)=\sum_{(\E,s)} \Pp \left(F_{\eta(y)^T\mu,\ \eta(y)^T \Sigma \eta(y)}^{[\V^-(y), \V^+(y)]}\left(\eta(y)^T y\right)\le t, \hat \E(y)=\E \right)\\
&=\sum_{(\E,s)}\Pp \left(F_{\eta(y)^T\mu,\ \eta(y)^T \Sigma \eta(y)}^{[\V^-(y), \V^+(y)]}\left(\eta(y)^T y\right)\le t \big| \hat \E(y)=\E\right) \Pp \left( \hat \E(y)=\E\right)\\
&=\sum_{(\E,s)}\Pp \left(F_{f(\E)^T\mu,\ f(\E)^T \Sigma f(\E)}^{[\V^-(y), \V^+(y)]}\left(f(\E)^T y\right)\le t \big| \hat \E(y)=\E\right) \Pp \left( \hat \E(y)=\E\right)
\end{align*}
Using Theorem \ref{thm:truncated-gaussian-pivot}, $\Pp \left(F_{f(\E)^T\mu,\ f(\E)^T \Sigma f(\E)}^{[\V^-(y), \V^+(y)]}\left(f(\E)^T y\right)\le t \big| \hat \E(y)=\E\right) =t$.
Thus 
\begin{align*}
\Pp \left(F_{\eta(y)^T\mu,\ \eta(y)^T \Sigma \eta(y)}^{[\V^-(y), \V^+(y)]}\left(\eta(y)^T y\right))\le t\right)&= \sum_{(\E,s)} t\Pp \left( \hat \E(y)=\E\right)\\
&=t \sum_{(\E,s)} \Pp \left( \hat \E(y)=\E\right)\\
&=t.
\end{align*}
This shows that $F_{\eta(y)^T\mu,\ \eta(y)^T \Sigma \eta(y)}^{[\V^-(y), \V^+(y)]}\left(\eta(y)^T y\right)\sim \unif(0,1)$.
\end{proof}

\section{Application to Inference for the Lasso}
\label{sec:lasso}

In this section, we apply the theory developed in in Sections \ref{sec:lasso-selection} and \ref{sec:truncated-gaussian-test} to the lasso. In particular, we will construct confidence intervals for the active variables and test the chosen model based on the pivot developed in Section \ref{sec:truncated-gaussian-test}.

To summarize the developments so far, recall that our model says that $y \sim N(\mu, \sigma^2 I)$. The distribution of interest is $y\ |\ \{(\hat\E, \hat z_{\hat\E}) = (\E, z_\E)\}$. By Theorem \ref{lem:equiv_sets}, this is equivalent to $y\ |\ {\{A(\E, z_\E)y \leq b(\E, z_\E) \}}$ defined in Proposition \ref{prop:A_b}. Now we can apply Theorem \ref{thm:truncated-gaussian-pivot} to obtain the (conditional) pivot  
\begin{align}
F_{\eta^T\mu,\ \sigma^2 ||\eta||_2^2}^{[\V^-, \V^+]}(\eta^Ty)\ \big|\ \{(\hat\E, \hat z_{\hat\E}) = (\E, z_\E)\} \sim \unif(0,1)
\label{eq:pivot-lasso}
\end{align}
for any $\eta$, where $\V^-$ and $\V^+$ are defined in \eqref{eq:v_minus} and \eqref{eq:v_plus}. Note that $A(\E, z_\E)$ and $b(\E, z_\E)$ appear in this pivot through $\V^-$ and $\V^+$. This pivot will play a central role in all of the applications that follow.

\subsection{Confidence Intervals for the Active Variables}
In this section, we describe how to form confidence intervals for the components of $\beta^\star_ {\hat\E}= X_{\hat\E}^\dagger \mu$. 
 If we choose 
\begin{equation}
\label{eq:eta_confint}
\eta_j = (X_{\hat\E}^T)^\dagger e_j,
\end{equation} 
then $\eta_j^T \mu = \beta_{\hat\E, j}^\star$, so the above framework provides a method for inference about the $j^\text{th}$ variable in the model $\hat\E$. Note that this reduces to inference about the true $\beta^0_j$ if $\hat\E \supset S := \{ j: \beta^0_j \neq 0 \}$, as discussed in Section \ref{sec:intro}. Conditions under which this holds are well known in the literature, cf. \cite{buhlmann2011statistics}, and provided in Section \ref{appendix:screening}.


By applying Theorem \ref{thm:truncated-gaussian-pivot}, we obtain the following (conditional) pivot for $\beta^\star_{\hat\E,j}$:
\[
F_{\beta^\star_{\hat\E,j},\ \sigma^2 ||\eta_j||^2}^{[\V^-, \V^+]}(\eta_j^Ty)\ \big|\ \{(\hat\E, \hat z_{\hat\E}) = (\E, z_\E)\} \sim \unif(0,1).
\]
Note that $j$ and $\eta_j$ are both random---but only through $\hat\E$, a quantity which is fixed after conditioning---so Theorem \ref{thm:truncated-gaussian-pivot} holds even for this ``random'' choice of $\eta$. The obvious way to obtain an interval is to ``invert'' the pivot. In other words, since
\[ \Pp\left(\frac{\alpha}{2} \leq F_{\beta^\star_{\hat\E,j},\ \sigma^2 ||\eta_j||^2}^{[\V^-, \V^+]}(\eta_j^Ty) \leq 1-\frac{\alpha}{2}\ \big|\ \{(\hat\E, \hat z_{\hat\E}) = (\E, z_\E)\} \right) = \alpha, \]
one can define a $(1-\alpha)$ (conditional) confidence interval for $\beta_{\hat\E,j}^\star$ as 
\[ \left\{\beta_{\hat\E,j}^\star: \frac{\alpha}{2} \leq F_{\beta^\star_{\hat\E,j},\ \sigma^2 ||\eta_j||^2}^{[\V^-, \V^+]}(\eta_j^Ty) \leq 1-\frac{\alpha}{2} \right\}. \]
In fact, $F$ is monotone decreasing in $\beta^\star_{\hat\E,j}$, so to find its endpoints, one need only solve for the root of a smooth one-dimensional function. The monotonicity is a consequence of the fact that the truncated Gaussian distribution is a natural exponential family and hence has monotone likelihood ratio in $\mu$. The details can be found in Appendix \ref{appendix:monotone}.

We now formalize the above observations in the following result, an immediate consequence of Theorem \ref{thm:truncated-gaussian-pivot}.

\begin{corollary}
Let $\eta_j$ be defined as in \eqref{eq:eta_confint}, and let $L_\alpha^j = L_\alpha^j(\eta_j,\hat\E, \hat z_{\hat\E})$ and $U_\alpha^j = U_\alpha^j(\eta_j,\hat\E, \hat z_{\hat\E})$ be the (unique) values satisfying 
\begin{align*}
F_{L_\alpha^j,\ \sigma^2 ||\eta_j||^2}^{[\V^-, \V^+]}(\eta_j^Ty) &= 1-\alpha & F_{U_\alpha^j,\ \sigma^2 ||\eta_j||^2}^{[\V^-, \V^+]}(\eta_j^Ty) &= \alpha 
\end{align*}
Then $[L_\alpha^j, U_\alpha^j]$ is a $(1-\alpha)$ confidence interval for $\eta_j^T \mu$, conditional on $(\hat\E, \hat z_{\hat\E})$:
\begin{equation}
\label{eq:coverage}
\Pp \left( \beta^\star _{\hat M ,j} \in [L_{\alpha}^j, U_{\alpha}^j]\ \big|\ \{ (\hat\E, \hat z_{\hat\E}) = (\E, z_\E) \}  \right) = 1-\alpha.
\end{equation}
\end{corollary}


The above discussion has focused on constructing intervals for a single $j$. If we repeat the procedure for each $j \in \hat\E$, our intervals in fact control the false coverage rate (FCR) of \cite{benjamini2005false}.
\begin{corollary}
For each $j\in \hat \E$,
\begin{equation}
\label{eq:fcr}
\Pp \left( \beta^\star _{\hat M ,j}  \in [L_{\alpha}^j,U_{\alpha}^j]  \right) = 1-\alpha.
\end{equation}
Furthermore, the FCR of the intervals $\left\{[L_{\alpha}^j, U_{\alpha}^j]\right\}_{j\in \hat \E}$  is $\alpha$.
\end{corollary}

If $\eta^T y$ are not near the boundaries $[\V^-,\V^+]$, then the intervals will be relatively short. This is shown in Figure \ref{fig:intervals}. Figure \ref{fig:uncorrelated-intervals} shows two simulations that demonstrate our intervals cover at the nominal rate. We leave an exhaustive study of such intervals for the lasso to future work, noting that the truncation framework described can be used to form intervals with exact coverage properties.

\begin{figure}[!h]
\includegraphics[width = .48\textwidth]{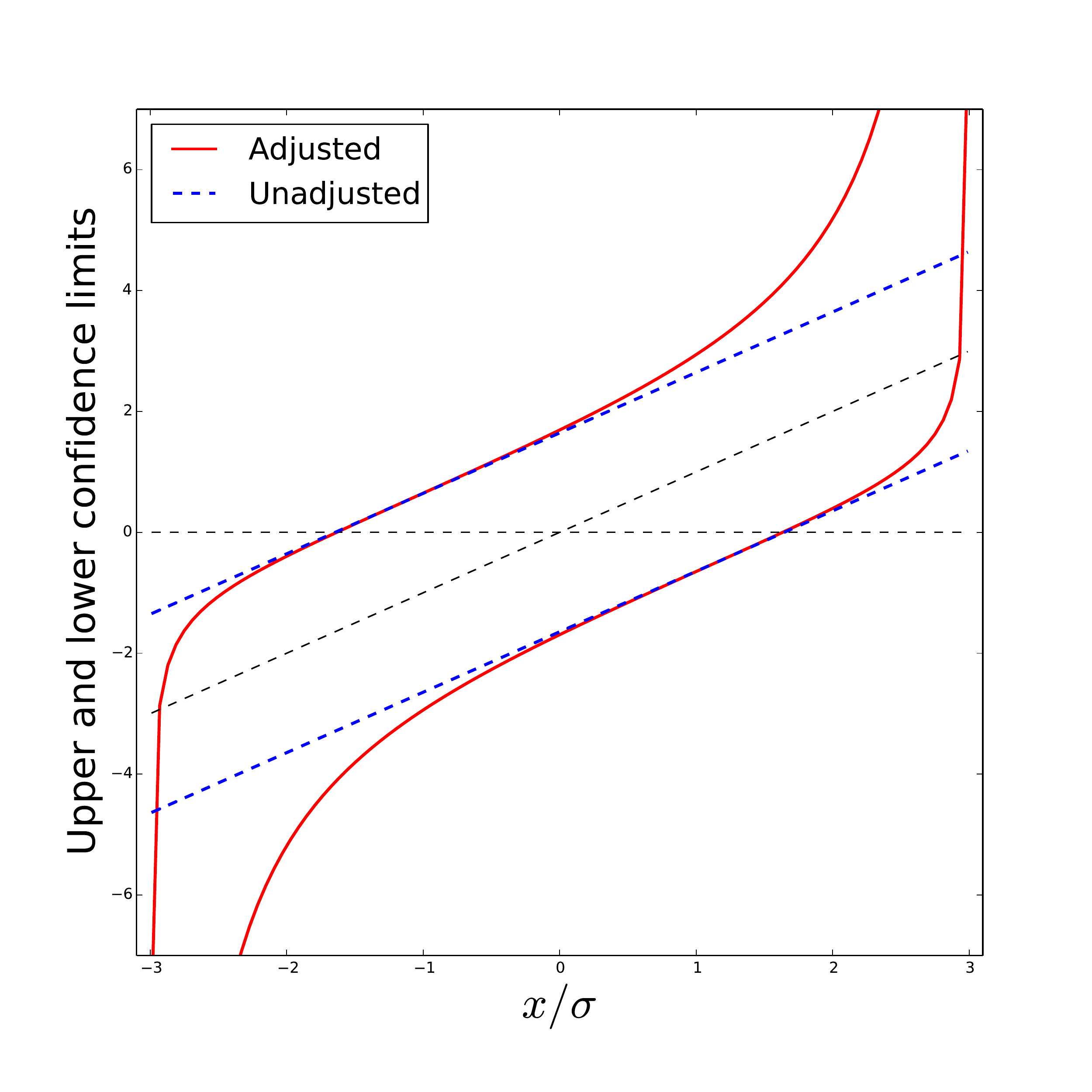}
\includegraphics[width = .48\textwidth]{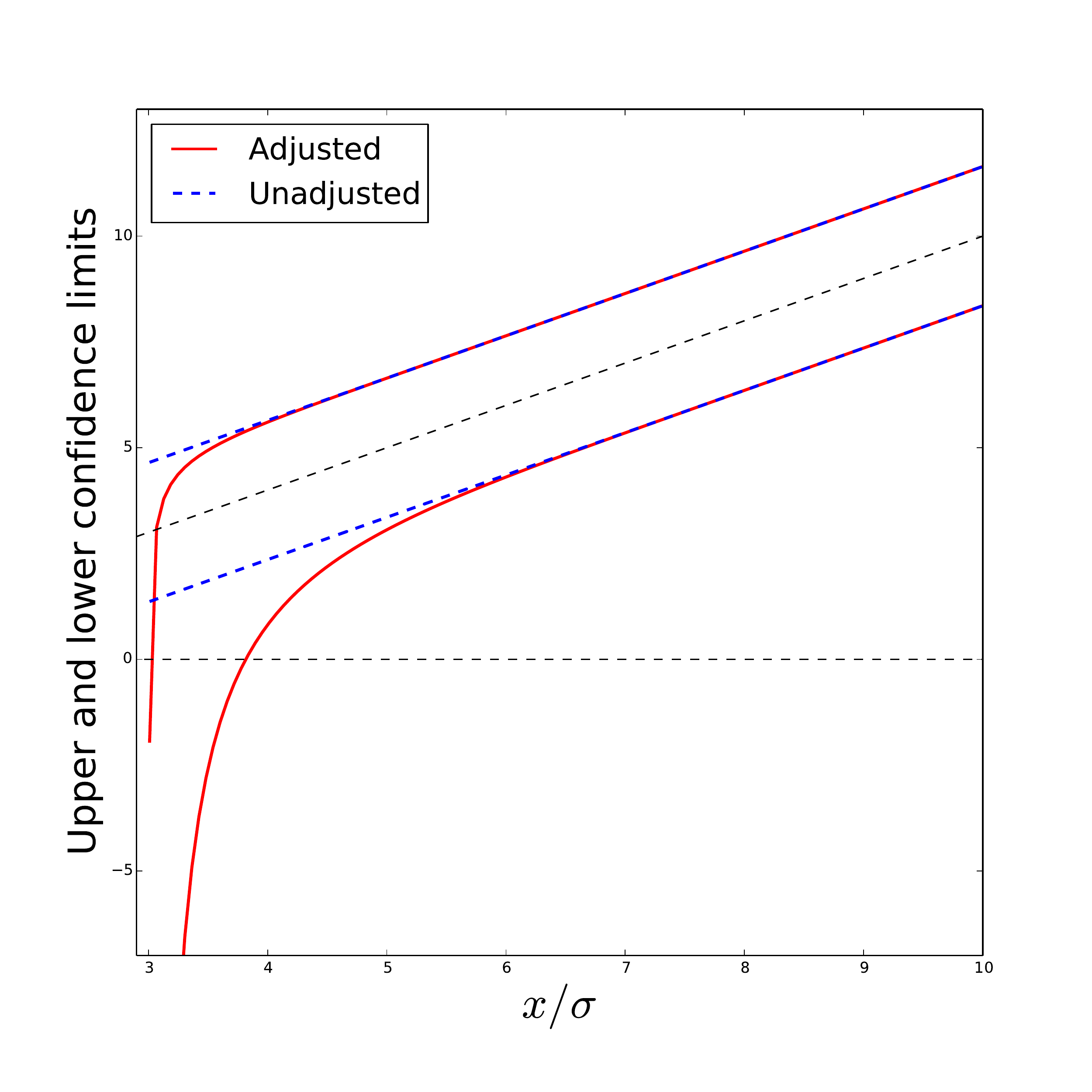}
\caption{Upper and lower bounds of 90\% confidence intervals based on $[a,b]=[-3\sigma,3\sigma]$ as a function of the observation $x/\sigma$. We see that as long as the observation $x/\sigma$ is roughly $0.5\sigma$ away from either boundary, the size of the intervals is comparable to an unadjusted confidence interval.}
\label{fig:intervals}
\end{figure}

\begin{figure}[!h]
\includegraphics[width = .48\textwidth]{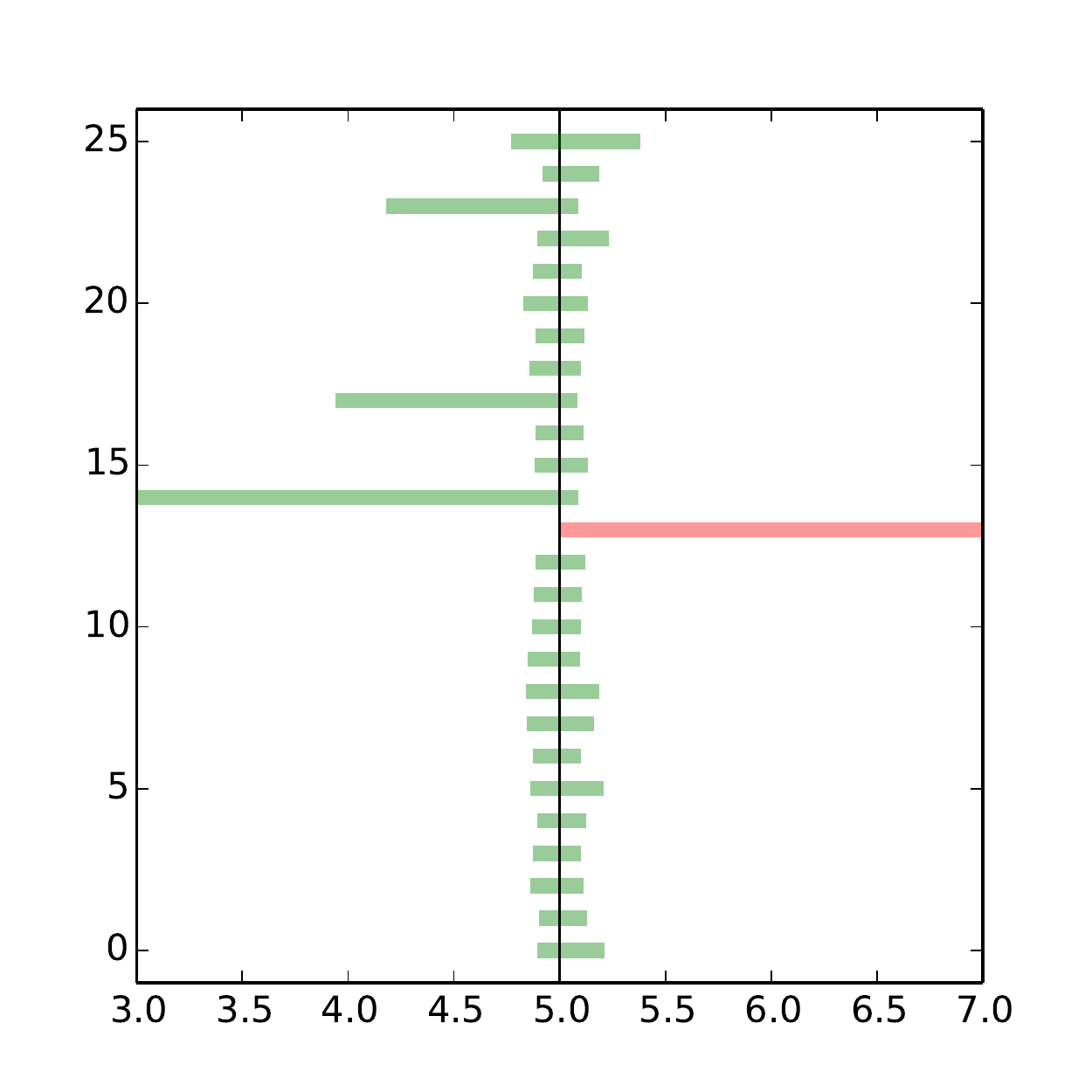}
\includegraphics[width = .48\textwidth]{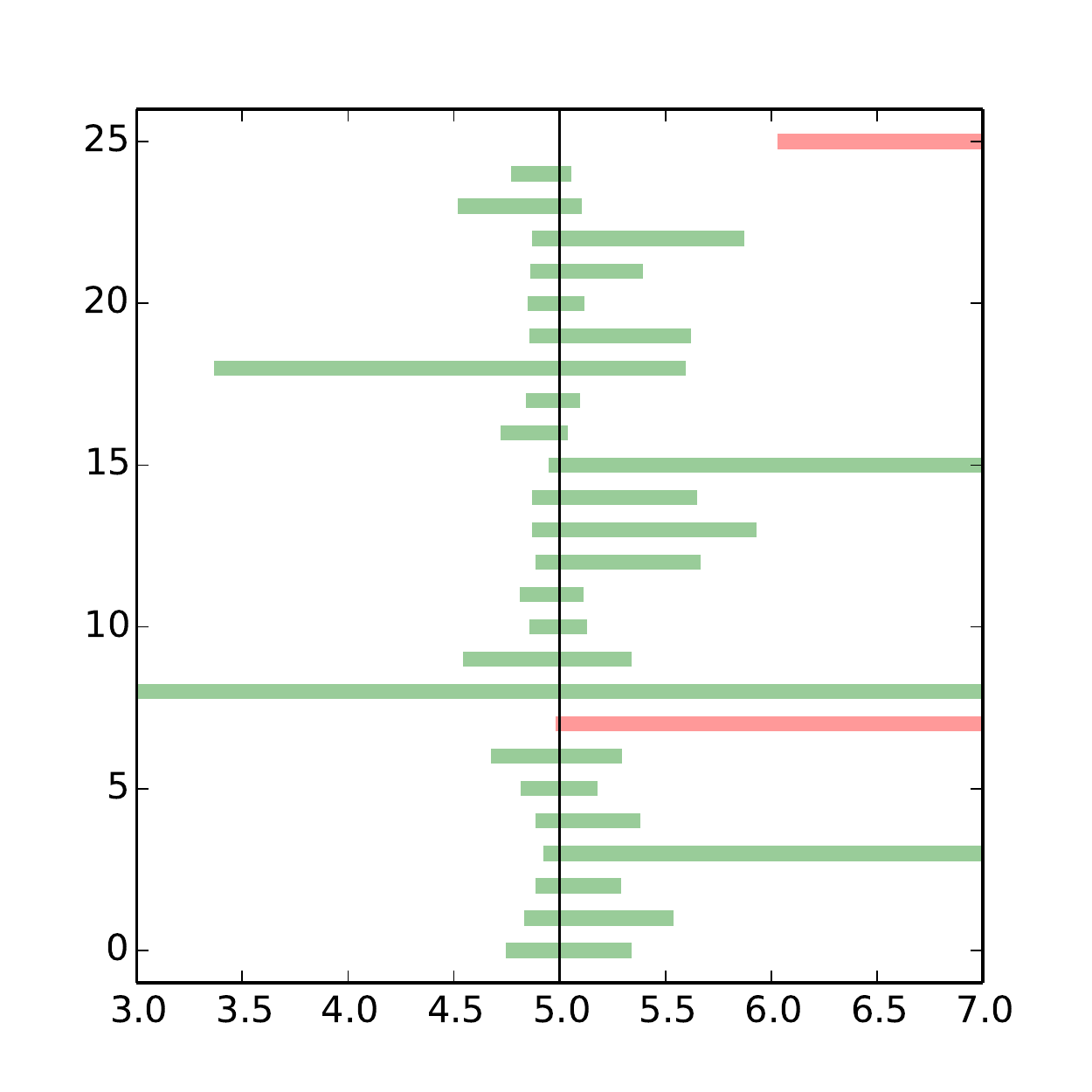}
\caption{90\% confidence intervals for $\eta_1^T\mu$ for a small ($n = 100,\,p=50$) and a large ($n=100,\,p=200$) uncorrelated Gaussian design, computed over 25 simulated data sets. The true model has five non-zero coefficients, all set to 5.0, and the noise variance is 0.25. A green bar means the confidence interval covers the true value while a red bar means otherwise.}
\label{fig:uncorrelated-intervals}
\end{figure}

\subsection{Testing the Lasso-Selected Model}
\label{sec:goodness}

Having observed that the lasso selected the variables $\hat\E$, another relevant question is whether it has captured all of the signal in the model, i.e.,
\begin{equation}
H_0: \beta^0_{-\hat\E} = \zeros.
\label{eq:simple_hyp}
\end{equation} 
We consider a slightly more general question, which does not assume the correctness of the linear model $\mu = X\beta^0$ and also takes into account whether the non-selected variables can improve the fit:
\begin{equation}
H_0: X_{-\hat \E}^T(I-P_{\hat \E}) \mu = \zeros.
\label{eq:hyp}
\end{equation} 
This quantity is the partial correlation of the non-selected variables with $\mu$, adjusting for the variables in $\hat\E$. This is more general because if we assume $\mu = X\beta^0$ for some $\beta^0$ and $X$ is full rank, then rejecting \eqref{eq:hyp} implies that there exists $i \in \text{supp}(\beta^0)$ not in $\hat\E$, so we would also reject \eqref{eq:simple_hyp}.
The natural approach is to compare the observed partial correlations $X_{-\E}^T (I - P_\E) y$ to $\zeros$. However, the framework of Section \ref{sec:truncated-gaussian-test} only allows tests of $\mu$ in a single direction $\eta$. To make use of that framework, we can choose $\eta$ such that it selects the maximum magnitude of $X_{-\E}^T (I - P_\E) y$. In particular, this direction provides the most evidence against the null hypothesis of zero partial correlation, so if the null hypothesis cannot be rejected in this direction, it would not be rejected in any direction. 

Letting $j^\star := \text{argmax}_j\ |e_j^T X_{-\E}^T(I - P_\E)y|$ and $s_{j} := \sign(e_{j}^T X_{-\E}^T(I - P_\E)y)$, we set 
\begin{equation}
\eta_{j^\star} = s_{j^\star} (I - P_\E) X_{-\E} e_{j^\star},
\label{eq:eta_test}
\end{equation}
and test $H_0: \eta_{j^\star}^T \mu = 0$. However, the results in Section \ref{sec:truncated-gaussian-test} cannot be directly applied to this setting because $j^\star$ and $s_{j^\star}$ are random variables that are not measurable with respect to $(\hat\E, \hat z_{\hat\E})$. 

To resolve this issue, we propose a test conditional not only on $(\hat\E, \hat z_{\hat\E})$, but also on the index and sign of the maximizer:
\begin{equation}
\{ (\hat\E, \hat z_{\hat\E}) = (\E, z_\E),\,(j^\star, s_{j^\star}) = (j,s)\}.
\label{eq:expanded_constraint}
\end{equation}
A test that is level $\alpha$ conditional on \eqref{eq:expanded_constraint} for all $(\E, z_\E)$ and $(j, s)$ is also level $\alpha$ conditional on $\{ (\hat\E, \hat z_{\hat\E}) = (\E, z_\E) \}$.

In order to use the results of Section \ref{sec:truncated-gaussian-test}, we must show that \eqref{eq:expanded_constraint} can be written in the form $A(\E, z_\E, j, s)y \leq b(\E, z_\E, j, s)$. This is indeed possible, and the following proposition provides an explicit construction.
\begin{proposition}
\label{prop:expanded_constraints}
Let $A_0, b_0, A_1, b_1$ be defined as in Proposition \ref{prop:A_b}. Then:
\begin{align*}
\{ (\hat\E, \hat z_{\hat\E}) = (\E, z_\E),\,(j^\star, s_{j^\star}) = (j,s)\}  = \left\{ \begin{pmatrix} A_0(\E, z_\E) \\ A_1(\E, z_\E) \\ A_2(\E, j, s)\end{pmatrix} y < \begin{pmatrix} b_0(\E, z_\E) \\ b_1(\E, z_\E) \\ \zeros \end{pmatrix}\right\}
\label{eq:laso-polytope}
\end{align*}
where $A_2(\E, j, s)$ is defined as 
$$ A_2(\E, j, s) = -s \begin{pmatrix} D_j(\E) \\ S_j(\E)\end{pmatrix}X_{-\E}^T (I - P_{\E}) $$
and $D_j$ and $S_j$ are $(|\E|-1) \times |\E|$ operators that compute the difference and sum, respectively, of the $j^\text{th}$ element with the other elements, e.g.,
\begin{align*}
D_1 &= \begin{pmatrix} 1 & -1 \\ 1 & & -1 \\ & & & \ddots \\ 1 & & & & -1 \end{pmatrix} & S_1 &= \begin{pmatrix} 1 & 1 \\ 1 & & 1 \\ & & & \ddots \\ 1 & & & & 1 \end{pmatrix}.
\end{align*}
\end{proposition}

\begin{proof}
The constraints $\{ A_0 y < b_0 \}$ and $\{A_1 y < b_1 \}$ come from Proposition \eqref{prop:A_b} and encode the constraints $\{(\hat\E, \hat z_{\hat\E}) = (\E, z_\E)\}$. We show that the last two sets of constraints encode $\{ (j^\star, s_{j^\star}) = (j,s)\}$. 

Let $r := X_{-\E}^T (I - P_\E) y$ denote the vector of partial correlations. If $s = +1$, then $|r_j| > |r_i|$ for all $i \neq j$ if and only if $r_j - r_i > 0$ and $r_j + r_i > 0$ for all $i \neq j$. We can write this as $D_j r > 0$ and $S_j r > 0$. If $s = -1$, then the signs are flipped: $D_j r < 0$ and $S_j r < 0$. This establishes 
$$\{ (j^\star, s_{j^\star}) = (j,s)\} = \left\{-s\begin{pmatrix} D_j \\ S_j \end{pmatrix} r < \zeros \right\} = \{ A_2 y < \zeros \}.$$
\end{proof}

Because of Proposition \ref{prop:expanded_constraints}, we can now obtain the following result as a simple consequence of Theorem \ref{thm:truncated-gaussian-pivot}, which says that $F_{0, \sigma^2 ||\eta_{j^\star}||^2}^{[\V^-, \V^+]}(\eta_{j^\star}^T y) \sim \unif(0,1)$, conditional on the set \eqref{eq:expanded_constraint} and $H_0$.
We reject when $ F_{0, \sigma^2 ||\eta_j^{*}||^2}^{[\V^-, \V^+]}(\eta_{j^\star}^T y)$ is large because $F_{0,\ \sigma^2 ||\eta_j^{*}||^2}^{[\V^-, \V^+]}(\cdot)$ is monotone increasing in the argument and $\eta_{j^*}^T\mu$ is likely to be positive under the alternative.
\begin{corollary}
\label{cor:test}
Let $H_0$ and $\eta_{j^\star}$ be defined as in \eqref{eq:eta_test}. Then, the test which rejects when 
$$ \left\{F_{0,\ \sigma^2 ||\eta_j^{*}||^2}^{[\V^-, \V^+]}(\eta_{j^\star}^T y) > 1 - \alpha \right\}$$
is level $\alpha$, conditional on $\{ (\hat\E, \hat z_{\hat\E}) = (\E, z_\E), (j^\star, s_{j^\star}) = (j,s)\}$. That is, 
\begin{gather*}
\Pp\left(F_{0,\ \sigma^2 ||\eta_{j^\star}||^2}^{[\V^-, \V^+]}(\eta_{j^\star}^T y) > 1-\alpha\ \big|\ \{ (\hat\E, \hat z_{\hat\E}) = (\E, z_\E),\, (j^\star, s_{j^\star}) = (j,s)\} \cap H_0 \right) = \alpha.\\
\intertext{In particular, since this holds for every $(\E, z_\E, j, s)$, this test also controls Type I error conditional only on $(\hat\E, \hat z_{\hat\E})$, and unconditionally:}
\Pp\left(F_{0,\ \sigma^2 ||\eta_{j^\star}||^2}^{[\V^-, \V^+]}(\eta_{j^\star}^T y) > 1-\alpha\ \big|\ \{ (\hat\E, \hat z_{\hat\E}) = (\E, z_\E)\} \cap H_0 \right) = \alpha \\
\Pp\left(F_{0,\ \sigma^2 ||\eta_{j^\star}||^2}^{[\V^-, \V^+]}(\eta_{j^\star}^T y) > 1-\alpha\ \big|\ H_0\right) = \alpha.
\end{gather*}
\end{corollary}
Figures \ref{fig:uncorrelated-p-values} and \ref{fig:correlated-p-values} show the results of four simulation studies that demonstrate that the p-values are uniformly distributed when $H_{0,\lambda}$ is true and stochastically smaller than $\unif(0,1)$ when it is false.

\begin{figure}[!h]
\includegraphics[width = .48\textwidth]{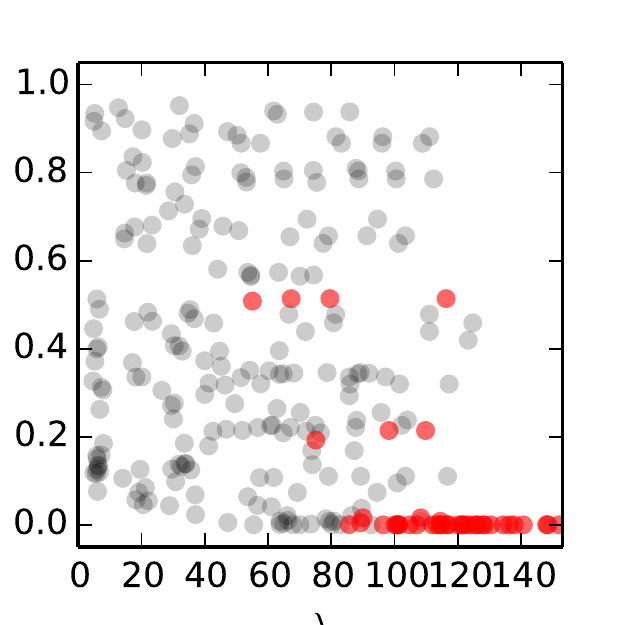}
\includegraphics[width = .48\textwidth]{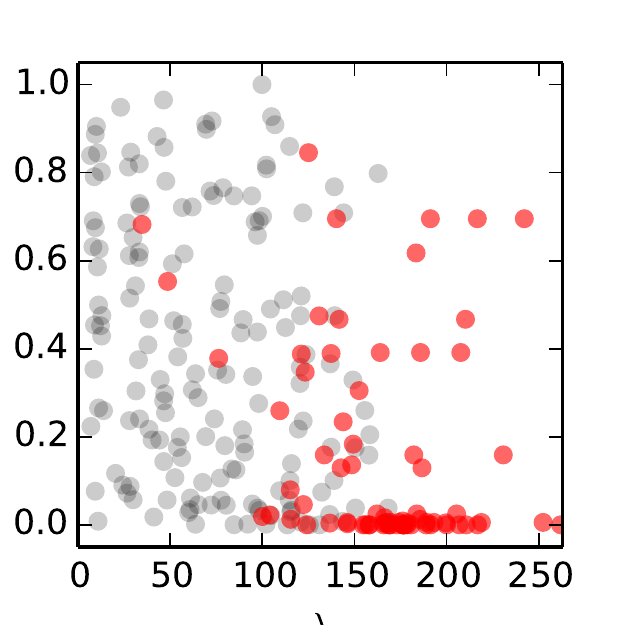}
\caption{P-values for $H_{0,\lambda}$ at various $\lambda$ values for a small ($n = 100,\,p=50$) and a large ($n=100,\,p=200$) uncorrelated Gaussian design, computed over 50 simulated data sets. The true model has three non-zero coefficients, all set to 1.0, and the noise variance is 2.0. We see the p-values are $\unif(0,1)$ when the selected model includes the truly relevant predictors (black dots) and are stochastically smaller than $\unif(0,1)$ when the selected model omits a relevant predictor (red dots).}
\label{fig:uncorrelated-p-values}
\end{figure}

\begin{figure}[!h]
\includegraphics[width = .48\textwidth]{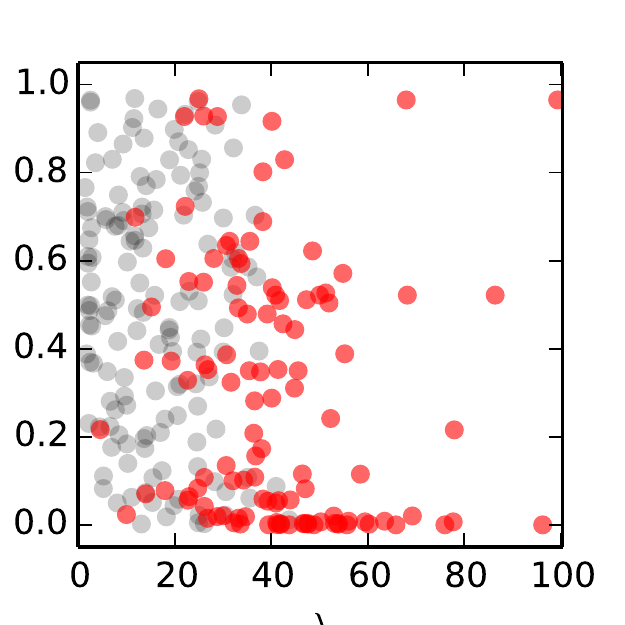}
\includegraphics[width = .48\textwidth]{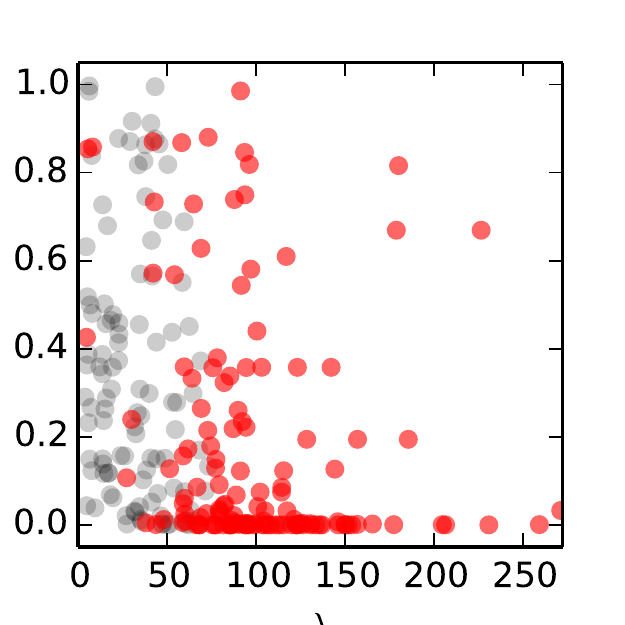}
\caption{P-values for $H_{0,\lambda}$ at various $\lambda$ values for a small ($n = 100,\,p=50$) and a large ($n=100,\,p=200$) \emph{correlated} ($\rho = 0.7$) Gaussian  design, computed over 50 simulated data sets. The true model has three non-zero coefficients, all set to 1.0, and the noise variance is 2.0. Since the predictors are correlated, the relevant predictors are not always selected first. However, the p-values remain uniformly distributed when $H_{0,\lambda}$ is true and stochastically smaller than $\unif(0,1)$ otherwise.}
\label{fig:correlated-p-values}
\end{figure}

\section{Data Example}
\label{sec:examples}

We illustrate the application of inference for the lasso to the diabetes data set from \citet{efron2004least}. First, all variables were standardized. Then, we chose $\lambda$ according to the strategy in \cite{negahban2012unified}, $\lambda = 2 \Expect(\|X^T\epsilon\|_{\infty})$, using an estimate of $\sigma$ from the full model, resulting in $\lambda \approx 190$. The lasso selected four variables: \verb\BMI\, \verb\BP\, \verb\S3\, and \verb\S5\. 

The intervals are shown in Figure \ref{fig:diabetes}, alongside the unadjusted confidence intervals produced by fitting OLS to the four selected variables, ignoring the selection. The latter is not a valid confidence interval conditional on the model. Also depicted are the confidence intervals obtained by \emph{data splitting}; that is, if one splits the $n$ observations into two halves, then uses one half for model selection and the other for inference. This is a competitor method that also produces valid confidence intervals conditional on the model. In this case, data splitting selected the same four variables, and the confidence intervals were formed based on OLS on the half of the data set not used for model selection.

We can make two main observations from Figure \ref{fig:diabetes}. 
\begin{enumerate}
\item The adjusted intervals provided by our method essentially reproduces the OLS intervals for the strong effects, whereas data splitting results in a loss of power by roughly a factor of $\sqrt{2}$ (since only $n/2$ observations are used in the inference). 
\item One variable, \verb\S3\, which would have been deemed significant using the OLS intervals, is no longer significant after adjustment. This demonstrates that taking model selection into account can have substantive impacts on the conclusions that are made.
\end{enumerate} 

\begin{figure}[!h]
\centering
\includegraphics[width=.72\textwidth]{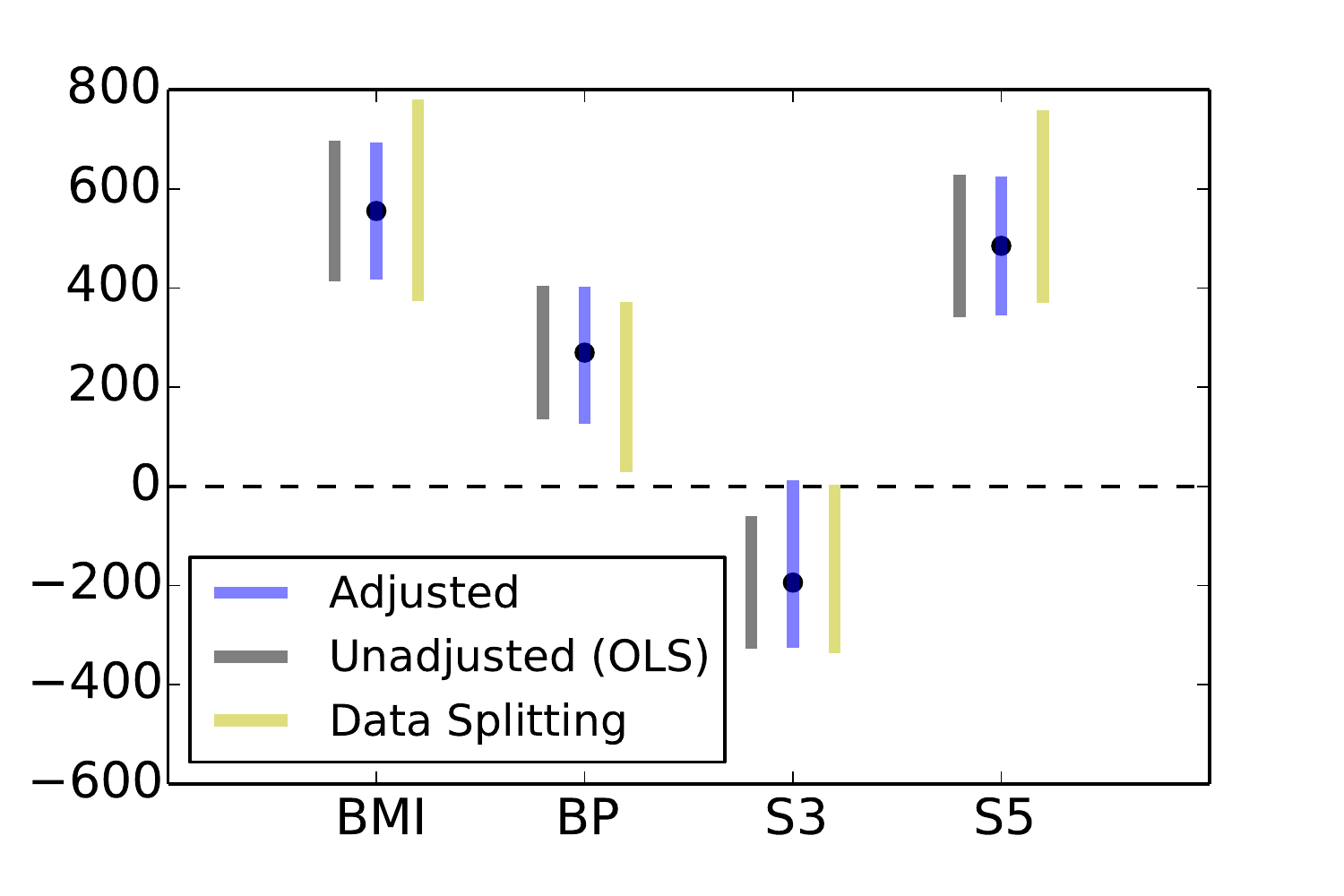}
\caption{Inference for the four variables selected by the lasso ($\lambda = 190$) on the diabetes data set. The point estimate and adjusted confidence intervals using the approach in Section \ref{sec:lasso} are shown in blue. The gray show the OLS intervals, which ignore selection. The yellow lines show the intervals produced by splitting the data into two halves, forming the interval based on only half of the data.}
\label{fig:diabetes}
\end{figure}

\section{Minimal Post-Selection Inference}
\label{sec:minimal}

We have described how to perform post-selection inference for the lasso conditional on both the active set and signs $\{(\hat{\E},\hat{z}_{\hat{\E}}) = (\E,z_\E)\}$. However, recall from Section \ref{sec:intro} that the goal was inference conditional solely on the model, i.e., $\{\hat\E = \E\}$. In this section, we extend our framework to this setting, which we call minimal post-selection inference because we condition on the minimal set necessary for the random $\eta$ to be measurable. This results in more precise confidence intervals at the expense of greater computational cost.

To this end, we note that $\{\hat \E = \E \}$ is simply 
\[\textstyle
\underset{z_\E \in \{-1, 1\}^{|E|}}\bigcup \{(\hat{\E},\hat{z}_{\hat{\E}}) = (\E,z_\E)\},
\]
where the union is taken over all choices of signs. Therefore, the distribution of $y$ conditioned on only the active set $\{\hat \E = \E \}$ is a Gaussian vector constrained to a union of polytopes 
\[\textstyle
y\ \big|\ \underset{z_\E \in \{-1, 1\}^{|E|}}\bigcup\{A(\E,z_\E) y \le b(\E,z_\E)\},
\]
where $A(\E,z_\E)$ and $b(\E,z_\E)$ are given by \eqref{prop:A_b}.

To obtain inference about $\eta^T\mu$, we follow the arguments in Section \ref{sec:truncated-gaussian-test} to obtain that this conditional distribution is equivalent to 
\BEQ\textstyle
\eta^T y\ \big|\ \underset{z_\E \in \{-1, 1\}^{|E|}}\bigcup \{\V_{z_\E}^-(y)\le\eta^Ty\le\V_{z_\E}^+(y),\V_{z_\E}^0(y)\ge 0\},
\label{eq:union-intervals}
\EEQ
where $\V_{z_\E}^-,\,\V_{z_\E}^+,\,\V_{z_\E}^0$ are defined according to \eqref{eq:v_minus}, \eqref{eq:v_plus}, \eqref{eq:v_zero} with $A = A(\E,z_\E)$ and $b = b(\E,z_\E)$. Moreover, all of these quantities are still independent of $\eta^T y$, so instead of having a Gaussian truncated to a single interval $[\V^-, \V^+]$ as in Section \ref{sec:truncated-gaussian-test}, we now have a Gaussian truncated to the union of intervals $\bigcup_{z_\E} [\V^-_{z_\E}, \V^+_{z_\E}]$. The geometric intuition is illustrated in Figure \ref{fig:union}.

\begin{figure}[!h]
\includegraphics[width = \textwidth]{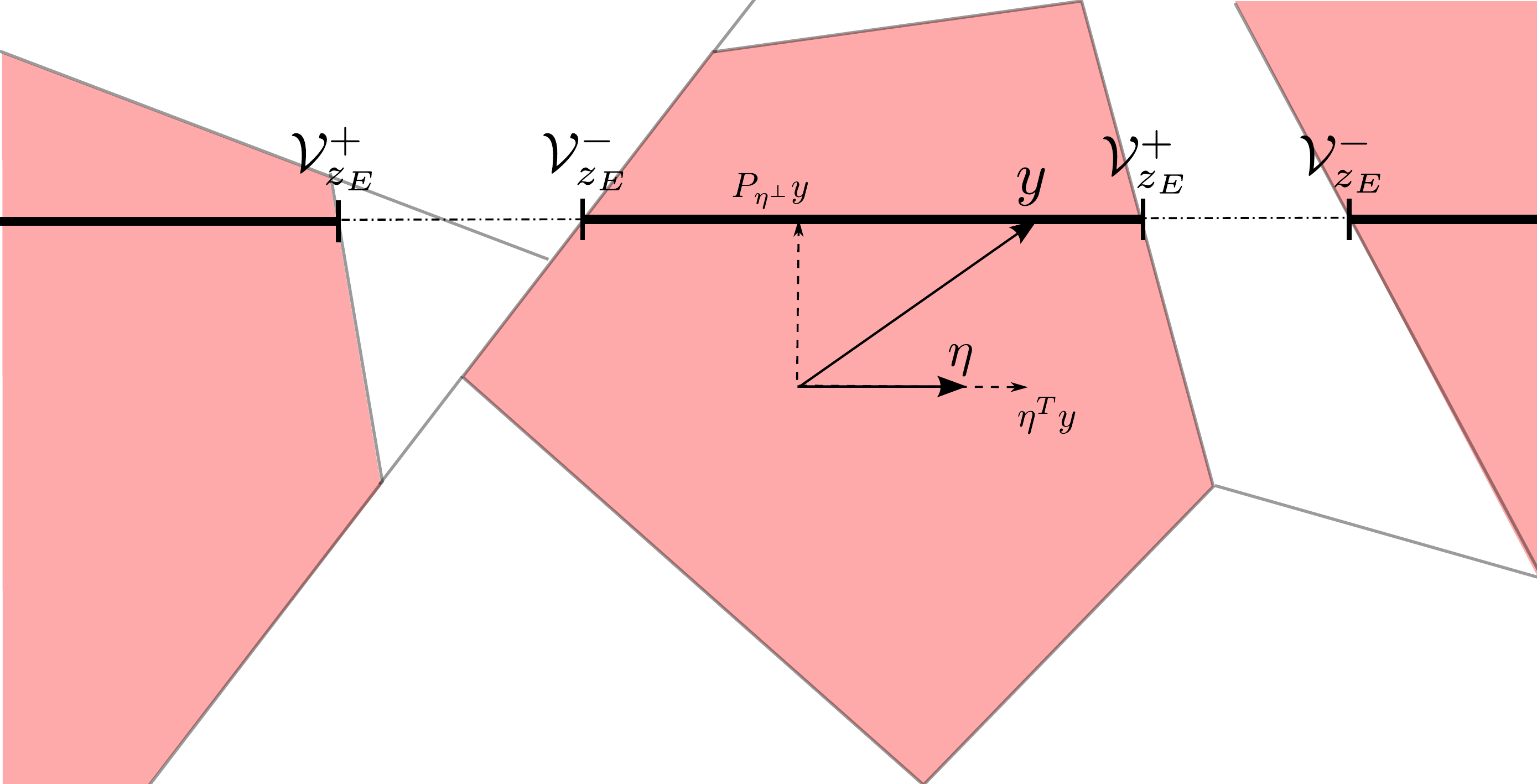}
\caption{A picture demonstrating the effect of taking a union over signs. The polytope in the middle corresponds to the $(\hat\E, \hat z_{\hat\E})$ that was observed and is the same polytope as in Figure \ref{fig:polytope}. The difference is that we now consider potential $(\E, z_\E)$ in addition to the one that was observed. The polytopes for the other $(\E, z_\E)$ which have the same active set $\hat\E$ are red. The conditioning set is the union of these polytopes. We see that for $y$ to be in this union, $\eta^T y$ must be in $\bigcup_{z_\E} [\V^-_{z_\E}, \V^+_{z_\E}]$. The key point is that all of the $\V^-_{z_\E}$ and $\V^+_{z_\E}$ are still functions of only $P_{\eta^\perp} y$ and so are independent of $\eta^T y$.}
\label{fig:union}
\end{figure}

Finally, the probability integral transform once again yields a pivot:
\begin{align*}
F_{\eta^T\mu,\ \eta^T\Sigma\eta}^{\bigcup_{z_E}[\V_{z_E}^-(y), \V_{z_E}^+(y)]}(\eta^Ty)\ \big|\ \{\hat\E = \E\} \sim \unif(0,1).
\end{align*}
It is now more useful to think of the notation of $F$ as indicating the truncation set $C \subset \R$:
\begin{equation}
F_{\mu, \sigma^2}^{C}(x) := \frac{\Phi( (-\infty, x] \cap C )}{ \Phi( C )},
\label{eq:normal-truncated-intervals-cdf}
\end{equation}
where $\Phi$ is the law of a $N(0, 1)$ random variable. We summarize these results in the following theorem. 

\begin{theorem}
\label{thm:minimal}
Let $F_{\mu, \sigma^2}^{\bigcup_i[a_i, b_i]}$ be the CDF of a normal truncated to the union of intervals $\bigcup_i[a_i, b_i]$, i.e., given by \eqref{eq:normal-truncated-intervals-cdf}. Then:
\begin{equation}
F_{\eta^T\mu,\ \eta^T\Sigma\eta}^{\bigcup_{z_\E}[\V_{z_\E}^-(y), \V_{z_\E}^+(y)]}(\eta^Ty)\ \big|\ \{\hat\E = \E\} \sim \unif(0,1),
\label{eq:minimal-pivotal-quantity}
\end{equation}
where $\V_{z_\E}^-(y)$ and $\V_{z_\E}^+(y)$ are defined in \eqref{eq:v_minus} and \eqref{eq:v_plus} with $A = A(\E,z_\E)$ and $b = b(\E,z_\E)$.
\end{theorem}

The derivations of the confidence intervals and hypothesis tests in Section \ref{sec:lasso} remain valid using \eqref{eq:minimal-pivotal-quantity} as the pivot instead of \eqref{eq:pivot-lasso}. Figure \ref{fig:minimal} illustrates the effect of minimal post-selection inference in a simulation study, as compared with the ``simple'' inference described previously. The intervals are similar in most cases, but one can obtain great gains in precision using the minimal intervals when the simple intervals are very wide.

\begin{figure}
\includegraphics[width=.5\textwidth]{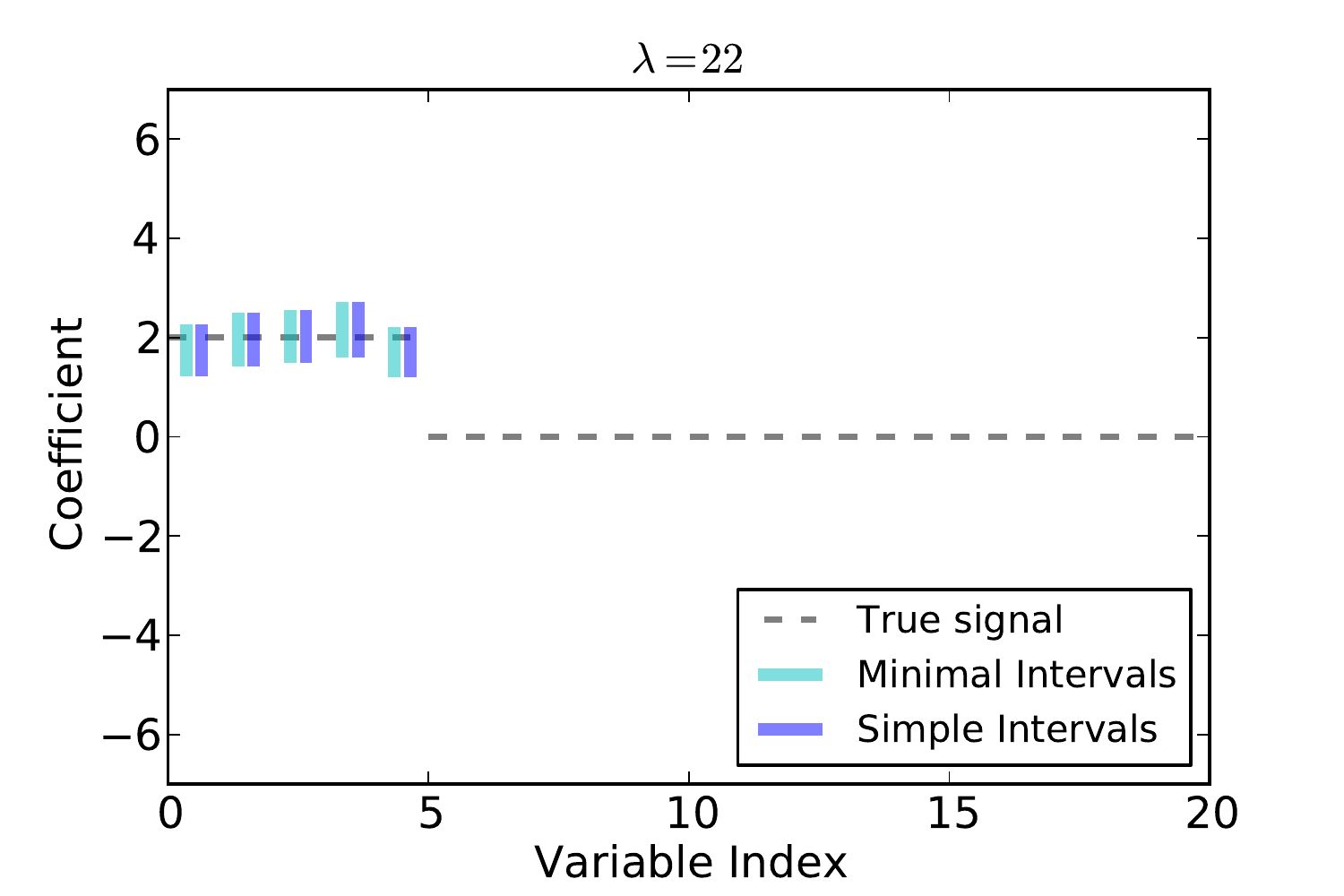}\includegraphics[width=.5\textwidth]{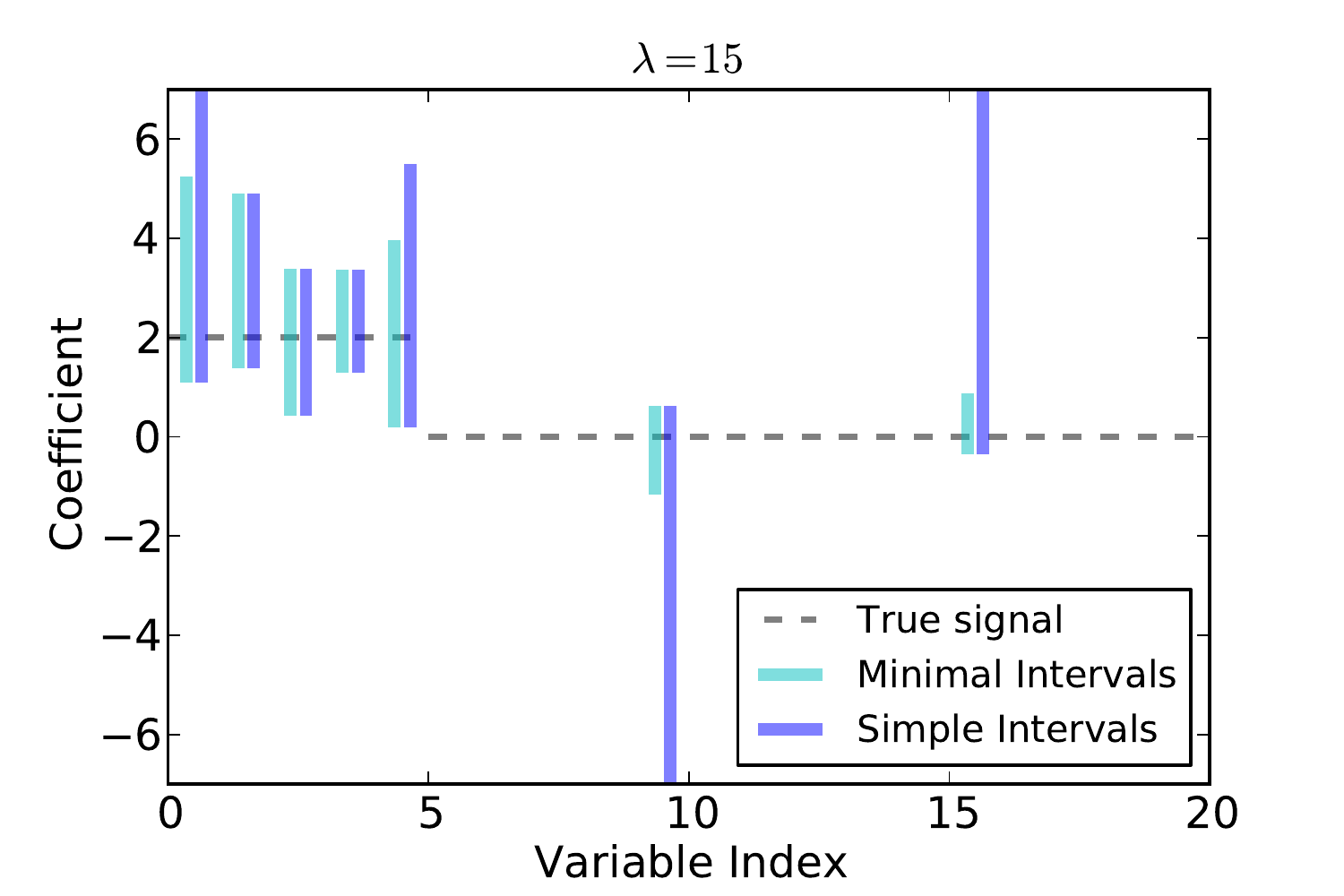}
\caption{Comparison of the minimal and simple intervals as applied to the same simulated data set for two values of $\lambda$. The simulated data featured $n=25$, $p=50$, and 5 true non-zero coefficients; only the first 20 coefficients are shown. (We have included variables with no intervals to emphasize that inference is only on the selected variables.) We see that the simple intervals are virtually as good as the minimal intervals most of the time; the advantage of the minimal intervals is realized when the estimate is unstable and the simple intervals are very long, as in the right plot.}
\label{fig:minimal}
\end{figure}

However, the tradeoff for this increased precision is greater computational cost. We computed $\V^-_{z_\E}$ and $\V^+_{z_\E}$ for all $z_\E \in \{ -1, 1 \}^{|\E|}$, which is only feasible when $|\E|$ is fairly small. In what follows, we revert to the simple intervals described in Section \ref{sec:lasso}, but extensions to the minimal inference setting are straightforward.

\section{Extensions}
\label{sec:extensions}

\subsection{Elastic net}
One problem with the lasso is that it tends to select only one variable out of a set of correlated variables, resulting in estimates which are unstable. The elastic net \citep{zou2005regularization} adds an $\ell_2$ penalty to the lasso objective in order to stabilize the estimates:
\begin{align}
\hat \beta^e = \underset{\beta}{\text{argmin}}\ \frac{1}{2} \norm{y-X \beta}_2^2 +\lambda \norm{\beta}_1+ \frac{\gamma}{2} \norm{\beta}_2 ^2.
\label{eq:elastic-net}
\end{align}
Using a nearly identical argument to the one in Section \ref{sec:lasso-selection}, we see that necessary and sufficient conditions for $\{ (\hat\E, \hat z_{\hat\E}) = (\E, z_\E) \}$ are the existence of $U(\E, z_\E)$ and $W(\E, z_\E)$ satisfying 
\begin{align*}
(X_\E^T X_\E + \gamma I) U - X_\E^T y + \lambda z_\E &= 0 \\
X_{-\E}^T X_\E U - X_{-\E}^T y + \lambda W &= 0 \\
\sign(U) = z_\E,\ W &\in (-1, 1).
\end{align*}
Solving for $U$ and $W$, we see that the selection event can be written 
\begin{equation}
\{ (\hat\E, \hat z_{\hat\E}) = (\E, z_\E) \} = \left\{ \begin{pmatrix} A_0(\E, z_\E) \\ A_1(\E, z_\E) \end{pmatrix} y < \begin{pmatrix} b_0(\E, z_\E) \\ b_1(\E, z_\E) \end{pmatrix}\right\}
\label{eq:enet_A_b}
\end{equation}
where $A_0$, $A_1$, $b_0$, and $b_1$ are the same as in Proposition \ref{prop:A_b}, except replacing $(X_\E^T X_\E)^{-1}$, which appears in the expressions through $P_\E$ and $(X_\E^T)^\dagger$, by the ``damped'' version $(X_\E^T X_\E + \gamma I)^{-1}$.

Having rewritten the selection event in the form \eqref{eq:enet_A_b}, we can once again apply the framework of Section \ref{sec:truncated-gaussian-test} to obtain a test for the elastic net conditional on this event.

\subsection{Alternative norms as test statistics}

In Section \ref{sec:goodness} we used the test statistic
$$
T_{\infty} = \|X_{-\hat{\E}}^T(I-P_{\hat\E})y\|_{\infty}
$$
and its conditional distribution on $\{(\hat\E, \hat z_{\hat\E}) = (\E, z_\E)\}$ to test whether we had missed
any large partial correlations in using $\hat{\E}$ as the estimated active set. If we have indeed missed
some variables in $\E$ there is no reason to suppose that the mean of
$X_{-\E}^T(I-P_{\E})y$ is sparse; hence the $\ell_{\infty}$ norm may not be the best norm to use as a test statistic.

In principle, we could have used virtually any norm, as long as we can say something about the distribution of this
norm conditional on  $\{(\hat\E, \hat z_{\hat\E}) = (\E, z_\E)\}$. Problems of this form are considered in \cite{kacrice}.
For example, if we consider the quadratic
$$
T_2 = \|X_{-\E}^T(I-P_{\E})y\|_{2}
$$
the general approach in \cite{kacrice} derives the conditional distribution of $T_2$ conditioned on
$$
\eta^*_2 = \argmax_{\|\eta\|_2 \leq 1} \eta^T(X_{-\E}^T(I-P_{\E})y).
$$
In general, this distribution will be a $\chi^2$ subject to random truncation as in Section \ref{sec:truncated-gaussian-test} (see the group lasso examples
in \cite{kacrice}). Adding the constraints encoded by $\{(\hat\E, \hat z_{\hat\E}) = (\E, z_\E)\}$ affects only the random truncation $[\V^-,\V^+]$.

\subsection{Estimation of $\sigma^2$}
\label{sec:estimate:sigma}

As noted above, all of our results rely on a reliable estimate of $\sigma^2$.  While there are several 
approaches to estimating $\sigma^2$ in the literature, the truncated Gaussian theory described in this work
itself provides a natural estimate.

Suppose the linear model is correct ($\mu=X\beta^0$). Then, on the event $\{\hat{\E}=\E, \hat{E} \supset S\}$, which we assume, the residual 
$$
(I-P_{\E})y
$$
is a (multivariate) truncated Gaussian with mean $\zeros$, with law 
$$
\Pp_{C,\sigma^2}(B) = \Pp(Z \in B | Z \in C), \qquad Z \sim N(\zeros, \sigma^2 I).
$$
As $\sigma^2$ , one obtains a one-parameter exponential family with density
$$
\frac{d\Pp_{C,\sigma^2}}{dz} = e^{-\alpha \|z\|^2_2 - \Lambda_{C}(\alpha)} 1_{C}(z)
$$
and natural parameter $\alpha = \sigma^2/2$. On the event $\{ (\hat\E, \hat z_{\hat\E}) = (\E, z_\E) \}$, we set 
$$C = \left\{y: A(\E,z_{\E})y \leq b(\E,z_{\E}) \right\},$$
and then choose $\alpha$ (or equivalently, $\sigma^2$) to satisfy the score equation
\begin{align}
\Ee_{C, \hat\sigma^2}(\|Z\|^2_2) = \|(I-P_{\E})y\|^2_2.
\label{eq:sigma}
\end{align}
This amounts to a maximum likelihood estimate of $\sigma^2$.
The expectation on the left is generally impossible to do analytically, but there exist fast
algorithms for sampling from $\Pp_{C,\sigma^2}$, c.f. \cite{truncnorm1,truncnorm2}. A rough outline of a
naive version of such algorithms is to pick a direction such as $e_i$ one of the coordinate axes. Based on the current state of
$Z$, draw a new entry for the $Z_i$ from the appropriate univariate truncated normal determined
from the cutoffs described in Section \ref{sec:truncated-gaussian-test}. We repeat this procedure to evaluate the expectation on the left, and use gradient descent to find $\hat \sigma^2$.

\subsection{Composite Null Hypotheses}

In Section \ref{sec:lasso}, we considered hypotheses of the form $H_0: \eta_{j^\star}^T \mu = 0$, which said that the partial correlation of the variables in $-\E$ with $y$, adjusting for the variables in $\E$, was exactly 0. This may be unrealistic, and in practice, we may want to allow some tolerance for the partial correlation.

We consider testing instead the \emph{composite} hypothesis 
\begin{equation}
\label{eq:composite}
H_0: | \eta_{j^\star}^T\mu | \leq \delta_0.
\end{equation}
The following result characterizes a test for $H_0$.

\begin{proposition}
The test which rejects when  $F_{\delta_0,\ \sigma^2||\eta_{j^\star}||^2}^{[\V^-, \V^+]}(\eta^T y) > 1-\alpha$ is exact level $\alpha$.
\end{proposition}

\begin{proof}
Let $\delta := \eta_{j^\star}^T \mu$. Define $\displaystyle T_{\delta_0} := \inf_{|\delta| \leq \delta_0} F_{\delta,\ \sigma^2||\eta_{j^\star}||^2}^{[\V^-, \V^+]}(\eta^T y)$. Then:
\begin{align*}
\text{Type I error} &:= \sup_{|\delta| \leq \delta_0} \Pp_\delta(T_{\delta_0} > 1-\alpha) \\
&\leq  \sup_{|\delta| \leq \delta_0} \Pp_\delta\left(F_{\delta,\ \sigma^2||\eta_{j^\star}||^2}^{[\V^-, \V^+]}(\eta^T y) > 1-\alpha\right)\\
&= \alpha 
\end{align*}
Next, we have that $T_{\delta_0} = F_{\delta_0,\ \sigma^2||\eta_{j^\star}||^2}^{[\V^-, \V^+]}(\eta^T y)$, i.e., the infimum is achieved at $\delta = \delta_0$, so calculating $T_{\delta_0}$ is a simple matter of evaluating $F_{\delta_0}$. This follows from the fact that $F_\delta$ is monotone decreasing in $\delta$ (c.f. Appendix \ref{appendix:monotone}). 

Finally, the Type I error is exactly $\alpha$ because the reverse inequality also holds:
\[ \text{Type I error} \geq  \Pp_{\delta_0}(T_{\delta_0} > 1-\alpha) = \alpha. \]
\end{proof}

Although the test is exact level $\alpha$, the significance level of a test for a composite null is a ``worst-case'' Type I error; for most values of $\mu$ such that $|\eta^T \mu| \leq \delta_0$, the Type I error will be less than $\alpha$, so the test will be conservative. Of course, what we lose in power, we gain in robustness to the assumption that $\eta^T \mu = 0$ exactly. 

%

\subsection{How long a lasso should you use?}
\label{sec:fwer}

Procedures for fitting the lasso, such as {\tt glmnet} \citep{friedman2010regularization}, solve \eqref{eq:lasso} for a decreasing sequence of $\lambda$ values starting from $\lambda_1 = \|X^T y\|_{\infty}$. The framework developed so far provides a means to decide when to stop along the regularization path, i.e., when the lasso has done enough ``fitting.''  In this section, we describe a path-wise testing procedure for the lasso,

The path-wise procedure is simple. At each value of $\lambda$:
\BNUM
\item Solve the lasso and obtain an active set $\hat{\E}_{\lambda}$ and signs $\hat{z}_{\hat{\E}_\lambda}$.
\item Test $H_{0, \lambda}:  X_{\hat E_{\lambda}}^T (I - P_{\hat E_{\lambda}})(\mu) = 0$ at level $\alpha$. Rather than being conditional on only $(\hat\E_\lambda, \hat z_{\hat\E_\lambda})$, this test is conditional on the entire sequence of active sets and signs $\{(\hat\E^m, \hat z^m) = (\E^m, z^m)\}$, as we describe below.
\ENUM
As $\lambda$ decreases, we expect to reject the null hypotheses as the fit improves and stop once the first null hypothesis has been accepted.

To understand the properties of this procedure, we formalize it as a multiple testing problem. For each value $\lambda_1, ..., \lambda_m$, we test $H_{0, \lambda_i}$. We test these hypotheses sequentially and stop after the first hypothesis has been accepted. Implicitly, this means that we accept all the remaining hypotheses.

Our next result shows that this procedure controls the family-wise error rate (FWER) at level $\alpha$. Let $V$ denote that number of false rejections. Then FWER is defined as $\Pp(V \geq 1)$. The practical implication of this result is the model selected by this procedure will be larger than the true model with probability $\alpha$.

\begin{proposition}
The path-wise testing procedure controls FWER at level $\alpha$.
\end{proposition}

\begin{proof}

Let $\hat\E^m$ and $\hat z^m$ denote the complete sequence of active sets and signs at $\lambda_1,\ldots,\lambda_m$, i.e.,
\begin{align*}
\hat\E^m &= \{\hat\E_{\lambda_1},\dots,\hat\E_{\lambda_m}\} \\
\hat z^m &= \{\hat z_{\hat\E_{\lambda_1}}, \dots, \hat z_{\hat\E_{\lambda_m}}\}.
\end{align*}

We seek to control the family-wise error rate (FWER) when testing the hypotheses $H_{0,\lambda_1},\dots,H_{0,\lambda_m}$, i.e., $\mathbb{P}(V \ge 1)$. We partition the space over all possible sequences $\hat\E^m$ and $\hat z^m$:
$$
\Pp(V \ge 1) =\sum_{(\E^m, z^m)} \Pp\left( V \ge 1\ \big|\ (\hat\E^m, \hat{z}^m) = (\E^m, z^m)\right) \Pp\left((\hat\E^m, \hat{z}^m) = (\E^m, z^m)\right).
$$
Since $\sum_{(\E^m, z^m)} \Pp\left((\hat\E^m, \hat{z}^m) = (\E^m, z^m)\right)=1$, we can ensure $\fwer\le\alpha$ by ensuring
$$
\Pp\left( V \ge 1\ \big|\ (\hat\E^m, \hat{z}^m) = (\E^m, z^m)\right) \le \alpha\text{ for any } (\E^m, z^m).
$$

Let $\lambda_{k}$ denote the first $\lambda_i$ for which $H_{0, \lambda_i}$ is true. Then the event $V \ge 1$ is equivalent to the event that we reject $H_{0,\lambda_{k}}$ because the preceding hypotheses $H_{0,\lambda_1},\dots,H_{0,\lambda_{k-1}}$ are all false so we cannot make a false discovery before the $k^\text{th}$ hypothesis. Thus
$$
\Pp\left( V \ge 1\ \big|\ (\hat\E^m, \hat{z}^m) = (\E^m, z^m)\right) = \Pp\left(\text{reject }H_{0,\lambda_k}\ \big|\ (\hat\E^m, \hat{z}^m) = (\E^m, z^m)\right).
$$
Therefore, we can control FWER at level $\alpha$ by ensuring 
$$
\Pp\left(\text{reject }H_{0,\lambda}\ \big|\ (\hat\E^m, \hat z^m) = (\E^m, z^m) \right) \le \alpha
$$
for each $\lambda \in \{\lambda_1,\dots,\lambda_k\}$. 
\end{proof}

To perform a test of $H_{0,\lambda}$ conditioned on $\{(\hat\E^m, \hat z^m) = (\E^m, z^m)\}$, we apply the framework of Section \ref{sec:truncated-gaussian-test}. Let 
$$
\{ A(\E_i,s_i)y < b(\E_i,s_i) \}
$$
be the affine constraints that characterize the event $\{(\hat{\E}_{\lambda_i},\hat{z}_{\lambda_i}) = (\E_i, z_i)\}$ from Proposition \ref{prop:A_b}. The event $\{(\hat\E^m, \hat z^m) = (\E^m, z^m)\}$ is equivalent to the intersection of all of these constraints:
\begin{align*}
\underbrace{\BMAT A(\E_1, z_1) \\ \vdots \\ A(\E_m, z_m)\EMAT}_{A(\E^m, z^m)}y < \underbrace{\BMAT b (\E_1, z_1)\\ \vdots \\ b(\E_m, z_m)\EMAT}_{b(\E^m, z^m)}.
\end{align*}
Now Theorem \ref{thm:truncated-gaussian-pivot} applies, and we can obtain the usual pivot as a test statistic.

\section{Conclusion}

We have described a method for making inference about $\eta^T \mu$ in the linear model based on the lasso estimator, where $\eta$ is chosen adaptively after model selection. The confidence intervals and tests that we propose are conditional on $\{ (\hat\E, \hat z_{\hat\E}) = (\E, z_\E) \}$. In contrast to existing procedures on inference for the lasso, we provide a pivot whose conditional distribution can be characterized exactly (non-asymptotically). This pivot can be used to derive confidence intervals and hypothesis tests based on lasso estimates anywhere along the solution path, not necessarily just at the knots of the LARS path as in \citet{lockhart2012significance}. Finally, our test is computationally simple: the quantities required to form the test statistic are readily available from the solution of the lasso.


\section{Appendix}

%
%

\subsection{Monotonicity of $F$}
\label{appendix:monotone}

\begin{lemma}
Let $F_{\mu}(x) := F_{\mu, \sigma^2}^{[a,b]}(x)$ denote the cumulative distribution function of a truncated Gaussian random variable, as defined as in \eqref{eq:U}. Then $F_\mu(x)$ is monotone decreasing in $\mu$. 
\end{lemma}

\begin{proof}
First, the truncated Gaussian distribution with CDF $F_{\mu} := F_{\mu, \sigma^2}^{[a,b]}$ is a natural exponential family in $\mu$, since it is just a Gaussian with a different base measure. Therefore, it has monotone likelihood ratio in $\mu$. That is, for all $\mu_1 > \mu_0$ and $x_1 > x_0$:
$$ \frac{f_{\mu_1}(x_1)}{f_{\mu_0}(x_1)} > \frac{f_{\mu_1}(x_0)}{f_{\mu_0}(x_0)}$$
where $f_{\mu_i} := dF_{\mu_i}$ denotes the density. (Instead of appealing to properties of exponential families, this property can also be directly verified.)

This implies 
\begin{align*}
f_{\mu_1}(x_1) f_{\mu_0}(x_0) &> f_{\mu_1}(x_0) f_{\mu_0}(x_1) & x_1 > x_0.
\end{align*}
Therefore, the inequality is preserved if we integrate both sides with respect to $x_0$ on $(-\infty, x)$ for $x < x_1$. This yields:
\begin{align*}
\int_{-\infty}^{x} f_{\mu_1}(x_1) f_{\mu_0}(x_0)\,dx_0 &> \int_{-\infty}^{x} f_{\mu_1}(x_0) f_{\mu_0}(x_1)\,dx_0 & x < x_1 \\
f_{\mu_1}(x_1) F_{\mu_0}(x) &> f_{\mu_0}(x_1) F_{\mu_1}(x) & x < x_1
\end{align*}
Now we integrate both sides with respect to $x_1$ on $(x, \infty)$ to obtain:
\begin{align*}
(1 - F_{\mu_1}(x)) F_{\mu_0}(x)  &> (1 - F_{\mu_0}(x)) F_{\mu_1}(x)
\end{align*}
which establishes $F_{\mu_0}(x) > F_{\mu_1}(x)$ for all $\mu_1 > \mu_0$.
\end{proof}

\section{Lasso Screening Property}
\label{appendix:screening}
In this section, we state some sufficient conditions that guarantee $\text{support}(\beta^0) \subset \text{support}(\hat \beta)$. Let $\E=\text{support}(\beta^0)$ and $\hat \E \subset\text{support}(\hat \beta)$. The results of this section are well known in the literature and can be found in \cite[Chapter 2.5]{buhlmann2011statistics}.
\begin{definition}[Restricted Eigenvalue Condition]
Restricted eigenvalue condition requires that $X$ satisfy
$$
\norm{Xv}_2 ^2 \ge m \norm{v}^2 _2 
$$
for all $v \in \{ x: \norm{x_{-\E}}_1 \le 3 \norm{x_{\E}} \}$.
\end{definition}
\begin{definition}[Beta-min Condition]
The beta-min condition requires that for all $j \in \E$,
$$
|\beta^0 _j|> \beta_{min}.
$$
\end{definition}
\begin{theorem}
Let $y= X\beta^0 +\epsilon$, where $\epsilon$ is subgaussian with parameter $\sigma$, and $\hat \beta$ be the solution to \ref{eq:lasso} with $\lambda=4 \sigma\sqrt{\frac{\log p}{n}}$. Assume that $X$ satisfies the restricted eigenvalue condition, $\beta^0$ satisfies the beta-min condition with $\beta_{min} =\frac{8\sigma}{m} \sqrt{\frac{s\log p}{n}}$ , and $X$ is column normalized, $\norm{x_j}_2 \le \sqrt{n}$. Then $M \subset \hat \E$.
\end{theorem}
\begin{proof}
From \cite[Corollary 2]{negahban2012unified}, 
$$
\big\|\hat \beta- \beta^0\big\|_2 \le \frac{8\sigma}{m} \sqrt{\frac{s\log p}{n}}.
$$
Assume that their is a $j$ such that $j \in \E$, but $j \not \in \hat \E$. We must have
\begin{align*}
\big\|\hat \beta - \beta^0\big\|_2 > |\beta^0 _j | \ge \beta_{min}=\frac{8\sigma}{m} \sqrt{\frac{s\log p}{n}}.
\end{align*}
This is a contradiction, so for all $j \in \E$ we have $j \in \hat \E$.
\end{proof}

Next we provide a geometric proof of Lemma \ref{lem:conditional} which will be useful in the next chapter.
\begin{lemma}
	\label{lem:conditional-geometric}
	The conditioning set can be rewritten in terms of $\eta^T y$ as follows:
	\[  \{Ay \leq b\} = \{\V^-(y) \leq \eta^T y \leq \V^+(y), \V^0(y) \geq 0 \} \]
	where 
	\begin{align}
	\alpha &= \frac{A\Sigma\eta}{\eta^T\Sigma\eta} \label{eq:alpha} \\
	\V^- = \V^-(y) &= \max_{j:\ \alpha_j < 0} \frac{b_j - (Ay)_j + \alpha_j\eta^T y}{\alpha_j} \label{eq:v_minus-geometric} \\
	\V^+ = \V^+(y) &= \min_{j:\ \alpha_j > 0} \frac{b_j - (Ay)_j + \alpha_j\eta^T y}{\alpha_j}. \label{eq:v_plus-geometric} \\
	\V^0 = \V^0(y) &= \min_{j:\ \alpha_j = 0} b_j - (Ay)_j \label{eq:v_zero}
	\end{align}
	Moreover, $(\V^+, \V^-, \V^0)$ are independent of $\eta^T y$.
\end{lemma}
\begin{proof}
	Although the proof of Lemma \ref{lem:conditional-geometric} is elementary, the geometric picture gives more intuition as to why $\V^+$ and $\V^-$ are independent of $\eta^T  y$. Since $\Sigma$ is assumed known, let $\tilde y = \Sigma^{-\frac{1}{2}} y$ so that $\tilde y \sim N(\Sigma^{-\frac{1}{2}}\mu, I)$. We can decompose $\tilde y$ into two independent components: a one-dimensional component along $\tilde\eta := \Sigma^{\frac{1}{2}}\eta$ and a $(p-1)$-dimensional component orthogonal to $\tilde\eta$:
	\[ \tilde y = \tilde y_{\tilde\eta} + \tilde y_{\tilde\eta^\perp}. \]
	From Figure \ref{fig:polytope}, it is clear that the extent of the set $\{A y \leq b\} = \{A\Sigma^{\frac{1}{2}}\tilde y \leq b \}$ (i.e., $\V^+$ and $\V^-$) along the direction $\tilde\eta$ depends only on $\tilde y_{\tilde\eta^\perp}$ and is hence independent of $\tilde\eta^T \tilde y = \eta^T  y$. We present a geometric derivation below.
	The values $\V^+$ and $\V^-$ are the maximum and minimum possible values of $\tilde\eta^T \tilde y $, holding $\tilde y _{\tilde\eta^\perp}$ fixed, while remaining inside the polytope $A\Sigma^{\frac{1}{2}}\tilde y  \leq b$. Writing $\tilde y  = c\tilde\eta + \tilde y _{\tilde\eta^\perp}$ where $c$ is allowed to vary, $\V^+$ and $\V^-$ are the optimal values of the optimization problems:
	\begin{align*}
	\text{max. / min.}&\ \ \ \tilde\eta^T \tilde y  = c||\tilde\eta||_2^2 \\
	\text{subject to}&\ \ \ A\Sigma^{\frac{1}{2}}(c\tilde\eta + \tilde y _{\tilde\eta^\perp}) \leq b
	\end{align*}
	Rewriting this problem in terms of the original variables $\eta$ and $y$, we obtain:
	\begin{align*}
	\text{max. / min.}&\ \ \ c (\eta^T \Sigma \eta) \\
	\text{subject to}&\ \ \ c(A\Sigma\eta) \leq b - Ay + \frac{A\Sigma\eta}{\eta^T \Sigma\eta}\eta^Ty
	\end{align*}
	Since $c$ is the only free variable, we see from the constraints that the optimal values $\V^+$ and $\V^-$ are precisely those given in \eqref{eq:v_minus} and $\eqref{eq:v_plus}$.
\end{proof}

\chapter{Condition-on-Selection Method}
\label{chap:selection-additional}
In the previous chapter, we focused on selective inference for the sub-model coefficients selected by the lasso by conditioning on the event that lasso selects a certain subset of variables.  However the procedure we developed is not restricted to the sub-model coefficients, nor is it restricted to the lasso. In \cite{lee2014exact}, we used the same Condition-on-Selection (COS) method for  marginal screening, orthogonal matching pursuit, and screening+lasso variable selection methods.

In this chapter, we first discuss some definitions and formalism, which will help us understand how to generalize the results of Chapter \ref{chap:sel-lasso} to other selection procedures. In Section \ref{sec:formalism}, we see that the COS method results in tests that control the selective type 1 error. Then in Section \ref{sec:other-affine-selection}, we show how the selection events for several variable selection methods such as marginal screening, and orthogonal matching pursuit are affine in the response $y$. For non-affine selection events, we propose a general algorithm in Section \ref{sec:general-method}. We then describe inference for the full model regression coefficients, provide a method for FDR control and establish the asymptotic coverage property in the high-dimensional setting in Section \ref{sec:full-model}. Finally in Section \ref{sec:knockoff}, we show how to construct selectively valid confidence intervals for regression coefficients selected by the knockoff filter \citep{foygel2014controlling}.
\section{Formalism}
\label{sec:formalism}
This section closely follows the development in \cite{fithian2014optimal}, which in turn uses the COS method developed in earlier works \cite{lee2014exact,lee2013exact,taylor2014post}. Our main result of this section is to show that tests constructed using the COS method control selective type 1 error. This is the original motivation of \cite{lee2014exact,lee2013exact} for designing tests with the COS method.

We start off by defining a valid test in the classical setting.
\begin{definition}[Valid test]
Let $H \in \cH$ be a hypothesis, and $\phi(y; H) \in \{0,1\}$ is a test of $H$ meaning we reject $H$ if $\phi(y;H)=1$. $\phi(y;H)$ is a valid test of $H$ if
\[
\Pp_{F} \left( \phi(y;H) =1\right) \le \alpha
\]
for all  $F$ null with respect to $H$, meaning $F \in N_H$ , where $N_H$ is the set of distributions null with respect to $H$.

\end{definition}

For selective inference, there is an analog of type 1 error.
\begin{definition}[Selective Type 1 Error ]
	$\phi(y, H(y))$ is a valid test of the hypothesis $H(y)$ if it controls the selective type 1 error,
	\[
	\Pp_F \left(\phi(y;H(y)) = 1 \mid  F \in N_{H(y)} \right) \le \alpha.
	\]	
	\label{def:sel-type-1}
\end{definition}

The framework laid out in Chapter \ref{chap:sel-lasso} proposes controlling the selective type 1 error via the COS method. As we showed in the case of confidence intervals for regression coefficients and goodness-of-fit tests, by conditioning on the lasso selection event, we are guaranteed to control the conditional type 1 error by design, and this implies the control of the unconditional type 1 error. We now show that this is not specific to the lasso; in fact controlling the conditional type 1 error always controls the unconditional type 1 error in Definition \ref{def:sel-type-1}.

\begin{definition}
	Let $\cH$ be the hypothesis space. The selection algorithm $H : \reals^n \to \cA$ maps data to hypothesis. This induces the selection event $S(H) = \{ y:  H(y) =H\}$.
\end{definition}

The following definition motivates the construction in Equation \eqref{eq:coverage}.
\begin{definition} [Condition-on-Selection method]
	A test $ \phi$ is constructed via the Condition-on-Selection (COS) method if for all $F \in N_{H_i}$
		\begin{equation}
	\Pp_F \left(  \phi(y;H_i)=1 \mid H(y) =H_i  \right).
	\label{eq:conditional-type-1}
	\end{equation}
	This means that $\phi(y; H_i)$ controls the conditional type 1 error rate.
\end{definition}

By a simple generalization of the argument in Theorem \ref{thm:unconditional-pivot}, we show that using the COS method to design a conditional test \ref{eq:conditional-type-1} implies control of the selective type 1 error \ref{def:sel-type-1}.
\begin{theorem}[Selective Type 1 Error control]
	A test constructed using the COS method, \ie\ satisfies \eqref{eq:conditional-type-1}, controls the selective type 1 error meaning
		$$\Pp_F \left(\phi(y;H(y)) = 1 \mid  F \in N_{H(y)} \right) \le \alpha.$$
\end{theorem}
\begin{proof}
	\begin{align*}
&	\Pp_F ( \phi(y; H(y)) =1 \mid F \in N_{H(y)}) = \sum_{i=1}^{|\cH|} \Pp ( \phi(y; H(y))=1, H(y)=H_i)\mid F \in N_{H(y)})\\
&	= \sum_{i: F \in N_{H_i} } \Pp( \phi(y; H(y))=1, H(y) = H_i \mid F \in N_{H(y)})+\\
&\sum_{i: F \not\in N_{H_i} } \Pp( \phi(y; H(y))=1, H(y) = H_i \mid F \in N_{H(y)})\\
&	= \sum_{i: F \in N_{H_i}} \Pp( \phi(y; H(y))=1, H(y) = H_i \mid F \in N_{H(y)})+0\\
&=\sum_{i: F \in N_{H_i}} \Pp( \phi(y; H(y))=1 \mid F \in N_{H(y)} , H(y) =H_i) \Pp (  H(y) =H_i \mid F \in N_{H(y)}) \\
&= \sum_{i: F \in N_{H_i} } \Pp ( \phi( y; H_i) =1 \mid H(y) = H_i) \Pp ( H(y) = H_i \mid F \in N_{H(y)} ) \\
&= \sum_{i: F \in N_{H_i}} \alpha \Pp ( H(y) =H_i \mid F \in N_{H(y)}) \\
&\le  \alpha.
	\end{align*}
where all of the previous probabilities are with respect to the distribution $F$. The first equality is the law of total probability, and the second equality is breaking the sum over disjoint sets. Since $ F \not \in N_{H_i} $, implies   $ \Pp_F ( H(y) =H_i \mid F \in N_{H(y)} )=0$, so $ \sum_{i: F \not \in H_i} \Pp ( \phi(y; H(y))= 1 , H(y) = H_i \mid F \in N_{H(y)} ) =0 $, which establishes the third equality. The fourth equality is the definition of conditional probability, and the fifth follows from noticing that $\{ F \in N_{H(y)} , H(y) =H_i , F \in N_{H_i} \} = \{ H(y) =H_i , F \in N_{H_i}\}$. The sixth equality uses the COS property of $\phi$: $ \Pp_F ( \phi(y; H_i) =1 \mid H(y) =H_i)  \le \alpha$ for any $F \in N_{H_i}$. Finally, the result follows since probabilities sum to less than or equal to 1.
\end{proof}
This result allows us to interpret the tests constructed via the COS method as unconditionally valid.  
\section{Marginal Screening, Orthogonal Matching Pursuit, and other Variable Selection methods}
\label{sec:other-affine-selection}
In lieu of the developments of the previous section, it is clear that the COS method developed for affine selection events in Chapter \ref{chap:sel-lasso} is not specific to the lasso. By changing the variable selection method, we are simply changing the selection algorithm and the selection event. The main work is in characterizing the selection event $\{ y : \hat M(y) =M\}$, the event that the variable selection methods chooses the subset $M$. In this section, we characterize the selection event for several variable selection methods: marginal screening, orthogonal matching pursuit (forward stepwise), non-negative least squares, and marginal screening+lasso. 
\subsection{Marginal Screening}
\label{sec:marg-selection-event}
In the case of marginal screening, the selection event $\hat \E(y)$ corresponds to the set of selected variables $\hat \E$ and signs $s$: 
\begin{align}
&\hat \E(y) =\left\{y: \text{sign}(x_i ^T y) x_i ^T y > \pm x_j^T y \text{ for all $i \in \hat \E$ and $j \in \hat \E^c$} \right \} \nonumber\\
&=\left\{ y:\hat s_i x_i^T y > \pm x_j ^T y \text{ and } \hat s_i x_i^T y \ge 0 \text{ for all $i \in \hat \E$ and $j \in \hat \E^c$}\right\} \nonumber\\
&=\left\{y: A(\hat \E,\hat s)y \le 0\right \}
\label{eq:A-b-defn}
\end{align}
for some matrix $A(\hat \E,\hat s)$.

\subsection{Marginal screening + Lasso}
The marginal screening+Lasso procedure was introduced in \cite{fan2008sure} as a variable selection method for the ultra-high dimensional setting of $p=O(e^{n^k})$. Fan et al. \cite{fan2008sure} recommend applying the marginal screening algorithm with $k= n-1$, followed by the Lasso on the selected variables. This is a two-stage procedure, so to properly account for the selection we must encode the selection event of marginal screening followed by Lasso. This can be done by representing the two stage selection as a single event. Let $(\hat \E_m, \hat s_m)$ be the variables and signs selected by marginal screening, and the $(\hat \E_L, \hat z_L)$ be the variables and signs selected by Lasso. In Proposition 2.2 of \cite{lee2013exact}, it is shown how to encode the Lasso selection event $(\hat \E_L, \hat z_L)$ as a set of constraints $\{ A_L y \le b_L\}$ \footnote{The Lasso selection event is with respect to the Lasso optimization problem after marginal screening.}, and in Section \ref{sec:marg-selection-event} we showed how to encode the marginal screening selection event $(\hat \E_m, \hat s_m)$ as a set of constraints $\{A_m y \le b_m\}$. Thus the selection event of marginal screening+Lasso can be encoded as $\{ A_L y \le b_L, A_m y \le 0\}$. 

\subsection{Orthogonal Matching Pursuit}
Orthogonal matching pursuit (OMP) is a commonly used variable selection method \footnote{OMP is sometimes known as forward stepwise regression.}. At each iteration, OMP selects the variable most correlated with the residual $r$, and then recomputes the residual using the residual of least squares using the selected variables. The description of the OMP algorithm is given in Algorithm \ref{alg:omp}.

\begin{algorithm}
	\caption{Orthogonal matching pursuit (OMP)}
	\begin{algorithmic}[1]
		\State \textbf{Input:} Design matrix $X$, response $y$, and model size $k$.
		\State \textbf{for}: $i=1$ to $k$
		\State \quad  $p_i = \arg \max_{j=1,\ldots,p} |r_i ^T x_j|$.
		\State \quad  $\hat S_i =\cup_{j=1}^i  \ \{p_i\}$.
		\State \quad $r_{i+1} = (I- X_{\hat S_i} X_{\hat S_i} ^{\dagger} ) y$.
		\State \textbf{end for}
		\State \textbf{Output}: $\hat S :=\{p_1, \ldots, p_k\}$, and $\hat \beta_{\hat S} = (X_{\hat S} ^T X_{\hat S} )^{-1} X_{\hat S}^T y$
	\end{algorithmic}
	\label{alg:omp}
\end{algorithm}

The OMP selection event as a set of linear constraints on $y$.
\begin{align*}
\hat \E(y) &=\left\{y:\text{sign}(x_{p_i}^T r_i  )x_{p_i}^T r_i > \pm x_{j}^T r_i\text{, for all } j\neq p_i \text{ and all $i \in [k]$}   \right\}\\
&=\{ y:\hat s_i x_{p_i}^T(I-X_{\hat \E_{i-1}}X_{\hat \E_{i-1}}^\dagger)y  > \pm x_{j}^T(I-X_{\hat \E_{i-1}}X_{\hat \E_{i-1}}^\dagger)y \text{ and }     \\
& \hat s_i x_{p_i}^T(I-X_{\hat \E_{i-1}}X_{\hat \E_{i-1}}^\dagger)y>0 \text{, for all }  j\neq p_i \text{, and all $i \in [k]$ }\}\\
&=\left\{y: A(\hat \E_1,\ldots, \hat \E_k,\hat s_1,\ldots, \hat s_k)  \le b(\hat \E_1,\ldots, \hat \E_k,\hat s_1,\ldots, \hat s_k)\right\}.
\end{align*}
The selection event encodes that OMP selected a certain variable and the sign of the correlation of that variable with the residual, at steps $1$ to $k$. The primary difference between the OMP selection event and the marginal screening selection event is that the OMP event also describes the order at which the variables were chosen. The marginal screening event only describes that the variable was among the top $k$ most correlated, and not whether a variable was the most correlated or $kth$ most correlated.

\subsection{Nonnegative Least Squares}
Non-negative least squares (NNLS) is a simple modification of the linear regression estimator with non-negative constraints on $\beta$:
\begin{align}
\arg \min_{\beta: \beta \ge 0} \frac{1}{2} \norm{y-X\beta}^2 .
\label{eq:nnls}
\end{align}
Under a positive eigenvalue conditions on $X$, several authors \cite{slawski2013non,meinshausen2013sign} have shown that NNLS is comprable to the Lasso in terms of prediction and estimation errors. The NNLS estimator also does not have any tuning parameters, since the sign constraint provides a natural form of regularization. NNLS has found applications when modeling non-negative data such as prices, incomes, count data. Non-negativity constraints arise naturally in non-negative matrix factorization, signal deconvolution, spectral analysis, and network tomography; we refer to \cite{chen2009nonnegativity} for a comprehensive survey of the applications of NNLS.

We show how our framework can be used to form exact hypothesis tests and confidence intervals for NNLS estimated coefficients. The primal dual solution pair $(\hat \beta, \hat \lambda)$ is a solution iff the KKT conditions are satisfied, 
\begin{align*}
\hat \lambda_i :=-x_i^T (y-X\hat\beta) &\ge 0 \text{ for all i}\\
\hat \beta &\ge 0.
\end{align*}
Let $\hat \E =\{i: -x_i^T (y-X\hat\beta)=0\}$. By complementary slackness $\hat \beta_{-\hat \E} =0$, where $-\hat \E$ is the complement to the ``active" variables $\hat \E$ chosen by NNLS. Given the active set we can solve the KKT equation for the value of $\hat \beta_{\hat \E}$,
\begin{align*}
-X_{\hat \E}^T (y- X\hat \beta) =0\\
-X_{\hat \E} ^T (y- X_{\hat \E} \hat \beta_{\hat \E}) =0\\
\hat \beta_{\hat \E} = X_{\hat \E} ^\dagger y,
\end{align*}
which is a linear contrast of $y$. The NNLS selection event is 
\begin{align*}
\hat \E(y)&=\{y: X_{\hat \E} ^T (y-X\hat \beta) =0,\ X_{-\hat \E}^T (y-X\hat \beta) >0\}\\
&=\{y: X_{\hat \E}^T (y-X\hat \beta) \ge 0, -X_{\hat \E}^T (y-X \hat \beta) \ge 0, X_{-\hat \E}^T (y-X\hat \beta) >0\}\\
&=\{y: X_{\hat \E}^T (I- X_{\hat \E} X_{\hat \E}^\dagger)y \ge 0, -X_{\hat \E}^T(I- X_{\hat \E} X_{\hat \E}^\dagger)y \ge 0, X_{-\hat \E}^T (I- X_{\hat \E} X_{\hat \E}^\dagger)y >0\}\\
&=\{y: A(\hat \E) y \le 0\}.
\end{align*}
The selection event encodes that for a given $y$ the NNLS optimization program will select a subset of variables $\hat \E(y)$. 

\subsection{Logistic regression with Screening}
The focus up to now has been on the linear regression estimator with additive Gaussian noise. In this section, we discuss extensions to conditional MLE (maximum likelihood estimator) such as logistic regression. This section is meant to be speculative and non-rigorous; our goal is only to illustrate that these tools are not restricted to the linear regression. A future publication will rigorously develop the inferential framework for conditional MLE.

Consider the logistic regression model with loss function and gradient below,
\begin{align*}
\ell(\beta) &= \frac{1}{n}\left(-y^T X\beta +\sum_{i=1}^n \log ( 1+e^{\beta^T x_i})\right)\\
\nabla\ell (\beta) &= -\frac{1}{n}X^T (y- s(X\beta)),
\end{align*}
where $s(X\beta)$ is the sigmoid function applied entrywise. By taylor expansion, the empirical estimator is given by
\begin{align*}
\hbeta &\approx \beta^0 - \left(\nabla^2\ell(\beta^0)\right)^{-1} \nabla \ell(\beta^0)\\
&=\beta^0 + \left(\nabla^2\ell(\beta^0)\right)^{-1}X^T (y- s(X\beta))
\end{align*}

By the Lindeberg CLT (central limit theorem), $\frac{1}{\sqrt{n}}X^T (y- s(X\beta^0))\to\cN(0,\Expect(\nabla^2 \ell(\beta^0)))$, and thus $w:=\frac{1}{\sqrt{n}} X^T y$ converges to a Gaussian. The marginal screening selection procedure can be expressed as a set of inequalities $\{ \text{sign}(w_i)w_i \ge \pm w_j, i \in \hat \E, j \in \hat \E^c \}= \{Aw \le b\}$. Thus conditional on the selection, $w$ is approximately a constrained Gaussian. The framework in Chapter \ref{sec:truncated-gaussian-test} and \ref{sec:formalism} can be applied to $w$, instead of $y$, to derive hypothesis tests and confidence intervals for the coefficients of logistic regression. The resulting test and confidence intervals should be correct asymptotically. However, this is the best we can expect for logistic regression and other conditional MLE because even in the classical case the Wald test is only asymptotically correct. For other conditional maximum likelihood estimator similar reasoning applies, since the  gradient $\nabla \ell(\beta)$ converges in distribution to a Gaussian.

For logistic regression with $\ell_1$ regularizer, $
 \frac{1}{n}\left(-y^T X\beta +\sum_{i=1}^n \log ( 1+e^{\beta^T x_i})\right) +\lambda \norm{\beta}_1$ the selection event cannot be analytically described. However, the COS method can still be applied using the general method presented in Chapter \ref{sec:general-method}.
\section{General method for Selective inference}
\label{sec:general-method}
In this section, we describe a computationally-intensive algorithm for finding selection events, when they are not easily described analytically. 

We first review the construction used in Chapter \ref{chap:sel-lasso} for affine selection events. Let $\Pind (y)= (I - \frac{\Sigma\eta\eta^T}{\eta^T \Sigma \eta}) y $. Recall that $y$ can be decomposed into two independent components $y= (\eta^Ty) \frac{\Sigma \eta}{\eta^T \Sigma \eta }  +(I-\frac{\Sigma\eta\eta^T}{\eta^T \Sigma \eta})y$. This is derived by defining $ \tilde y= \Sigma^{-1/2} y  \sim \cN(0,  I)$ and $\tilde \eta = \Sigma^{1/2} y $. $\tilde y $ can  be orthogonally decomposed as $\tilde y =(\tilde \eta^T \tilde y )  \frac{\tilde \eta}{\norm{\tilde \eta}} +( I - \frac{\tilde \eta \tilde \eta^T}{\norm{\tilde \eta}^2}) \tilde y $, so 
\begin{align*}
y = \Sigma^{1/2} \tilde y =(\tilde \eta^T \tilde y )  \frac{\Sigma^{1/2}\tilde \eta}{\norm{\tilde \eta}} +\Sigma^{1/2}( I - \frac{\tilde \eta \tilde \eta^T}{\norm{\tilde \eta}^2}) \tilde y \\
= (\eta^Ty )\frac{\Sigma \eta}{\eta^T \Sigma \eta} + (I - \frac{\Sigma \eta\eta^T}{\eta^T \Sigma \eta}) y .
\end{align*}

Lemma \ref{lem:conditional} shows that 
\[
\eta^Ty | \{ Ay \le b, \Pind y=y_0 \} \sim TN(\eta^T \mu , \sigma^2 \norm{\eta}^2, \V^-(y_0, A,b), \V^+ (y_0, A,b)).
\]

We can generalize this result to arbitrary selection events, where the selection event is not explicitly describable. Recall that $H$ is a selection algorithm that maps $\reals^n \to \mathcal{H}$. The selection event is $S(H) = \{ x: H(x) = H\}$, so $y \in S(H) \text{ iff } H(y) =H$. In the upcoming section, it will be convenient to work with the definition using $H(\cdot)$, since the set $ S(H)$ cannot be described, but the function $H(\cdot)$ can be efficiently computed. Thus we can only verify if a point $ y \in S(H)$.

The following Theorem is a straightforward generalization of Theorem \ref{thm:truncated-gaussian-pivot} from polyhedral sets to arbitrary sets $S$.
\begin{theorem}[Arbitrary selection events]
Let $y$ be a multivariate truncated normal, so $L(y) \propto e(-\frac1{2} (y-\mu)^T \Sigma^{-1}(y-\mu) ) \ones(y \in S(H))$. Then 
\begin{align*}
\eta^Ty | \{y \in S(H), \Pind y =y_0 \} \overset{d}{=} TN(\eta^T\mu, \eta^T \Sigma \eta, U(H,y_0, \frac{\Sigma \eta}{\eta^T \Sigma \eta}))
\end{align*}
and  $U(H,y_0,\frac{\Sigma \eta}{\eta^T \Sigma \eta}))=\{c: H(y_0+c \frac{\Sigma \eta}{\eta^T \Sigma \eta})=H \}$.

\label{thm:arbitrary-selection-event}
\end{theorem}

\begin{proof}
We know that $\eta^Ty | \{y \in S(H), \Pind y =y_0 \} \overset{d}{=} TN(\eta^T\mu, \norm{\eta}^2, U(H,y_0,\frac{\Sigma \eta}{\eta^T \Sigma \eta}))$, so $\eta^Ty |\{y \in S(H), \Pind y =y_0 \}$ is a univariate normal truncated to some region $U$. The goal is to check that $U(H,y_0,\Sigma \eta)=\{c: H(y_0+c \frac{\Sigma \eta}{\eta^T \Sigma \eta})=H \}$. 
We can describe the conditioning set as
\begin{align*}
&\{y:y \in S(H), \Pind y = y_0  \} = \{ y:H(y) =H, \Pind y =y_0 \}\\
&= \{y=y_0 + c  \frac{\Sigma \eta}{\eta^T \Sigma \eta} :H(y_0 + c  \frac{\Sigma \eta}{\eta^T \Sigma \eta} )=H, \Pind y = y_0\}\\
&= \{ y= y_0 + c   \frac{\Sigma \eta}{\eta^T \Sigma \eta} : \Pind y = y_0, c \in U(H,y_0,c  \frac{\Sigma \eta}{\eta^T \Sigma \eta}) \}\\
&= \{ y : \Pind y = y_0, \eta^Ty  \in U(H,y_0,c  \frac{\Sigma \eta}{\eta^T \Sigma \eta}) \}
\end{align*}
Thus we have that 
\begin{align*}
\left[\eta^Ty | \{y \in S(H), \Pind y =y_0 \}\right] &\overset{d}{=} \left[\eta^Ty | \eta^Ty \in U(H,y_0,\frac{\Sigma \eta}{\eta^T \Sigma \eta}), \Pind y =y_0 \}\right] \\
&\overset{d}{=} \left[\eta^Ty | \eta^Ty \in  U(H,y_0,\frac{\Sigma \eta}{\eta^T \Sigma \eta}) \}\right]\\
& \sim TN(\eta^T\mu , \eta^T \Sigma \eta,  U(H,y_0,\frac{\Sigma \eta}{\eta^T \Sigma \eta}))
\end{align*}
where the second equality follows from independence of $\eta^Ty$ and $\Pind y$.
\end{proof}

\subsection{Computational Algorithm for arbitrary selection algorithms}
In this section, we study the case of where the set $S(H)$ cannot be explicitly described, but the function $H(\cdot)$ is easily computable. Our goal will be to approximately compute the p-value $F(\eta^Ty; \eta^T \mu, U(H,y_0,\frac{\Sigma \eta}{\eta^T \Sigma \eta}))$, where $F$ is the cdf of $TN(\eta^T\mu , \eta^T \Sigma \eta,  U(H,y_0,\frac{\Sigma \eta}{\eta^T \Sigma \eta})$.

Algorithm \ref{alg:approx-pvalue} is the primary contribution of this section. This allows us to compute the pivotal quantity for algorithms $H(\cdot)$ with difficult to describe selection events. This includes linear regression with the SCAD/MCP regularizers, and logistic regression with $\ell_1$-regularizer, where the selection events do not have analytical forms.

Let $\tilde \phi(z;  \nu,  \sigma) = \phi(\frac{z-\nu }{\sigma})$ be the pdf of a univariate truncated normal with mean $\nu$ and variance $\sigma^2$.
\begin{algorithm}
	\caption{Compute approximate p-value}    
	\label{alg:approx-pvalue}
\begin{algorithmic}
	\State {\bf Input: } Grid points $D=\{d_1, \ldots, d_n\}$ and empty set $C = \emptyset$
	\State {\bf Output: } Approximate p-value $p$
	 \ForAll{ $d_i \in D$}
	 \State Compute $H_i= H(y_0+d \frac{\Sigma \eta}{\eta^T \Sigma \eta})$.
	 \If {$H_i =H$}
	 \State $C= C \cup d_i$.
	 \EndIf
	 \EndFor
	
	\State {\bf Return: } \[
	p = \frac{\sum_{c\in C, c \le \eta^Ty}  \tilde \phi( c ;\gamma, \eta^T \Sigma \eta) }  {\sum_{c \in C} \tilde \phi( c ; \gamma, \eta^T \Sigma \eta)}
	\]
\end{algorithmic}
\end{algorithm}
Algorithm \ref{alg:approx-pvalue} gives an approximate p-value for the null hypothesis $H_0: \eta^T \mu = \gamma$. The advantage of this algorithm is it does not need an explicit description of the set $S$, nor the set $U$. It runs the selection algorithm $H(\cdot)$ at the grid points $d_i$, and determines if the point $y_0 + d \frac{\Sigma \eta}{\eta^T \Sigma \eta}$ is in the selection event. Then it approximates the CDF of the univariate truncated normal by a discrete truncated normal.

\begin{conjecture}
	Let $D_m$ be a set of grid points $2 m^2$ grid points that is equispaced on $[ -m,m]$.  Let $U \subset \reals$ be an open interval, and  $ p_m$ be the p-value from Algorithm \ref{alg:approx-pvalue} using $D_m$.
	We have 
	\[
	\lim_{m\to \infty} p_m = F(\eta^Ty;\gamma, U(H,y_0,\frac{\Sigma \eta}{\eta^T \Sigma \eta})).
	\]
\end{conjecture}
\section{Inference in the full model}
\label{sec:full-model}
In Chapter \ref{chap:sel-lasso}, we focused on inference for the submodel coefficients $\beta^\star_{M}= X_M ^\dagger \mu$. In selective inference, the choice of the model $M$ is selected via an algorithm \eg\ the lasso, and the COS method constructed confidence intervals 
\[
\Pp\left( \beta^\star _{j,\hat M } \in C_j \right) = 1-\alpha.
\]

One possible criticism of the selective confidence intervals for submodel coefficients is the interpretability of the quantity $\beta^\star _{j,\hat M}$, since this is the population regression coefficient of variable $j$ within the model $\hat M$. The significance of variable $j$ depends on the choice of model meaning variable $j$ can be significant in model $M_1$, but not significant in $M_2$, which makes interpretation difficult. 

However, this is not an inherent limitation of the COS method. As we saw in the previous two sections, the COS method is not specific to the submodel coefficients. We simply need to change the space of hypothesis $\cH$ and the selection function $H$ to perform inference for other regression coefficients.

In many scientific applications, the quantity of interest is the regression coefficient within the full model $M=[1 \ldots p]$. We first discuss the case of $n\ge p$. Let us assume that $y \sim \cN(\mu, \sigma^2 I)$. In ordinary least squares , the parameter of interest is  $\beta^0 = X^\dagger \mu$, and a classical confidence interval guarantees 
\[
\Pp( \beta^0_j \in C_j ) = 1- \alpha
. \]
In the case of least squares after variable selection, we only want to make a confidence interval for the $j \in \hat M$, or variables selected by the lasso. This corresponds to inference for a subset $ \beta^0 _{\hat M } = E_{\hat M} \beta^0$, where $E_M$ selects the coordinates in $M$. The interpretation of $\beta^0_{\hat M} $ for $ j \in \hat M$ is clear;  this is the regression coefficient of the least squares coefficient restricted to the set selected by the lasso. 

For each coefficient $j\in \hat M$, Equation \eqref{eq:pivot-lasso} provides a valid p-value of the hypothesis $\beta^0_{j, \hat M} =\gamma$ ,
\begin{equation}
p_j = F_{\gamma,\ \sigma^2 ||\eta_j||_2^2}^{[\V^-, \V^+]}(\hat \beta_{j, \hat M}),
\label{eq:p-value}
\end{equation}
where $\eta_j = (X_{\hat\E}^T)^\dagger e_j$. 
By inverting, we obtain a selective confidence interval
\begin{equation}
\Pp\left( \beta^0_{j , \hat M} \in C_j \right)= 1-\alpha.
\end{equation}

\subsection{False Discovery Rate}
In this section, we show how to combine selective confidence intervals with the Benjamini-Yeuketieli procedure for FDR control. False discovery rate (FDR) is defined as,
$$
\Expect\left[ \frac{V}{R} \right],
$$
where $V$ is the number of incorrectly rejected hypotheses and $R$ is the total number of rejected hypotheses. We will restrict ourselves to the case of the well-specified linear model, $y=X\beta^0 +\epsilon$,  and $n\ge p$ with $X$ having full rank. In the context of linear regression, there is a sequence of hypotheses $H_{0,j} : \beta^0_j=0$ and a hypothesis is considered to be incorrectly rejected if $H_{0,j}$ is true, yet the variable is selected.

Given p-values, we can now apply the Benjamini-Yekutieli procedure \citep{benjamini2001control} for FDR control. Let $p_{(1)}\le p_{(2)}\le ... \le p_{(|\hat \E|)}$ be the order statistics, and $h_{|\hat \E|} = \sum_{i=1}^{|\hat \E|} \frac{1}{i}$. Let $k$ be
\begin{equation}
k=\max\left\{k: p_{(k)} \le \frac{k}{|\hat \E| h_{|\hat \E|}} \alpha\right\},
\label{eq:BY-rule}
\end{equation}
then reject $p_{(1)},\ldots, p_{(k)}$.
\begin{theorem}
	Consider the procedure that forms p-values using Equation \eqref{eq:p-value}, chooses $k$ via Equation \eqref{eq:BY-rule}, and rejects $p_{(1)}, \ldots,p_{(k)}$. Then FDR is controlled at level $\alpha$.
\end{theorem}
\begin{proof}
	Conditioned on the event that variable $j$ is in the lasso active set, $j \in \hat \E$, then $p_j$ is uniformly distributed among the null variables. Applying the Benjamini-Yekutieli procedure to the p-values $p_{(1)},\ldots, p_{(\hat \E)}$ guarantees FDR. The Benjamini-Yekutieli procedure allows for arbitrary dependence among the p-values, and only requires that the null p-values are uniformly distributed
\end{proof}


\subsection{Intervals for coefficients in full model when $n<p$}

In this section, we present a method for selective inference for coordinates of the full-model parameter $\beta^0$. We will assume the sparse linear model, namely, 
\[
y= X \beta^0 +\epsilon
\]
where $ \epsilon \sim \cN(0,\sigma^2$ and $\beta^0$ is $s$-sparse. 
Since $n<p$, we cannot use the method in the previous section since $\beta^0 \neq X^\dagger X \beta^0$.  Instead, we will construct a quantity $\beta^d$ that is extremely close to $\beta^0$ and show that $\beta^d _j = \eta_j ^T (X\beta^0) + h_j$. We do this by constructing a population version of the debiased estimator.

The debiased estimator presented in \cite{javanmard2013confidence,van2013asymptotically,zhang2011confidence} is
\begin{align*}
\hat \beta^d &= \hat \beta + \frac1n \hat \Theta X^T (y- X \hat \beta)\\
&=\frac1n \hat \Theta X^Ty+(I - \hat \Theta \hat \Sigma) \hat \beta\\
&=\frac1n \hat \Theta X^Ty+(I - \hat \Theta \hat \Sigma) \begin{bmatrix} 
\frac1n \hat \Sigma_{\hat M} ^{-1} X_{\hat M} ^T y - \lambda \hat \Sigma _{\hat M} ^{-1} s_{\hat M}\\ 0
\end{bmatrix}
\end{align*}
where $\hat \Sigma _{\hat M} := \frac 1n X_{\hat M} ^T X_{\hat M}$ and $\hat \Theta$ is an approximate inverse covariance that is the solution to 
\begin{align*}
\min& \sum_{j} \hat \Theta_j ^T \hat \Sigma \hat \Theta_j\\
\text{ subject to } & \norm{\hat \Sigma \hat \Theta - I }_\infty \le C \sqrt{\frac{ \log p }{n}}.
\end{align*}

Define the population quantity $\beta^d$ by replacing all occurrences of $y$ with $\mu$:
\begin{align}
&\beta^d ( M, s) :=\frac1n \hat \Theta X^T \mu+(I - \hat \Theta \hat \Sigma) \begin{bmatrix} 
\frac1n \hat \Sigma_{ M} ^{-1} X_{ M} ^T \mu - \lambda \hat \Sigma _{ M} ^{-1} s\\ 0
\end{bmatrix} \nonumber \\
&=\beta^0 +\frac1n \hat \Theta X^T \mu+(I - \hat \Theta \hat \Sigma) \begin{bmatrix} 
\frac1n \hat \Sigma_{ M} ^{-1} X_{ M} ^T \mu - \lambda \hat \Sigma _{ M} ^{-1} s-\beta^0_{M}\\ -\beta^0_{-M}
\end{bmatrix} \label{eq:betad-beta0}\\
&= \left(\frac1n \Theta X^T  +(I- \Theta \hat \Sigma) F_M\hat \Sigma_{M} ^{-1} X_M^T\right) \mu -\lambda (I- \Theta \hat \Sigma ) F_M  \hat \Sigma_{M}^{-1} s  \nonumber\\
&:= B\mu +h \nonumber,
\end{align}
where $F_M$ is the matrix such that it takes an $|M|$ vector and pads with $0$ to make a $p$ vector.

By choosing $\eta$ as a row of $B$, COS framework provides a selective test and confidence interval, 
\begin{align*}
H_0: \beta^d _j (\Hat M,\hat s) = \gamma-\eta^T h\\
\Pp ( \beta^d _j (\hat M, \hat s) \in C_j ) = 1-\alpha.
\end{align*}

The next step is to show that $\beta^d (\hat M, \hat s) $ is close to $\beta^0$, so by appropriately widening $C_j$, we cover $\beta^0$.
\begin{theorem}
	Assume that lasso is consistent in the sense $\norm { \hat \beta - \beta^0}_1 \le c_L s \sqrt { \frac {\log p}{n}}$, $\Theta $ satisfies $\norm{ \Theta \hat \Sigma - I}_\infty \le c_{\Theta}  \sqrt{\frac{ \log p }{n}}$, and $X$ has the sparse eigenvalue condition  $\mu(S,k) := \min_{\norm{v}_0\le  k, \norm{v}_2 = 1}\frac1n \norm{v^T Sv}>0$, and the empirical sparsity $\hat s := |\hat M| < c_{M} s$, then 
	\[
	\norm{\beta^d (\hat M, \hat s) - \beta^0 }_\infty \le C_{\beta^d} \frac{s \log p }{n}.
	\]
\end{theorem}
\begin{proof}
	Starting from Equation \eqref{eq:betad-beta0}, we have 
	\begin{align*}
	\beta^d - \beta^0 &= (I- \Theta \hat \Sigma) \begin{bmatrix} 
	\frac1n \hat \Sigma_{ \hat M} ^{-1} X_{\hat M} ^T \mu - \lambda \hat \Sigma _{\hat M} ^{-1} s-\beta^0_{\hat M}\\ -\beta^0_{-\hat M}
	\end{bmatrix}\\
	\norm{\beta^d - \beta^0 }_\infty &\le \norm{I- \Theta \hat \Sigma}_\infty \norm{\begin{bmatrix} 
		\frac1n \hat \Sigma_{\hat M} ^{-1} X_{\hat M} ^T \mu - \lambda \hat \Sigma _{ \hat M} ^{-1} s-\beta^0_{\hat M}\\ -\beta^0_{-\hat M}
		\end{bmatrix}}_1\\
	&\le c_{\Theta}  \sqrt{\frac{ \log p }{n}} \left( \norm{\frac1n \hat \Sigma_{\hat M} ^{-1} X_{\hat M} ^T \mu - \lambda \hat \Sigma _{\hat M} ^{-1} s-\beta^0_{\hat M}}_1 + \norm{ \beta^0_{-\hat M}}_1 \right)\\
	&\le c_{\Theta}  \sqrt{\frac{ \log p }{n}} \left( \norm{\frac1n \hat \Sigma_{\hat M} ^{-1} X_{\hat M} ^T \mu - \lambda \hat \Sigma _{\hat M} ^{-1} s-\beta^0_{\hat M}}_1 + \norm{ \hat \beta - \beta^0}_1 \right)\\
	&\le c_{\Theta}  \sqrt{\frac{ \log p }{n}} \left( \norm{\frac1n \hat \Sigma_{\hat M} ^{-1} X_{\hat M} ^T \mu - \lambda \hat \Sigma _{\hat M} ^{-1} s-\beta^0_{\hat M}}_1 +c_{L} s\sqrt{\frac{\log p}{n}} \right)
	\end{align*}
	where we used the lasso consistency assumption,$\norm{ \Theta \hat \Sigma - I}_\infty \le c_{\Theta}  \sqrt{\frac{ \log p }{n}}$n, and the second to last inequality uses the fact that $\hat \beta_{-\hat M} =0$, so $\norm { \hat \beta - \beta^0}_1 = \norm{ \hat \beta_{\hat M}-\beta^0_{\hat M} }_1 + \norm{ -\beta^0_{-\hat M}}_1 \ge \norm{ -\beta^0 _{-\hat M}}_1.$
	
	We now show $\norm{\frac1n \hat \Sigma_{\hat M} ^{-1} X_{\hat M} ^T \mu - \lambda \hat \Sigma _{\hat M} ^{-1} s-\beta^0_{\hat M}}_1 \le s \sqrt{\frac{\log p }{n}}$.
	\begin{align*}
	&\norm{\frac1n \hat \Sigma_{\hat M} ^{-1} X_{\hat M} ^T \mu - \lambda \hat \Sigma _{\hat M} ^{-1} s-\beta^0_{\hat M}}_1\le \norm{(\frac1n \hat \Sigma_{\hat M} ^{-1} X_{\hat M} ^T \mu - \lambda \hat \Sigma _{\hat M} ^{-1} s) - \hat \beta_{\hat M}}_1 + \norm{ \hat \beta_{\hat M} - \beta^0 _{\hat M}}_1\\
	&\le \norm{(\frac1n \hat \Sigma_{\hat M} ^{-1} X_{\hat M} ^T \mu - \lambda \hat \Sigma _{\hat M} ^{-1} s) -(\frac1n \hat \Sigma_{\hat M} ^{-1} X_{\hat M} ^T y - \lambda \hat \Sigma _{\hat M} ^{-1} s)   }_1+ c_L s \sqrt{ \frac{\log p}{n}}\\
	&\le \norm{ \frac1n \hat \Sigma_{\hat M} ^{-1} X_{\hat M} ^T\epsilon}_1+c_L s \sqrt{ \frac{\log p}{n}}\\
	&\le \sqrt{\hat s} \norm{ \frac1n \hat \Sigma_{\hat M} ^{-1} X_{\hat M} ^T\epsilon}_2 +c_L s \sqrt{ \frac{\log p}{n}}\\
	&\le \sqrt{\hat s} \norm{  \hat \Sigma_{\hat M} ^{-1}}_2\norm{\frac1n X_{\hat M} ^T\epsilon}_2+c_L s \sqrt{ \frac{\log p}{n}} \\
	&\le \sqrt{\hat s} \norm{  \hat \Sigma_{\hat M} ^{-1}}_2 \sqrt{\hat s} \norm{\frac1n X^T\epsilon}_\infty+c_L s \sqrt{ \frac{\log p}{n}}\\
	&\le \hat s\sqrt{ \frac{\log p }{n}} \norm{  \hat \Sigma_{\hat M} ^{-1}}_2 +c_L s \sqrt{ \frac{\log p}{n}}\\
	&\le \frac{1}{\lambda_{\min} ( \frac1n  X_{\hat M} ^T X_{\hat M}) } \hat s \sqrt{ \frac{\log p }{n}}+c_L s \sqrt{ \frac{\log p}{n}}\\
	&\le \left(\frac{1}{\mu(\hat \Sigma,c_M s)} c_M  + c_L \right)s \sqrt{\frac{\log p}{n}}
	\end{align*}
	, where $\hat s = |\hat M|$,and  $\lambda_{\min} ( \frac1n  X_{\hat M} ^T X_{\hat M}) \ge \mu(\hat s)> \mu (c_M s)$.
	
	Plugging this into the expression for $\norm{\beta^d - \beta^0}$,
	\begin{align}
	\norm{\beta^d - \beta^0}_\infty &\le c_{\Theta} \sqrt{\frac{\log p }{n}}  \left(\left(\frac{1}{\mu(\hat \Sigma,c_M s)} c_M  + c_L \right)s \sqrt{\frac{\log p}{n}}+c_L s \sqrt{\frac{\log p}{n}} \right)\\
	&\le \left(\frac{c_\Theta}{\mu(\hat \Sigma,c_M s)} c_M  +  2 c_L c_\Theta \right) \frac{s \log p }{n} \label{eq:betad-error-bound}
	\end{align}
\end{proof}
\begin{lemma}[Assumptions hold under random Gaussian design with additive Gaussian noise]
	Assume that the rows of $X \sim \cN(0, \Sigma)$ and $\lim \frac{n}{ s \log p} =\infty$. Then the estimation consistency property, existence of a good approximation $\hat \Theta$, empirical sparsity $\hat s <c_M s$, and $\mu(\hat \Sigma,c_M s) > \frac12 \mu (\Sigma,c_M s)$ with probability tending to 1.
	\label{lem:assumptions-rand-design}
\end{lemma}

\begin{proof}
	The estimation consistency property follows from \cite{negahban2012unified}. The bound $\norm{\hat \Sigma \hat \Theta -I}_\infty < \sqrt{\frac{\log p }{n}}$ is established in \cite{javanmard2013confidence}. The empirical sparsity result $ \hat s \le c_M s$ is from \cite{belloni2011l1,belloni2013least}.
	
	The condition on concentration of sparse eigenvalues can be derived using  \citet[Lemma 15, Supplementary Materials]{loh2012high}. Lemma 15 states if $X$ is a zero-mean sub-Gaussian matrix with covariance $\Sigma$ and subgaussian parameter $\sigma^2$, then there is a universal constant $c>0$ such that
	\begin{equation}
	\Pp\left( \sup_{\norm{v}_0 \le s, \norm{v}_2 =1 } | \frac1n v^T X^TX v - v^T \Sigma v| \ge t\right) \le 2 \exp\left( -cn \min( \frac{t^2}{\sigma^4}, \frac{t}{\sigma^2}) +s \log p \right) .
	\label{eq:loh-sparse-eigenvalue}
	\end{equation}
	With high probability and for all $v \in K(s) : = \{v: \norm{v}_0 \le s, \norm{v}_2 = 1\}$,
	\begin{align*}
	| v^T \hat \Sigma v - v^T \Sigma v| &< t\\
	v^T \Sigma v - t &< v^T \hat \Sigma v\\
	\min_{w \in K(s)} w^T \Sigma w - t&<  v^T \hat \Sigma v \\
	\mu (\Sigma,s) -t &< v^T \hat \Sigma v\\
	\mu(\Sigma,s)- t &< \min_{v \in K(s)} v^T \hat \Sigma v\\
	\mu(\Sigma,s) - t &<\mu(\hat \Sigma,s).
	\end{align*}
	We now use this to show $\mu(\hat \Sigma, C_m s)> \frac12 \mu(\Sigma,c_M s)$. Let $t= \frac12 \mu(\Sigma,c_M s)$, then by the previous argument and Equation \eqref{eq:loh-sparse-eigenvalue},
	\[ \mu(\hat \Sigma, c_M s) > \frac12 \mu( \Sigma, c_M s) \]
	with probability at least 
	\[
	1-2 \exp\left( -cn \min(\frac{ \mu(\Sigma,c_M s)^2}{4\sigma^4}, \frac{\mu(\Sigma,c_M s)}{2\sigma^2} ) +c_M s \log p \right).
	\]
	For $n> \frac{2c_M s \log p }{c \min(\frac{ \mu(\Sigma,c_M s)^2}{4\sigma^4}, \frac{\mu(\Sigma,c_M s)}{2\sigma^2} ) } $, we have with probability at least,
	\[
	1-2 \exp \left(-\frac12 cn \min(\frac{ \mu(\Sigma,c_M s)^2}{4\sigma^4}, \frac{\mu(\Sigma,c_M s)}{2\sigma^2} ) \right).
	\]
	
\end{proof}

\begin{corollary}
	Under the assumptions of Lemma \ref{lem:assumptions-rand-design}, and $\lim \frac{n}{ s^2 \log^2 p} =\infty$, 
	\[
	\norm{\beta^d (\hat M, \hat s ) - \beta^0}_\infty \le \frac{\delta }{\sqrt{n}}
	\]
	for any $\delta >0$.
\end{corollary}
\begin{proof}
	Lemma \ref{lem:assumptions-rand-design} ensures that $\mu (\hat \Sigma, s) > \frac12 \mu(\Sigma,s)$, and plugging this into Equation \eqref{eq:betad-error-bound} gives
	\begin{align*}
	\norm{\beta^d - \beta^0}_\infty \le \left(\frac{2c_\Theta}{\mu( \Sigma,c_M s)} c_M  +  2 c_L c_\Theta \right) \frac{s \log p }{n} 
	\end{align*}
	Since $n \gg s^2 \log ^2p $, we have 
	\begin{align*}
	\norm{\beta^d - \beta^0}_\infty &\le \left(\frac{2c_\Theta}{\mu( \Sigma,c_M s)} c_M  +  2 c_L c_\Theta \right) \frac{s \log p }{n} \\
	&\le \left(\frac{2c_\Theta}{\mu( \Sigma,c_M s)} c_M  +  2 c_L c_\Theta \right) o(\frac{1}{\sqrt{n}})\\
	&\le \frac{\delta}{\sqrt{n}}.
	\end{align*}
	
\end{proof}

\begin{corollary}
	Let $C_j$ be a selective confidence interval for $\beta^d$ meaning $\Pp( \beta^d _{j} \in C_j ) = 1 -\alpha$, then 
	\[
	\liminf \Pp( \beta^0_j \in C_j \pm \frac{\delta}{\sqrt{n}}) \ge 1- \alpha.
	\]
\end{corollary}
\begin{proof}
	With probability at least $1-\alpha$, $\beta^d_j \in C_j$ and with probability tending to $1-o(1)$,  $\beta^d_j - \beta^0_j < \frac{\delta}{\sqrt{n}}$. Thus with probability at least $1-\alpha -o(1)$, $\beta^0_j \in C_j \pm \frac{\delta}{\sqrt{n}}$. 
\end{proof}

\begin{figure}
	\centering
	\includegraphics[width=.72\textwidth]{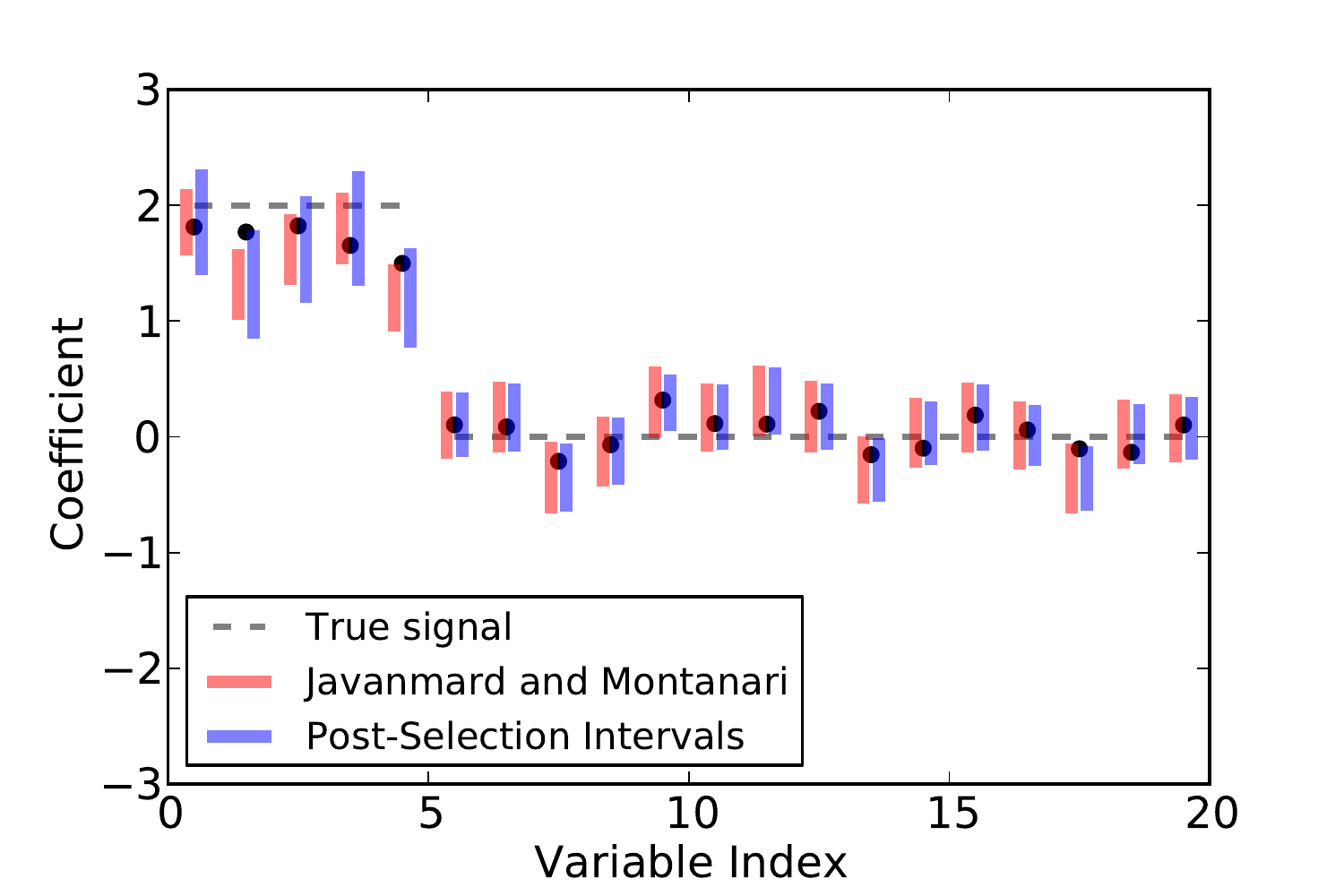}
	\caption{Confidence intervals for the coefficients in a design with $n=25$, $p=50$, and 5 non-zero coefficients. Only the first 20 coefficients are shown. The dotted line represents the true signal, and the points represent the (biased) post-selection target. The colored bars denote the intervals.}
	\label{fig:montanari}
\end{figure}

Figure \ref{fig:montanari} shows the results of a simulation study. It makes clear that the intervals of \citet{javanmard2013confidence} and our selective confidence intervals cover $\beta^d$, which is close to $\beta^0$.  The Javanmard-Montanari intervals are the high-dimensional analog of a z-interval, so they are not selectively valid, unlike the selective intervals in blue.

\section{Selective Inference for the Knockoff Filter}
\label{sec:knockoff}
In this section, we show how to make selectively valid confidence intervals for the knockoff method \cite{foygel2014controlling}.
Let $\tilde X$ be the knockoff design matrix, so the knockoff regression is done on $y=[X;  \tilde X] \beta+\epsilon$. The introduction of the knockoff variables, $\tilde X$, allows us to estimate the FDP as the number of knockoff variables selected divided by the number of true variables selected:
\begin{align}
 FDP( M) =\frac{| M \cap \tilde X|}{ | M \cap X| \vee 1}
 \label{eq:fdp-estimate}
\end{align}

Given a sequence of models $M(1), \ldots, M(k)$ be a sequence of nested models $M(k) \subset M_(k-1) \subset \ldots \subset M(1) \subset [ 1\ ..\ 2p]$. For the lasso, where the $M(j)$ correspond to the lasso active set at $\lambda_j$, the models are not necessarily nested. We define $M(j) = \cup_{l=k}^j A_j$, where $A_j$ is the active set of lasso at $\lambda_j$. We have an estimate FDP estimate for each model, $FDP(M(j))= \frac{|M(j) \cap \tilde X|}{| M(j) \cap X|}$, where $X$ and $\tilde X$ represent the indices of the real and knockoff variables respectively. This suggests selecting the largest model such that the FDP estimate is less than $\alpha$,
\begin{align}
T= \min \{ t: FDP(M(t))\le \alpha \}.
\label{eq:stopping-rule}
\end{align}
To show this controls modified FDR, we need to construct $W$-statistics such that our stopping rule, corresponds to the stopping rule of \cite{foygel2014controlling}. 
\begin{theorem}
	The model selected by the stopping rule in Equation \eqref{eq:stopping-rule} controls the modified FDR, that is
	\[
	\Ee \left[ \frac{ |M(T) \cap V| }{|M(T) \cap X| +1/\alpha}  \right] \le \alpha
	\]
	where $V=\{ 1\le j \le p : \beta_j =0\}$.
\end{theorem}
\begin{theorem}
	We construct some W statistics. Define $t_j = \min \{t: x_j \in M(t) \}$ and $\tilde t_j = \min\{ t: \tilde x _j \in M(t)\}  $. Define 
	\begin{align*}
	W_j = \begin{cases}
	t_j& \text{ if }  t_j < \tilde t_j \\
	-\tilde t_j & \text{ if } \tilde t_{j} < t_j
	\end{cases}
	\end{align*}
	
	We now verify that the FDP estimate given in \cite{foygel2014controlling} using the W-statistics are the same FDP estimate as \eqref{eq:fdp-estimate}.
	\begin{align*}
	FDP_W (t) = \frac{|\{j: W_j \le -t\} |}{ | \{j: W_j > t \} \vee 1 }\\
	= \frac{\{ j: \tilde t_j > t, j \in \tilde X\}}{   |\{ j: t_j > t, j \in X \}| \vee 1}\\
	= \frac{\{ j: M(t) \cap \tilde X\}}{   |\{ j: M(t) \cap X \}| \vee 1}.
	\end{align*}
	By invoking the main theorem of \cite{foygel2014controlling}, we see that \eqref{eq:stopping-rule} controls the modified FDR.
	\end{theorem}

By using the $FDP^+$ estimate in place of equation \eqref{eq:stopping-rule},
\begin{align}
 FDP^+ ( M) =\frac{| M \cap \tilde X|}{ | M \cap X|\vee 1 +1}\\
T^+= \min \{ t: FDP^+ (M(t))\le \alpha \}.
\label{eq:stopping-rule+}
\end{align}
we can control FDR, instead of modified FDR.
\begin{theorem}
	The model selected by the stopping rule in Equation \eqref{eq:stopping-rule+} controls  FDR, that is
	\[
	\Ee \left[ \frac{ |M(T) \cap V| }{|M(T) \cap X| \vee 1}  \right] \le \alpha
	\]
	where $V=\{ 1\le j \le p : \beta_j =0\}$.
\end{theorem}
\begin{proof}
Same as the previous theorem.
\end{proof}

Let $M^\star=KO(y)$ be the final model returned by the knockoff procedure $k$ applied to the regression pair $(y,X)$ using the lasso models at the sequence $\lambda_1, \ldots, \lambda_k$. Our goal is to do inference for $\beta^0_{j} = e_j X^\dagger \mu $ for some $j \in M^\star$. The selection event, the set of $y$'s that lead us to testing $\beta^0 _{j}$, is $S_{j} = \{ y: j \in KO(y) \}$. This precise set is difficult to analytically describe, so we resort to Algorithm \ref{alg:approx-pvalue}. 

We can analytically describe the finer event $$
S= \{ y: \left(L(y,\lambda_1), \ldots, L(y, \lambda_{T+1}) \right)  =( M(1), \ldots, M(T+1)) \} \subset S_j,
$$
 where $L(y, \lambda)$ is the active set of lasso at $\lambda$. For any $y \in S$, the knockoff procedure defined by the stopping rule \eqref{eq:stopping-rule} returns the same set of variables, so $S \subset S_j$. The set $S$ is described by the intersection of the  union of linear inequalities given in Section \ref{sec:minimal}. This allows us to do inference using the results of Theorem \ref{thm:unconditional-pivot}.

We next describe a method using the general method of Chapter \ref{sec:general-method}. Using the COS method, we first describe the knockoff selection event. The selection event for variable $j$ is $S_j = \{ y: j \in KO(y)\}$. The general method instead uses the one-dimensional finer selection event $U_j = \{ c: j \in  KO ( \Pind{y} +c \frac{ \Sigma \eta}{\eta^T \Sigma \eta} )\}$. This set is approximated using Algorithm \ref{alg:approx-pvalue} that computes an approximation to $U_j$ and an approximate p-value.

Since the knockoff method assumes a well-specified linear model, we can use the reference distribution $y \sim \cN( X \beta , \sigma^2 I)$ instead of $y \sim \cN( \mu, \sigma^2 I)$. This is the well-specified linear regression model of \cite{fithian2014optimal}. The selection event is now $U_j = \{ C: j \in KO ( P_{X_{-j}}^\perp y + C)  , C \in span( X_{-j})^\perp \}$. A multi-dimensional analog of Algorithm \ref{alg:approx-pvalue} can now be applied, but the search set $D$ is now over a $n-p+1$ dimensional subset. 

\part{Learning Mixed Graphical Models}
\chapter{Learning Mixed Graphical Models}
\section{Introduction}
Many authors have considered the problem of learning the edge
structure and parameters of sparse undirected graphical models. We
will focus on using the $l_1$ regularizer to promote sparsity.  This
line of work has taken two separate paths: one for learning continuous
valued data and one for learning discrete valued data.  However, typical
data sources contain both continuous and discrete
variables: population survey data, genomics data, url-click pairs etc. For genomics data, in addition to the gene expression values, we have attributes attached to each sample such as gender, age, ethniticy etc. In
this work, we consider learning mixed models with both continuous Gaussian
variables and discrete categorical variables. 

For only continuous variables, previous work assumes a
multivariate Gaussian (Gaussian graphical) model with mean $0$ and
inverse covariance $\Theta$. $\Theta$ is then estimated via the
graphical lasso by minimizing the regularized negative log-likelihood
$\ell(\Theta)+\lambda \norm{\Theta}_{1}$. Several efficient methods
for solving this can be found in \cite{friedman2008,
  banerjee2008}. Because the graphical lasso problem is
computationally challenging, several authors considered methods
related to the pseudolikelihood (PL) and nodewise regression
\citep{meinshausen06, friedman2010,peng2009}.  For discrete models,
previous work focuses on estimating a pairwise Markov random field of
the form $p(y) \propto \exp{\sum_{r\leq j} \phi_{rj}(y_r,y_j)}$, where $\phi_{rj}$ are pairwise potentials. The
maximum likelihood problem is intractable for models with a moderate
to large number of variables (high-dimensional) because it requires
evaluating the partition function and its derivatives. Again previous work has focused
on the pseudolikelihood approach 
\citep{guo2010joint,schmidt2010,schmidt2008,hoefling2009,jalali2011,lee2006,ravikumar2010}. 

Our main contribution here is to propose a model that connects the
discrete and continuous models previously discussed. The conditional
distributions of this model are two widely adopted and well understood
models: multiclass logistic regression and Gaussian linear
regression. In addition, in the case of only discrete variables, our
model is a pairwise Markov random field; in the case of only
continuous variables, it is a Gaussian graphical model. Our proposed
model leads to a natural scheme for structure learning that
generalizes the graphical Lasso. Here the parameters occur as
singletons, vectors or blocks, which we penalize using group-lasso
norms, in a way that respects the symmetry in the model. Since each
parameter block is of different size, we also derive a calibrated
weighting scheme to penalize each edge fairly. We also discuss a
conditional model (conditional random field) that allows the output
variables to be mixed, which can be viewed as a multivariate response
regression with mixed output variables. Similar ideas have been used
to learn the covariance structure in multivariate response regression
with continuous output variables
\cite{witten2009covariance,kim2009multivariate,rothman2010sparse}.

In Section \ref{sec:mgm}, we introduce our new mixed graphical model and discuss
previous approaches to modeling mixed data. Section \ref{sec:paramest} discusses the
pseudolikelihood approach to parameter estimation and connections to
generalized linear models. Section \ref{sec:penalty} discusses a natural method to
perform structure learning in the mixed model. Section \ref{sec:calibration} presents the
calibrated regularization scheme, Section \ref{sec:msc} discusses the consistency of the estimation procedures, and Section \ref{sec:optalg} discusses two methods
for solving the optimization problem. Finally, Section \ref{sec:condmodel} discusses a
conditional random field extension and Section \ref{sec:exp} presents empirical
results on a census population survey dataset and synthetic
experiments.

\section{Mixed Graphical Model}
\label{sec:mgm}
We propose a pairwise graphical model on continuous and discrete variables. The model is a pairwise Markov random field with density $p(x,y;\Theta)$ proportional to
\begin{align}
\exp{\left(\sum_{s=1}^p\sum_{t=1}^p-\frac{1}{2} \beta_{st} x_{s} x_{t}+\sum_{s=1}^{p}\alpha_{s} x_{s} +\sum_{s=1}^p \sum_{j=1}^{q} \rho_{sj}(y_{j})x_{s}+\sum_{j=1}^q \sum_{r=1}^q \phi_{rj}(y_r , y_j)\right)}.
\label{eq:jointdensity}
\end{align}
Here $x_s$ denotes the $s$th of $p$ continuous variables, and $y_j$ the $j$th of $q$ discrete variables. 
The joint model is parametrized by $\Theta= [\{\beta_{st}\}, \{\alpha_s \}, \{\rho_{sj}\},\{\phi_{rj}\}]$. The discrete $y_r$ takes on $L_r$ states. The model parameters are $\beta_{st}$ continuous-continuous edge potential, $\alpha_{s}$ continuous node potential, $\rho_{sj}(y_{j})$ continuous-discrete edge potential, and $\phi_{rj}(y_r,y_j)$ discrete-discrete edge potential. $\rho_{sj}(y_j )$ is a function taking $L_j$ values $\rho_{sj}(1),\ldots,\rho_{sj}(L_j)$. Similarly, $\phi_{rj}(y_r,y_j)$ is a bivariate function taking on $L_r \times L_j$ values. Later, we will think of $\rho_{sj}(y_j)$ as a vector of length $L_j$ and $\phi_{rj}(y_r,y_j)$ as a matrix of size $L_r \times L_j$.

The two most important features of this model are:
\begin{enumerate}
\item the conditional distributions are given by Gaussian linear regression and multiclass logistic regressions;
\item the model simplifies to a multivariate Gaussian in the case of only continuous variables and simplifies to the usual discrete pairwise Markov random field in the case of only discrete variables.
\end{enumerate}
The conditional distributions of a graphical model are of critical
importance. The absence of an edge corresponds to two variables being
conditionally independent. The conditional independence can be read
off from the conditional distribution of a variable on all others. For
example in the multivariate Gaussian model, $x_s$ is conditionally
independent of $x_t$ iff the partial correlation coefficient is
$0$. The partial correlation coefficient is also the regression
coefficient of $x_t$ in the linear regression of $x_s$ on all other
variables. Thus the conditional independence structure is captured by
the conditional distributions via the regression coefficient of a
variable on all others. Our mixed model has the desirable property
that the two type of conditional distributions are simple Gaussian
linear regressions and multiclass logistic regressions. This follows
from the pairwise property in the joint distribution. In more detail:
\begin{enumerate}
\item The conditional distribution of $y_r$ given the rest is multinomial, with probabilities defined  by a multiclass logistic regression where the covariates are the other variables $x_s$ and $y_{\bs r}$ (denoted collectively by $z$ in the right-hand side):
  \begin{equation}
    \label{eq:simple1}
p(y_r =k|y_{\bs r}, x; \Theta) =\frac{\exp{\left( \omega_{k}^{T} z \right)}}{\sum_{l=1}^{L_r} \exp{\left( \omega_{l}^{T} z\right) }} = \frac{\exp{\left( \omega_{0k} + \sum_{j} \omega_{kj} z_j \right)}} {\sum_{l=1}^{L_r}\exp{\left( \omega_{0l} + \sum_j \omega_{lj} z_j \right)}}
\end{equation}
Here we use a simplified notation, which we make explicit in Section~\ref{sec:pseudolikelihood}. The discrete variables are represented as dummy variables for each state, e.g.  $z_j = \indicator{y_u = k}$, and for continuous variables $z_s =x_s$.
\item The conditional distribution of $x_s$ given the rest is Gaussian, with a mean function defined by a linear regression with predictors $x_{\bs s}
$ and $y_r$.
\begin{align}
E(x_s | x_{\bs s}, y_r;\Theta ) &=\omega^{T} z= \omega_0 +\sum_j z_j \omega_j\label{eq:simple2}\\
p(x_s | x_{\bs s}, y_r;\Theta )&= \frac{1}{\sqrt{2\pi} \sigma_s} \exp{\left(-\frac{1}{2 \sigma_s^2} ( x_{s} -\omega^{T} z )^2 \right)}.\nonumber
\end{align}
As before, the discrete variables are represented as dummy variables for each state $z_j = \indicator{y_u = k}$ and for continuous variables $z_s =x_s$.
\end{enumerate}
The exact form of the conditional distributions (\ref{eq:simple1}) and (\ref{eq:simple2}) are given in (\ref{eq:discond}) and (\ref{eq:ctscond}) in Section~\ref{sec:pseudolikelihood}, where the  regression parameters $\omega_j$ are defined in terms of the parameters $\Theta$.

The second important aspect of the mixed model is the two special cases of only continuous and only discrete variables.
\begin{enumerate}
\item Continuous variables only.  The pairwise mixed model reduces to the familiar multivariate Gaussian parametrized by the symmetric positive-definite inverse covariance matrix $B=\{\beta_{st}\}$ and mean $\mu=B^{-1}\alpha$,
$$
p(x)\propto \exp\left( -\frac{1}{2}(x-B^{-1} \alpha)^{T} B (x-B^{-1}\alpha)\right).
$$
\item Discrete variables only. The pairwise mixed model reduces to a pairwise discrete (second-order interaction) Markov random field,
\begin{equation*}
p(y)\propto\exp{\left(\sum_{j=1}^q \sum_{r=1}^q \phi_{rj}(y_r , y_j)\right)}.
\label{eq:discrete}
\end{equation*}
\end{enumerate}

Although these are the most important aspects, we can characterize the joint distribution further.
The conditional distribution of the continuous variables given the discrete follow a multivariate Gaussian distribution, $p(x|y)= \No(\mu(y),B^{-1})$. Each of these Gaussian distributions share the same inverse covariance matrix $B$ but differ in the mean parameter, since all the parameters are pairwise. By standard multivariate Gaussian calculations,
\begin{align}
p(x|y)&=\No(B^{-1} \gamma(y),B^{-1})\\
\{\gamma(y)\}_s&= \alpha_s+\sum_{j} \rho_{sj}(y_j)\\
p(y)  &\propto \exp{\left(\sum_{j=1}^q \sum_{r=1}^j \phi_{rj} (y_r, y_j) +\frac{1}{2}  \gamma(y)^{T} B^{-1} \gamma(y)\right)}
\end{align}
Thus we see that the continuous variables conditioned on the discrete
are multivariate Gaussian with common covariance, but with means that
depend on the value of the discrete variables. The means depend
additively on the values of the discrete variables since
$\{\gamma(y)\}_s= \sum_{j=1}^r \rho_{sj}(y_j)$. The marginal $p(y)$
has a known form, so for models with few number of discrete variables
we can sample efficiently.

\subsection{Related work on mixed graphical models}
\citet{Lauritzen1996} proposed a type of mixed graphical model, with
the property that conditioned on discrete variables, $p(x|y) =
\No(\mu(y), \Sigma(y) )$. The homogeneous mixed graphical model
enforces common covariance, $\Sigma(y) \equiv \Sigma$. Thus our
proposed model is a special case of Lauritzen's mixed model with the
following assumptions: common covariance, additive mean assumptions
and the marginal $p(y)$ factorizes as a pairwise discrete Markov
random field. With these three assumptions, the full model simplifies
to the mixed pairwise model presented. Although the full model is more
general, the number of parameters scales exponentially with the
number of discrete variables, and the conditional distributions are
not as convenient.  For each state of the discrete variables there is
a mean and covariance. Consider an example with $q$ binary variables
and $p$ continuous variables; the full model requires estimates of
$2^q$ mean vectors and covariance matrices in $p$ dimensions.  Even if
the homogeneous constraint is imposed on Lauritzen's model, there are
still $2^q$ mean vectors for the case of binary discrete
variables. The full mixed model is very complex and cannot be easily
estimated from data without some additional assumptions. In
comparison, the mixed pairwise model has number of parameters
$O((p+q)^2)$ and allows for a natural regularization scheme which
makes it appropriate for high dimensional data. 

An alternative to the regularization approach that we take in this paper, is the limited-order correlation hypothesis testing method \cite{tur2012learning}. The authors develop a hypothesis test via likelihood ratios for conditional independence. However, they restrict to the case where the discrete variables are marginally independent so the maximum likelihood estimates are well-defined for $p>n$.

There is a line of work regarding parameter estimation in undirected
mixed models that are decomposable: any path between two discrete
variables cannot contain only continuous variables. These models allow
for fast exact maximum likelihood estimation through node-wise
regressions, but are only applicable when the structure is known and
$n>p$ \citep{edwards2000introduction}. There is also related work on
parameter learning in directed mixed graphical models. Since our
primary goal is to learn the graph structure, we forgo exact parameter
estimation and use the pseudolikelihood. Similar to the exact maximum
likelihood in decomposable models, the pseudolikelihood can be
interpreted as node-wise regressions that enforce symmetry.

To our knowledge, this work is the first to consider convex optimization procedures for learning the edge structure in mixed graphical models.
\section{Parameter Estimation: Maximum Likelihood and Pseudolikelihood}
\label{sec:paramest}
Given samples $(x_i,y_i)_{i=1}^n$, we want to find the maximum
likelihood estimate of $\Theta$. This can be done by minimizing the
negative log-likelihood of the samples:
\begin{align}
\ell(\Theta)&= -\sum_{i=1}^n \log{ p(x_i,y_i;\Theta)} \mbox{ where }\\
\log{p(x,y;\Theta)}&= \sum_{s=1}^{p}\sum_{t=1}^p -\frac{1}{2} \beta_{st} x_{s} x_{t}+\sum_{s=1}^{p}\alpha_{s} x_{s} +\sum_{s=1}^p \sum_{j=1}^q\rho_{sj}(y_{j})x_{s} \nonumber \\
& +\sum_{j=1}^q\sum_{r=1}^j \phi_{rj}(y_r , y_j)-\log{Z(\Theta)}
\end{align}
The negative log-likelihood is convex, so standard gradient-descent
algorithms can be used for computing the maximum likelihood
estimates. The major obstacle here is $Z(\Theta)$, which involves a high-dimensional integral. Since the pairwise mixed model includes both the discrete
and continuous models as special cases, maximum likelihood estimation
is at least as difficult as the two special cases, the first of  which is a
well-known computationally intractable problem. We defer the
discussion of maximum likelihood estimation to the supplementary material.
\subsection{Pseudolikelihood}
\label{sec:pseudolikelihood}
The pseudolikelihood method \cite{besag1975} is a computationally efficient and consistent estimator formed by products of all the conditional distributions:
\begin{align}
\tilde{\ell}(\Theta|x,y)=-\sum_{s=1}^{p} \log{p(x_s|x_{\bs s},y;\Theta)}-\sum_{r=1}^{q} \log{p(y_{r}|x,y_{\bs r};\Theta)} 
\label{eq:negpl}
\end{align}  
The conditional distributions $p(x_{s}| x_{\backslash s}, y; \theta)$ and $p(y_r = k | y_{\backslash r,}, x;\theta)$ take on the familiar form of linear Gaussian and (multiclass) logistic regression, as we pointed out in (\ref{eq:simple1}) and (\ref{eq:simple2}). Here are the details:
\begin{itemize}
\item The conditional distribution of a continuous variable $x_s$ is Gaussian with a linear regression model for the mean, and unknown variance.
\begin{equation}
p(x_s | x_{\backslash s}, y;\Theta)=\frac{\sqrt{\beta_{ss}}}{{\sqrt{2 \pi}}}\exp{ \left(\frac{-\beta_{ss}}{2} \left(\frac{\alpha_s + \sum_{j} \rho_{sj} (y_j) - \sum_{t \neq s} \beta_{st} x_{t} }{\beta_{ss}} -x_{s}\right)^2\right)}
\label{eq:ctscond}
\end{equation}
\item The conditional distribution of a discrete variable $y_r$ with $L_r$ states is a multinomial distribution, as used in (multiclass) logistic regression. Whenever a discrete variable is a predictor, each of its  levels contribute an additive effect; continuous variables contribute linear effects.
\begin{equation}
p(y_r| y_{\backslash r,}, x;\Theta) =\frac{\exp{\left(\sum_{s} \rho_{sr}(y_r) x_{s} +\phi_{rr} (y_r,y_r) +\sum_{j\neq r} \phi_{rj}(y_r,y_{j}) \right)}}{\sum_{l=1}^{L_r } \exp{\left(\sum_{s} \rho_{sr}(l) x_{s} +\phi_{rr} (l,l) +\sum_{j\neq r} \phi_{rj}(l,y_{j}) \right)} } \label{eq:discond}
\end{equation}
\end{itemize}
Taking the negative log of both gives us
\begin{align}
-\log{p(x_s | x_{\backslash s}, y;\Theta)}&= -\frac{1}{2} \log{\beta_{ss}} +\frac{\beta_{ss}}{2} \left(\frac{\alpha_s}{\beta_{ss}}+\sum_{j}\frac{\rho_{sj}(y_j)}{\beta_{ss}} - \sum_{t\neq s} \frac{\beta_{st}}{\beta_{ss}} x_{t} - x_{s}\right)^2\\
-\log{p(y_r| y_{\backslash r,}, x;\Theta)}&=-\log{\frac{\exp{\left(\sum_{s} \rho_{sr}(y_r) x_{s} +\phi_{rr} (y_r,y_r) +\sum_{j\neq r} \phi_{rj}(y_r,y_{j}) \right)}}{\sum_{l=1}^{L_r } \exp{\left(\sum_{s} \rho_{sr}(l) x_{s} +\phi_{rr} (l,l) +\sum_{j\neq r} \phi_{rj}(l,y_{j}) \right)} }}
\end{align}
A generic parameter block, $\theta_{uv}$, corresponding to an edge $(u,v)$ appears twice in the pseudolikelihood, once for each  of the conditional distributions $p(z_u|z_v)$ and $p(z_v|z_u)$. 
\begin{proposition}
The negative log pseudolikelihood in \eqref{eq:negpl} is jointly convex in all the parameters $\{\beta_{ss},\beta_{st}, \alpha_{s}, \phi_{rj}, \rho_{sj}\}$ over the region $\beta_{ss}>0$. 
\label{prop:cvx}
\end{proposition}
We prove Proposition~\ref{prop:cvx} in the Supplementary Materials.
\subsection{Separate node-wise regression}
A simple approach to parameter estimation is via separate node-wise
regressions; a generalized linear model is used to estimate
$p(z_s|z_{\bs s})$ for each $s$. Separate regressions were used in
\cite{meinshausen06} for the Gaussian graphical model and
\cite{ravikumar2010} for the Ising model.  The method can be thought of as
an asymmetric form of the pseudolikelihood since the pseudolikelihood
enforces that the parameters are shared across the conditionals. Thus
the number of parameters estimated in the separate regression is
approximately double that of the pseudolikelihood, so we expect that
the pseudolikelihood outperforms at low sample sizes and low
regularization regimes. The node-wise regression was used as our
baseline method since it is straightforward to extend it to the mixed
model. As we predicted, the pseudolikelihood or joint procedure
outperforms separate regressions; see top left box of Figures
\ref{fig:sepvspln100} and
\ref{fig:sepvspln10000}. \cite{liu2012distributed,liu2011learning}
confirm that the separate regressions are outperformed by
pseudolikelihood in numerous synthetic settings.

Concurrent work of
\cite{yang2012graphical,yang2013graphical} extend the separate
node-wise regression model from the special cases of Gaussian and
categorical regressions to generalized linear models, where the
univariate conditional distribution of each node $p(x_s|x_{\backslash
  s})$ is specified by a generalized linear model (e.g. Poisson,
categorical, Gaussian). By specifying the conditional distributions,
\cite{besag1974spatial} show that the joint distribution is also
specified. Thus another way to justify our mixed model is to define
the conditionals of a continuous variable as Gaussian linear
regression and the conditionals of a categorical variable as multiple
logistic regression and use the results in \cite{besag1974spatial} to
arrive at the joint distribution in \eqref{eq:jointdensity}.
However, the neighborhood selection algorithm in
\cite{yang2012graphical,yang2013graphical} is restricted to models of
the form $ p(x) \propto \exp\left( \sum_s \theta_s x_s +\sum_{s,t}
  \theta_{st} x_s x_t +\sum_s C(x_s) \right).
$
In particular, this procedure cannot be applied to edge selection in
our pairwise mixed model in \eqref{eq:jointdensity} or the
categorical model in  \eqref{eq:discrete} with greater than 2
states. Our baseline method of separate regressions is closely related
to the neighborhood selection algorithm they proposed; the baseline
can be considered as a generalization of
\cite{yang2012graphical,yang2013graphical} to allow for more general
pairwise interactions with the appropriate regularization to select
edges. Unfortunately, the theoretical results in
\cite{yang2012graphical, yang2013graphical} do not apply to the
baseline nodewise regression method, nor the joint pseudolikelihood.

\section{Conditional Independence and Penalty Terms}
\label{sec:penalty}
In this section, we show how to incorporate edge selection into the maximum likelihood or pseudolikelihood procedures. In the graphical representation of probability distributions, the absence of an edge $e=(u,v)$ corresponds to a conditional independency statement that variables $x_u$ and $x_v$ are conditionally independent given all other variables \citep{koller2009}. We would like to maximize the likelihood subject to a penalization on the number of edges since this results in a sparse graphical model. In the pairwise mixed model, there are 3 type of edges
\begin{enumerate}
\item $\beta_{st}$ is a scalar that corresponds to an edge from $x_s$
  to $x_t$. $\beta_{st}=0$ implies $x_s$ and $x_t$ are conditionally
  independent given all other variables. This parameter is in two
  conditional distributions, corresponding to either $x_s$ or $x_t$ is
  the response variable, $p(x_{s} |x_{\backslash s}, y;\Theta)$ and
  $p(x_{t} |x_{\backslash t}, y;\Theta)$.
\item $\rho_{sj}$ is a vector of length $L_j$. If $\rho_{sj}(y_j) =0 $
  for all values of $y_j$, then $y_j$ and $x_s$ are conditionally
  independent given all other variables. This parameter is in two
  conditional distributions, corresponding to either $x_s$ or $y_j$ being
  the response variable: $p(x_{s} |x_{\backslash s}, y;\Theta)$ and
  $p(y_{j} |x, y_{\backslash j};\Theta)$.
\item $\phi_{rj}$ is a matrix of size $L_r \times L_j$. If $\phi_{rj}
  (y_r, y_j) =0 $ for all values of $y_r$ and $y_j$, then $y_r$ and
  $y_j$ are conditionally independent given all other variables. This
  parameter is in two conditional distributions, corresponding to
  either $y_r$ or $y_j$ being the response variable, $p(y_{r} |x,
  y_{\backslash r};\Theta)$ and $p(y_{j} |x, y_{\backslash
    j};\Theta)$.
\end{enumerate}
For conditional independencies that involve discrete variables, the absence of that edge
requires that the entire matrix $\phi_{rj}$ or vector $\rho_{sj}$ is
$0$ \footnote{If $\rho_{sj} (y_j) =constant$, then $x_s$ and $y_j$ are also conditionally independent. However, the unpenalized term $\alpha$ will absorb the constant, so the estimated $\rho_{sj} (y_j)$ will never be constant for $\lambda>0$.}. The form of the pairwise mixed model motivates the following
regularized optimization problem
\begin{align}
\minimize_{\Theta}~ \ell_{\lambda}(\Theta)=\ell(\Theta)+\lambda\left(\sum_{s<t} \indicator{\beta_{st}\not = 0}  +\sum_{sj} \indicator{\rho_{sj} \not \equiv 0} + \sum_{r<j} \indicator{\phi_{rj} \not \equiv 0} \right).
\end{align}
All parameters that correspond to the same edge are grouped in the same indicator function. This problem is non-convex, so we replace the $l_{0}$ sparsity and group sparsity penalties with the appropriate convex relaxations. For scalars, we use the absolute value ($l_1$ norm), for vectors we use the $l_2$ norm, and for matrices we use the Frobenius norm. This choice corresponds to the standard relaxation from group $l_0$ to group $l_1/l_2$ (group lasso) norm \citep{bach2011optimization,yuan2006model},
\begin{align}
\minimize_{\Theta}\ \ell_{\lambda}(\Theta)=\ell(\Theta)+ \lambda \left(\sum_{s=1}^p \sum_{t=1}^{s-1} |\beta_{st}| +\sum_{s=1}^p \sum_{j=1}^q \norm{\rho_{sj}}_2 +\sum_{j=1}^q \sum_{r=1}^{j-1} \norm{\phi_{rj}}_F \right).
\label{eq:penpl}
\end{align}

\begin{figure}
\centering
\includegraphics[width=.4\textwidth]{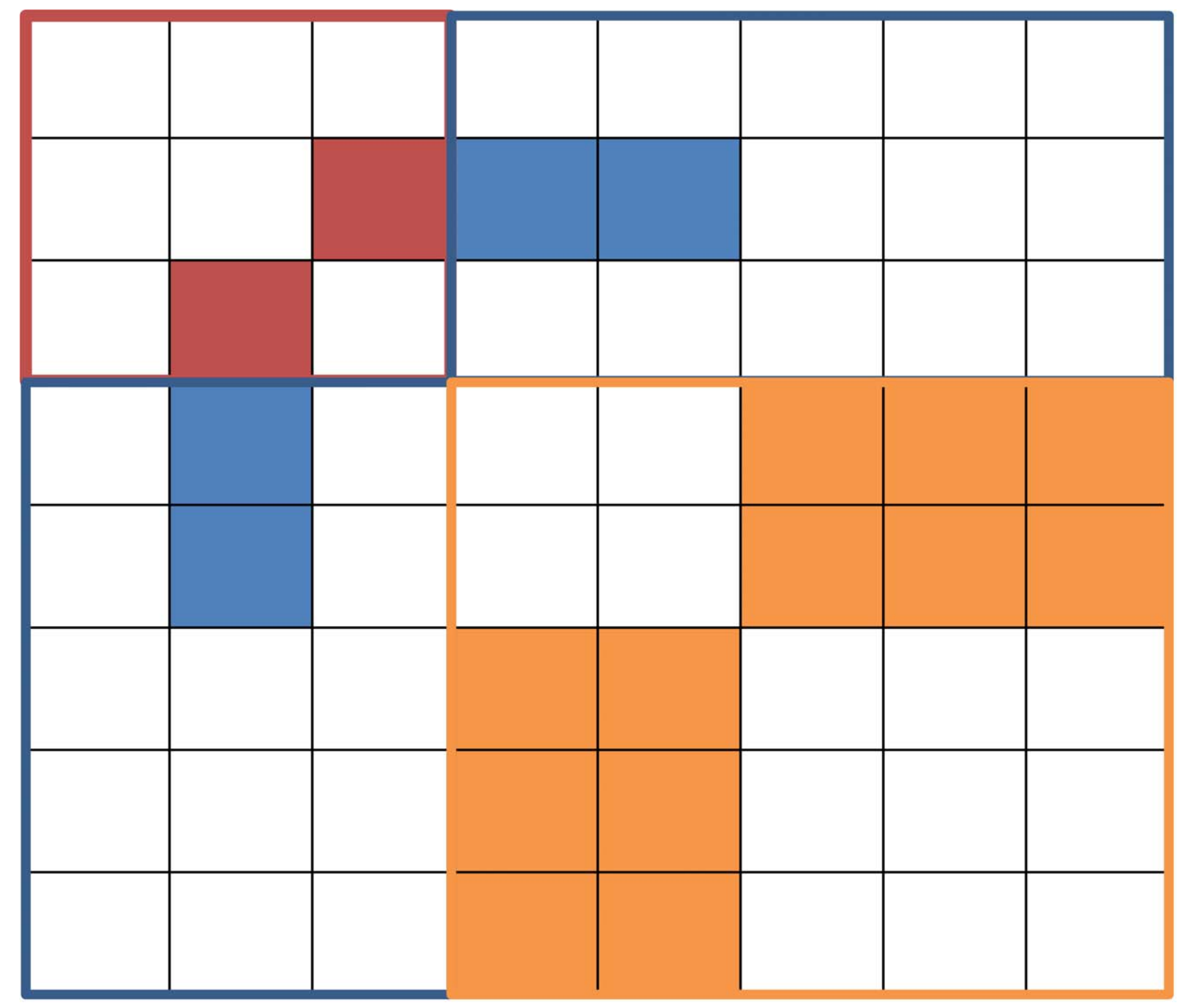}
\caption{\small\em Symmetric matrix represents the parameters $\Theta$ of the model. This example has $p=3$, $q=2$, $L_1 =2$ and $L_2 =3$. The red square corresponds to the continuous graphical model coefficients $B$ and the solid red square is the scalar $\beta_{st}$. The blue square corresponds to the coefficients $\rho_{sj}$ and the solid blue square is a vector of parameters $\rho_{sj} (\cdot)$. The orange square corresponds to the coefficients $\phi_{rj}$ and the solid orange square is a matrix of parameters $\phi_{rj} (\cdot, \cdot)$. The matrix is symmetric, so each parameter block appears in two of the conditional probability regressions.}
\end{figure}

\section{Calibrated regularizers}
\label{sec:calibration}
In  \eqref{eq:penpl} each of the group penalties are treated as equals, irrespective of the size of the group. We suggest a calibration or weighting scheme to balance the load in a more equitable way. We introduce weights for each group of parameters and show how to choose the weights such that each parameter set is treated equally under $p_F$, the fully-factorized independence model \footnote{Under the independence model $p_F$ is fully-factorized $ p(x,y) = \prod_{s=1}^{p} p(x_{s}) \prod_{r=1} ^{q} p(y_{r})$}
\begin{align}
\minimize_{\Theta}\ \ell(\Theta)+ \lambda \left(\sum_{t=1}^p \sum_{t=1}^{s-1} w_{st}|\beta_{st}| +\sum_{s=1}^p \sum_{j=1}^q w_{sj} \norm{\rho_{sj}}_2 +\sum_{j=1}^q \sum_{r=1}^{j-1} w_{rj}\norm{\phi_{rj}}_F \right)
\label{eq:weightpenpl}
\end{align}
Based on the KKT conditions \citep{friedman2007pathwise}, the parameter group $\theta_g$ is non-zero if 
\begin{gather*}
\norm{\frac{\partial \ell}{\partial \theta_{g}}} > \lambda w_{g}
\end{gather*}
where $\theta_{g}$ and $w_{g}$ represents one of the parameter groups and its corresponding weight.
Now $\frac{\partial \ell}{\partial \theta_{g}}$ can be viewed as a generalized residual, and for different groups these are different dimensions---e.g. scalar/vector/matrix. So even under the independence model (when all terms should be zero), one might expect some terms $\norm{\frac{\partial \ell}{\partial \theta_{g}}}$ to have a better than random chance of being non-zero  (for example, those of bigger dimensions). 
Thus for all parameters to be on equal footing, we would like to choose the weights $w$ such that
\begin{equation} 
 E_{p_F}\norm{\frac{\partial \ell}{\partial \theta_{g}}}=\text{constant}\times w_{g},
\label{eq:ideal-weights} 
 \end{equation}
where $p_F$ is the fully factorized (independence) model. We will refer to these as the exact weights. These weights do not have a closed form expression, so we propose an approximation to these.
It is simpler to compute in closed form $E_{p_F}\norm{\frac{\partial \ell}{\partial \theta_{g}}}^2$, so we may use approximate weights
\begin{equation}
w_{g} \propto \sqrt{E_{p_F}\norm{\frac{\partial \ell}{\partial \theta_{g}}}^2}
\label{eq:approx-weights}
\end{equation}

In the supplementary material, we show that the approximate weights \eqref{eq:approx-weights} are
\begin{equation}
\begin{aligned}
w_{st}&=\sigma_{s} \sigma_{t}\\
w_{sj}&=\sigma_{s} \sqrt{ \sum_{a} p_{a} (1-p_{a})}\\
w_{rj}&=\sqrt{ \sum_{a} p_{a} (1-p_{a}) \sum_{b} q_{b} (1-q_{b})}
\label{eq:approx-weights}
\end{aligned}
\end{equation}
$\sigma_{s}$ is the standard deviation of the continuous variable $x_{s}$. $p_{a} = Pr(y_{r}=a)$ and $q_b= Pr(y_j =b)$ . For all $3$ types of parameters, the weight has the form of $w_{uv} = \mathbf{tr}(\mathbf{cov}(z_{u})) \mathbf{tr} (\mathbf{cov}(z_{v}))$, where $z$ represents a generic variable and $\mathbf{cov}(z)$ is the variance-covariance matrix of $z$.

We conducted a simulation study to show that calibration is needed. Consider a model with $4$ independent variables: 2 continuous with variance $10$ and $1$, and 2 discrete variables with $10$ and $2$ levels.

\begin{figure}
\begin{tabular}{|l|c|c|c|c|c|c|}
\hline
&$\norm{\frac{\partial \ell}{\partial \phi_{12}}}_F$ &
$\norm{\frac{\partial \ell}{\partial \rho_{11}}}_2$&
$\norm{\frac{\partial \ell}{\partial \rho_{21}}}_2$&
$\norm{\frac{\partial \ell}{\partial \rho_{12}}}_2$&
$\norm{\frac{\partial \ell}{\partial \rho_{22}}}_2$&
$\left|\frac{\partial \ell}{\partial \beta_{12}}\right|$\\
\hline
 Exact weights $w_g$ \eqref{eq:ideal-weights}& 0.18 & 0.63 &  0.19 &  0.47 & 0.15 & 0.53\\
\hline
Approximate weights $w_g$ \eqref{eq:approx-weights} & 0.13 & 0.59 & 0.18 & 0.44 & 0.13 & 0.62\\
\hline
\end{tabular}
\caption{Row 1 shows the exact weights $w_g$ computed via Equation \eqref{eq:ideal-weights} using Monte Carlo simulation. These are the ideal weights, but they are not available in closed-form. Row 2 shows the approximate weights computed using Equation \eqref{eq:approx-weights}. As we can see, the weights are far from uniform, and the approximate weights are close to the exact weights. }
\label{tab:calibration}
\end{figure}

There are $6$ candidate edges in this model and from row 1 of Table \ref{tab:calibration} we can see the sizes of the gradients are different. In fact, the ratio of the largest gradient to the smallest gradient is greater than $4$. The edges $\rho_{11}$ and $\rho_{12}$ involving the first continuous variable with variance $10$ have large edge weights, than the corresponding edges, $\rho_{21}$ and $\rho_{22}$ involving the second continuous variable with variance $1$. Similarly, the edges involving the first discrete variable with $10$ levels are larger than the edges involving the second discrete variable with $2$ levels. This reflects our intuition that larger variance and longer vectors will have larger norm.

Had the calibration weights been chosen via Equation \ref{eq:ideal-weights}, $w=\{w_g\}_g$ and the vector of gradients $\nabla \ell=\{\norm{\frac{\partial \ell}{\partial \theta_g}}\}_g$ would have cosine similarity, $sim(u,v) = \frac{u^T v}{\norm{u}\norm{v}}=1$. The approximate weights we used are from Equation \eqref{eq:approx-weights} and have cosine similarity $$sim(w,\nabla \ell)=.993,$$ which is extremely close to $1$. Thus the calibration weights are effective in accounting for the size and variances of each edge group.

In the second simulation study, we used a model with $3$ independent variables: one continuous, and 2 discrete variables with $2$ and $4$ levels. There are $3$ candidate edges, and we computed the probability that a given edge would be the first allowed to enter the model using $3$ different calibration schemes. From Table \ref{tab:calibration-2}, we see that the uncalibrated regularizer would select the edge between the continuous variable and the $4$ level discrete variable about $73\%$ of the time. A perfect calibration scheme would select each edge $33\%$ of the time. We see that the two proposed calibration schemes are an improvement over the uncalibrated regularizer.

The exact weights do not have a simple closed form expression, but they can be easily computed via Monte Carlo. This can be done by simulating independent Gaussians and multinomials with the appropriate marginal variance $\sigma_s$ and marginal probabilities $p_a$, then approximating the expectation in \eqref{eq:ideal-weights} by an average. The computational cost of this procedure is negligible compared to fitting the mixed model, so the exact weights can also be used.
\begin{table}
\centering
\begin{tabular}{|l|c|c|c|}
\hline
&$\rho_{11}$ &
$\rho_{12}$&
$\phi_{12}$\\
\hline
No Calibration $w_g =1$ & 0.1350  &  0.7280   & 0.1370\\
\hline
Exact $w_g$  \eqref{eq:ideal-weights} & 0.3180 &   0.3310 &   0.3510\\
\hline
Approximate $w_g$  \eqref{eq:approx-weights} &    0.2650 &  0.2650 &0.4700\\
\hline
\end{tabular}
\caption{Frequency an edge is the first selected by the group lasso regularizer. The group lasso with equal weights is highly unbalanced, as seen in row 1. The weighing scheme with the weights from \eqref{eq:ideal-weights} is very good, and selects the edges with probability close to the ideal $\frac{1}{3}$. The approximate weighing scheme of \eqref{eq:approx-weights} is an improvement over not calibrating; however, not as good as the weights from \eqref{eq:ideal-weights}.  }
\label{tab:calibration-2}
\end{table}

\section{Model Selection Consistency}
\label{sec:msc}
In this section, we study the model selection consistency, whether the correct edge set is selected and the parameter estimates are close to the truth, of the pseudolikelihood and maximum likelihood estimators. Consistency can be established using the framework first developed in \cite{ravikumar2010} and later extended to general M-estimators by \cite{lee2013model}. Instead of stating the full results and proofs, we will illustrate the type of theorems that can be shown and defer the rigorous statements to the Supplementary Material. 

First, we define some notation. Recall that $\Theta$ is the vector of parameters being estimated $\{\beta_{ss},\beta_{st}, \alpha_{s}, \phi_{rj}, \rho_{sj}\}$, $\Theta^\star$ be the true parameters that estimated the model, and $Q= \nabla^2 \ell (\Theta^\star)$. Both maximum likelihood and pseudolikelihood estimation procedures can be written as a convex optimization problem of the form 
\begin{align}
\minimize\ \ell(\Theta) + \lambda \sum_{g \in G} \norm{\Theta_g }_2
\label{eq:generic_estimator}
\end{align}
where $\ell(\theta) = \{ \ell_{ML}, \ell_{PL}\}$ is one of the two log-likelihoods. The regularizer $$
\sum_{g \in G} \norm{\Theta_g} = \lambda \left(\sum_{s=1}^p \sum_{t=1}^{s-1} |\beta_{st}| +\sum_{s=1}^p \sum_{j=1}^q \norm{\rho_{sj}}_2 +\sum_{j=1}^q \sum_{r=1}^{j-1} \norm{\phi_{rj}}_F \right).
$$ The set $G$ indexes the edges $\beta_{st}$, $\rho_{sj}$, and $\phi_{rj}$, and $\Theta_g$ is one of the three types of edges. Let $A$ and $I$ represent the active and inactive groups in $\Theta$, so $\Theta^\star _g \neq 0 $ for any $g \in A$ and $\Theta_g ^\star =0$ for any $g \in I$.

Let $\hat{\Theta}$ be the minimizer to Equation \eqref{eq:generic_estimator}. Then $\hat{\Theta}$ satisfies,
\begin{enumerate}
\item $\norm{\hat{\Theta} - \Theta^\star}_2 \le C \sqrt{\frac{|A| \log |G|}{n}}$
\item $\hat{\Theta}_g = 0$ for $g \in I$.
\end{enumerate}
The exact statement of the theorem is given in the Supplementary Material.

\section{Optimization Algorithms}
\label{sec:optalg}
In this section, we discuss two algorithms for solving \eqref{eq:penpl}: the proximal gradient and the proximal newton methods. This is a convex optimization problem that decomposes into the form $f(x)+g(x)$, where $f$ is smooth and convex and $g$ is convex but possibly non-smooth. In our case $f$ is the negative log-likelihood or negative log-pseudolikelihood and $g$ are the group sparsity penalties.

Block coordinate descent is a frequently used method when the non-smooth function $g$ is the $l_1$ or group $l_1$. It is especially easy to apply when the function $f$ is quadratic, since each block coordinate update can be solved in closed form for many different non-smooth $g$ \citep{friedman2007pathwise}. The smooth $f$ in our particular case is not quadratic, so each block update cannot be solved  in closed form. However in certain problems (sparse inverse covariance), the update can be approximately solved by using an appropriate inner optimization routine \citep{glasso}.
\subsection{Proximal Gradient}
Problems of this form are well-suited for the proximal gradient and accelerated proximal gradient algorithms as long as the proximal operator of $g$ can be computed \citep{combettes2011proximal,beck2010gradient}
\begin{align}
prox_{t}(x)=\argmin_{u} \frac{1}{2t}\norm{x-u}^2+g(u)
\end{align}
For the sum of $l_2$ group sparsity penalties considered, the proximal operator takes the familiar form of soft-thresholding and group soft-thresholding \citep{bach2011optimization}. Since the groups are non-overlapping, the proximal operator simplifies to scalar soft-thresholding for $\beta_{st}$ and group soft-thresholding for $\rho_{sj}$ and $\phi_{rj}$.

The class of proximal gradient and accelerated proximal gradient algorithms is directly applicable to our problem. These algorithms work by solving a first-order model at the current iterate $x_{k}$
\begin{align}
\argmin_{u}~ &f(x_{k})+\nabla f(x_{k}) ^{T}(u-x_{k}) +\frac{1}{2t}\norm{u-x_{k}}^2 +g(u)\\
&=\argmin_{u}~ \frac{1}{2t}\norm{u-\left(x_{k}-t\nabla f(x_{k})\right)}^2+g(u)\\
&=prox_{t} (x_{k} -t\nabla f(x_{k}))
\end{align}
The proximal gradient iteration is given by $x_{k+1} = prox_{t} \left( x_{k} - t \nabla f(x_{k}) \right)$ where $t$ is determined by line search. The theoretical convergence rates and properties of the proximal gradient algorithm and its accelerated variants are well-established \citep{beck2010gradient}. The accelerated proximal gradient method achieves linear convergence rate of $O(c^k)$ when the objective is strongly convex and the sublinear rate $O(1/k^2)$ for non-strongly convex problems.

The TFOCS framework \citep{becker2011} is a package that allows us to experiment with $6$ different variants of the accelerated proximal gradient algorithm. The TFOCS authors found that the Auslender-Teboulle algorithm exhibited less oscillatory behavior, and proximal gradient experiments in the next section were done using the Auslender-Teboulle implementation in TFOCS. 
\subsection{ Proximal Newton Algorithms}
\label{sec:proxnewton}
The class of proximal Newton algorithms is a 2nd order analog of the proximal gradient algorithms with a quadratic convergence rate \citep{lee2012proximal,schmidt2010,schmidt2011}. It attempts to incorporate 2nd order information about the smooth function $f$ into the model function. At each iteration, it minimizes a quadratic model centered at $x_{k}$
\begin{align}
&\argmin_{u}~ f(x_{k})+\nabla f(x_{k})^{T} (u-x_{k})+\frac{1}{2t}(u-x_{k})^{T}H(u-x_{k}) +g(u)\\
&=\argmin_{u}~ \frac{1}{2t}\left(u-x_{k}+tH^{-1} \nabla f(x_{k}) \right)^{T} H \left(u-x_{k}+tH^{-1} \nabla f(x_{k}) \right)+g(u)\\
&=\argmin_{u}~ \frac{1}{2t} \norm{u-\left( x_{k}-tH^{-1}\nabla f(x_{k}) \right)}^{2}_{H}+g(u)\\
&:=Hprox_{t}\left(x_{k}-tH^{-1}\nabla f(x_{k}) \right)
\mbox{ where } H= \nabla^{2} f(x_{k})
\end{align}
\begin{algorithm}
\caption{Proximal Newton}
\begin{algorithmic}
\Repeat
\State Solve subproblem $p_{k} = Hprox_{t}\left(x_{k}-tH_{k}^{-1}\nabla f(x_{k}) \right) - x_{k}$ using TFOCS.
\State Find $t$ to satisfy Armijo line search condition with parameter $\alpha$
\begin{displaymath}
f(x_{k} + tp_{k}) +g(x_{k} +tp_{k}) \leq f(x_{k})+g(x_{k}) - \frac{t\alpha}{2} \norm{p_{k}}^2
\end{displaymath}
\State Set $x_{k+1}=x_{k}+tp_{k}$
\State $k=k+1$
\Until{$\frac{\norm{x_{k} - x_{k+1}}}{\norm{x_{k} }} < tol$} 
\end{algorithmic}
\end{algorithm}
The $Hprox$ operator is analogous to the proximal operator, but in the $\norm{\cdot}_{H}$-norm. It simplifies to the proximal operator if $H=I$, but in the general case of positive definite $H$ there is no closed-form solution for many common non-smooth $g(x)$ (including $l_{1}$ and group $l_{1}$). However if the proximal operator of $g$ is available, each of these sub-problems can be solved efficiently with proximal gradient. In the case of separable $g$, coordinate descent is also applicable. Fast methods for solving the subproblem $Hprox_{t} (x_k - t H^{-1} \nabla f(x_{k}))$ include coordinate descent methods, proximal gradient methods, or Barzilai-Borwein \citep{friedman2007pathwise,combettes2011proximal,beck2010gradient, wright2009sparse}. The proximal Newton framework allows us to bootstrap many previously developed solvers to the case of arbitrary loss function $f$.

Theoretical analysis in \cite{lee2012proximal} suggests that proximal Newton methods generally require fewer outer iterations (evaluations of $Hprox$) than first-order methods while providing higher accuracy because they incorporate 2nd order information. We have confirmed empirically that the proximal Newton methods are faster when $n$ is very large or the gradient is expensive to compute (e.g. maximum likelihood estimation). Since the objective is quadratic, coordinate descent is also applicable to the subproblems. The hessian matrix $H$ can be replaced by a quasi-newton approximation such as BFGS/L-BFGS/SR1. In our implementation, we use the \texttt{PNOPT} implementation \citep{lee2012proximal}. 
\subsection{Path Algorithm}
Frequently in machine learning and statistics, the regularization parameter $\lambda$ is heavily dependent on the dataset. $\lambda$ is generally chosen via cross-validation or holdout set performance, so it is convenient to provide solutions over an interval of $[\lambda_{min} , \lambda_{max}]$. We start the algorithm at $\lambda_1 = \lambda_{max}$ and solve, using the previous solution as warm start, for $\lambda_2 > \ldots> \lambda_{min}$. We find that this reduces the cost of fitting an entire path of solutions (See Figure \ref{fig:modelselect}). $\lambda_{max}$ can be chosen as the smallest value such that all parameters are $0$ by using the KKT equations \citep{friedman2007pathwise}. 

\section{Conditional Model}
\label{sec:condmodel}
In addition to the variables we would like to model, there are often additional features or covariates that affect the dependence structure of the variables. For example in genomic data, in addition to expression values, we have attributes associated to each subject such as gender, age and ethnicity. These additional attributes affect the dependence of the expression values, so we can build a conditional model that uses the additional attributes as features. In this section, we show how to augment the pairwise mixed model with features.

Conditional models only model the conditional distribution
$p(z|f)$, as opposed to the joint distribution $p(z,f)$, where $z$ are
the variables of interest to the prediction task and $f$ are
features. These models are frequently used in practice \cite{lafferty2001conditional}. 

In addition to observing $x$ and $y$, we observe features $f$ and we build a graphical model for the conditional distribution $p(x,y|f)$. Consider a full pairwise model $p(x,y,f)$ of the form \eqref{eq:jointdensity}. We then choose to only model the joint distribution over only the variables $x$ and $y$ to give us $p(x,y|f)$ which is of the form
\begin{align}
p(x,y|f;\Theta)=&\frac{1}{Z(\Theta |f)} \exp  \left( \sum_{s=1}^p \sum_{t=1}^{p}-\frac{1}{2} \beta_{st} x_{s} x_{t}+\sum_{s=1}^{p}\alpha_{s} x_{s} +\sum_{s=1}^p \sum_{j=1}^{q} \rho_{sj}(y_{j})x_{s} \right. \nonumber \\ 
& \left.+\sum_{j=1}^q \sum_{r=1}^j \phi_{rj}(y_r , y_j) +\sum_{l=1}^{F} \sum_{s=1}^p \gamma_{ls} x_s f_l + \sum_{l=1}^F \sum_{r=1}^{q} \eta_{lr}(y_r) f_l \right)
\label{eq:jointdensityfeat}
\end{align}
We can also consider a more general model where each pairwise edge potential depends on the features
\begin{align}
p(x,y|f;\Theta)=\frac{1}{Z(\Theta |f)} &\exp\left(\sum_{s=1}^p\sum_{t=1}^{p}-\frac{1}{2} \beta_{st}(f) x_{s} x_{t}+\sum_{s=1}^{p}\alpha_{s}(f) x_{s} \right. \nonumber \\ & \left. +\sum_{s=1}^p\sum_{j=1}^{q}\rho_{sj}(y_{j},f)x_{s}+\sum_{j=1}^{q}\sum_{r=1}^j \phi_{rj}(y_r , y_j,f) \vphantom{\sum_{s=1}^p\sum_{t=1}^{p}} \right)
\end{align}
\eqref{eq:jointdensityfeat} is a special case of this where only the node potentials depend on features and the pairwise potentials are independent of feature values.  The specific parametrized form we consider is $\phi_{rj}(y_r, y_j,f)\equiv \phi_{rj}(y_r,y_j)$ for $r\neq j$, $\rho_{sj}(y_j,f)\equiv \rho_{sj}(y_j)$, and $\beta_{st}(f)=\beta_{st}$. The node potentials depend linearly on the feature values, $\alpha_{s}(f)=\alpha_{s} + \sum_{l=1}^{F} \gamma_{ls} x_s f_l$, and $\phi_{rr}(y_r,y_r,f) = \phi_{rr}(y_r,y_r) + \sum_{l} \eta_{lr}(y_r)$.
\section{Experimental Results}
\label{sec:exp}
We present experimental results on synthetic data, survey data and on a conditional model.
\subsection{Synthetic Experiments}
In the synthetic experiment, the training points are sampled from a true model with $10$ continuous variables and $10$ binary variables. The edge structure is shown in Figure \ref{fig:syntheticgraph}. $\lambda$ is chosen proportional to $\sqrt{\frac{\log{(p+q)}}{n}}$ as suggested by the theoretical results in Section \ref{sec:msc}. We experimented with $3$ values $\lambda = \{1,5,10\} \sqrt{\frac{\log{(p+q)}}{n}}$ and chose $\lambda =5\sqrt{\frac{\log{(p+q)}}{n}}$ so that the true edge set was recovered by the algorithm for the sample size $n=2000$.  We see from the experimental results that recovery of the correct edge set undergoes a sharp phase transition, as expected. With $n=1000$ samples, the pseudolikelihood is recovering the correct edge set with probability nearly $1$. The maximum likelihood was performed using an exact evaluation of the gradient and log-partition. The poor performance of the maximum likelihood estimator is explained by the maximum likelihood objective violating the irrepresentable condition; a similar example is discussed in \cite[Section 3.1.1]{ravikumar2010}, where the maximum likelihood is not irrepresentable, yet the neighborhood selection procedure is. The phase transition experiments were done using the proximal Newton algorithm discussed in Section \ref{sec:proxnewton}.
\begin{figure}
\centering

\begin{subfigure}{.49\textwidth}
	 
\includegraphics[width=\textwidth]{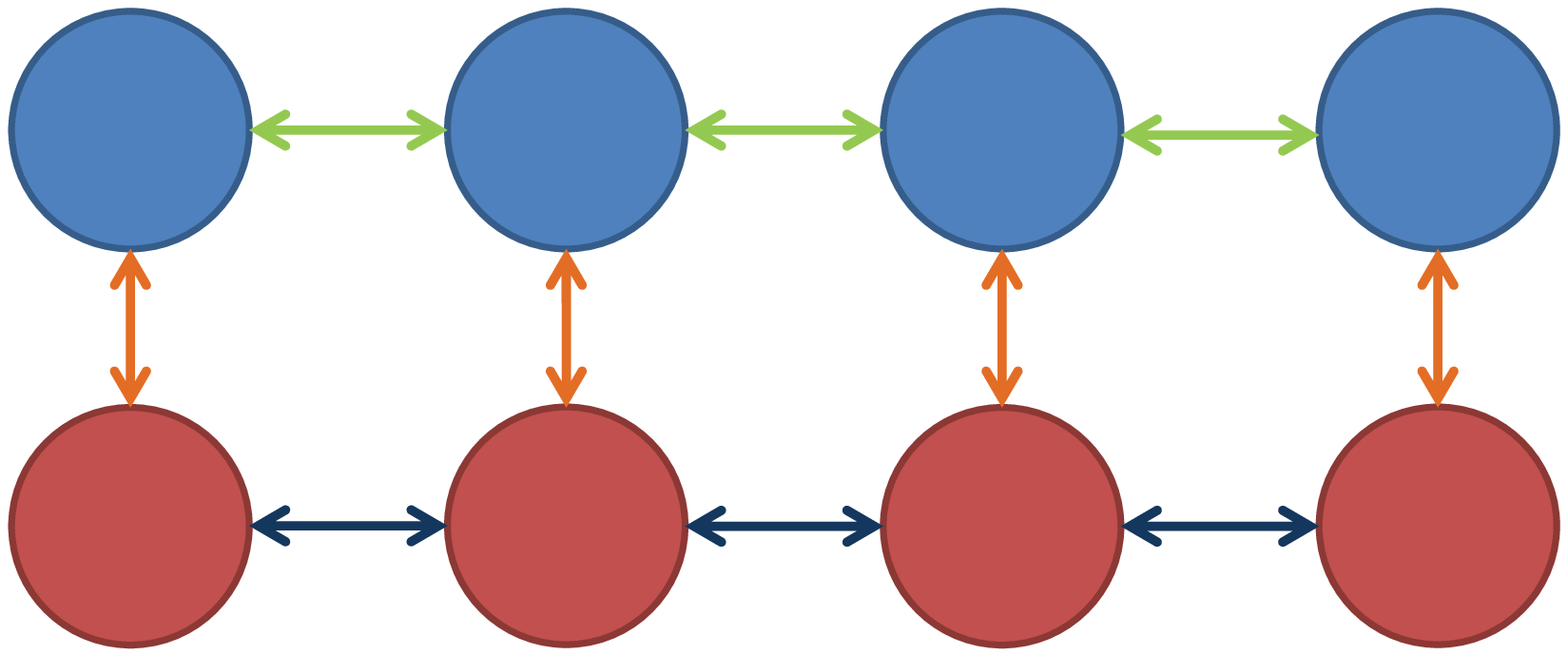}
\caption{\label{fig:syntheticgraph}}
\end{subfigure}

\begin{subfigure}{.49\textwidth}

\includegraphics[width=\textwidth]{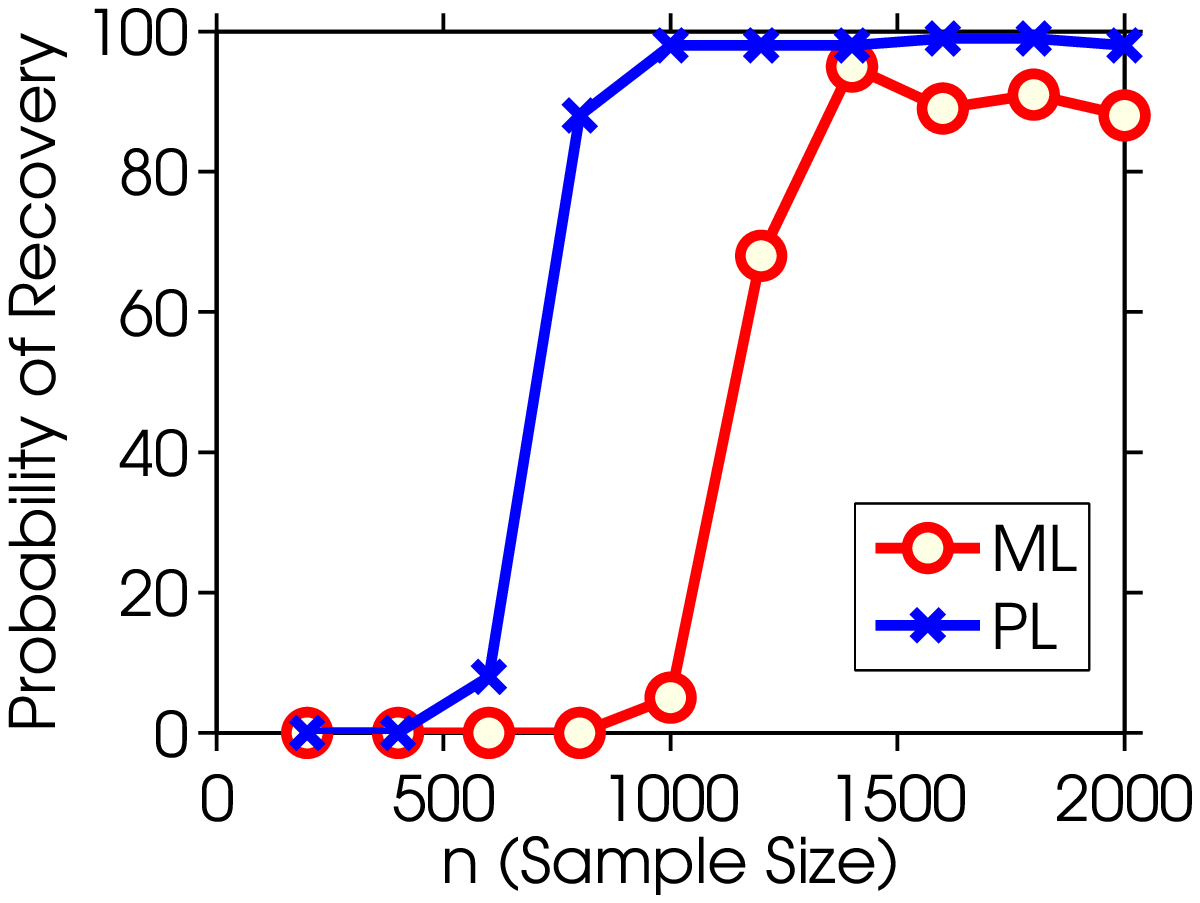}
\caption{	   \label{fig:syntheticplot}}
\end{subfigure}

 \caption{\small\em Figure \ref{fig:syntheticgraph} shows the graph used in the synthetic experiments for $p=q=4$; the experiment actually used $p$=10 and $q$=10. Blue nodes are continuous variables, red nodes are binary variables and the orange, green and dark blue lines represent the $3$ types of edges. Figure \ref{fig:syntheticplot} is a plot of the probability of correct edge recovery, meaning every true edge is selected and no non-edge is selected, at a given sample size using Maximum Likelihood and Pseudolikelihood. Results are averaged over $100$ trials.}
\end{figure}

We also run the proximal Newton algorithm for a sequence of instances with $p=q=10,50,100,500,1000$ and $n=500$. The largest instance has $2000$ variables and takes $12.5$ hours to complete. The timing results are summarized in Figure \ref{tab:timing}.
\begin{figure}
\begin{tabular}{|l|c|c|r|}
\hline
$p+q$ & Time per Iteration (sec) & Total Time (min) & Number of Iterations \\
\hline
20 &  .13 & .003 & 13\\
100 & 4.39 & 1.32 & 18\\
200 & 18.44 & 6.45 & 21\\
1000 & 245.34 & 139 & 34\\
2000 & 1025.6 & 752 & 44\\
\hline
\end{tabular}
\caption{Timing experiments for various instances of the graph in Figure \ref{fig:syntheticgraph}. The number of variables range from $20$ to $2000$ with $n=500$. }
\label{tab:timing}
\end{figure}
\subsection{Survey Experiments}
The census survey dataset we consider consists of $11$ variables, of
which $2$ are continuous and $9$ are discrete: age (continuous),
log-wage (continuous), year($7$ states), sex($2$ states),marital
status ($5$ states), race($4$ states), education level ($5$ states),
geographic region($9$ states), job class ($2$ states), health ($2$
states), and health insurance ($2$ states). The dataset was assembled
by Steve Miller of OpenBI.com from the March 2011 Supplement to
Current Population Survey data. All the evaluations are done using a
holdout test set of size $100,000$ for the survey experiments. The
regularization parameter $\lambda$ is varied over the interval
$[5\times 10^{-5}, 0.7]$ at $50$ points equispaced on log-scale for all
experiments. In practice, $\lambda$ can be chosen to minimize the holdout log pseudolikelihood.
\subsubsection{Model Selection}
\begin{figure}
\centering
\includegraphics[width=3in]{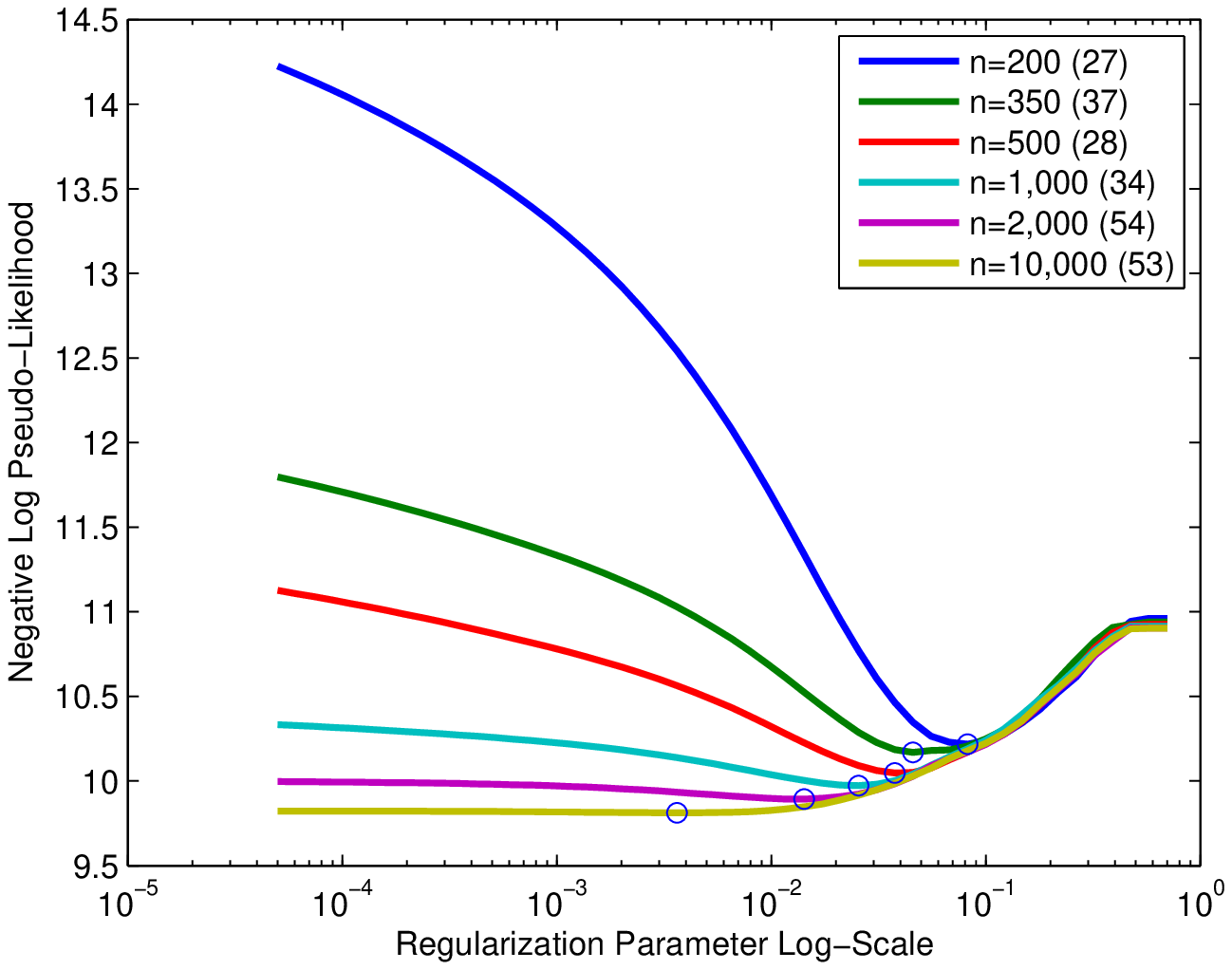}
\caption{\small\em Model selection under different training set sizes. Circle denotes the lowest test set negative log pseudolikelihood and the number in parentheses is the number of edges in that model at the lowest test negative log pseudolikelihood. The saturated model has $55$ edges.}
\label{fig:modelselect}
\end{figure}
In Figure \ref{fig:modelselect}, we study the model selection performance of learning a graphical model over the $11$ variables under different training samples sizes. We see that as the sample size increases, the optimal model is increasingly dense, and less regularization is needed. 

\subsubsection{Comparing against Separate Regressions}
A sensible baseline method to compare against is a separate regression algorithm. This algorithm fits a linear Gaussian or (multiclass) logistic regression of each variable conditioned on the rest. We can evaluate the performance of the pseudolikelihood by evaluating $-\log{p(x_{s}|x_{\backslash s},y)}$ for linear regression and $-\log{p(y_{r}|y_{\backslash r},x)}$ for (multiclass) logistic regression. Since regression is directly optimizing this loss function, it is expected to do better. The pseudolikelihood objective is similar, but has half the number of parameters as the separate regressions since the coefficients are shared between two of the conditional likelihoods. From Figures \ref{fig:sepvspln100} and \ref{fig:sepvspln10000}, we can see that the pseudolikelihood performs very similarly to the separate regressions and sometimes even outperforms regression. The benefit of the pseudolikelihood is that we have learned parameters of the joint distribution $p(x,y)$ and not just of the conditionals $p(x_{s}|y,x_{\backslash s})$. On the test dataset, we can compute quantities such as conditionals over arbitrary sets of variables $p(y_{A}, x_{B}|y_{A^{C}},x_{B^C})$ and marginals $p(x_{A},y_{B})$ \citep{koller2009}. This would not be possible using the separate regressions. 
\begin{figure}
\centering
\includegraphics[width=.9\textwidth]{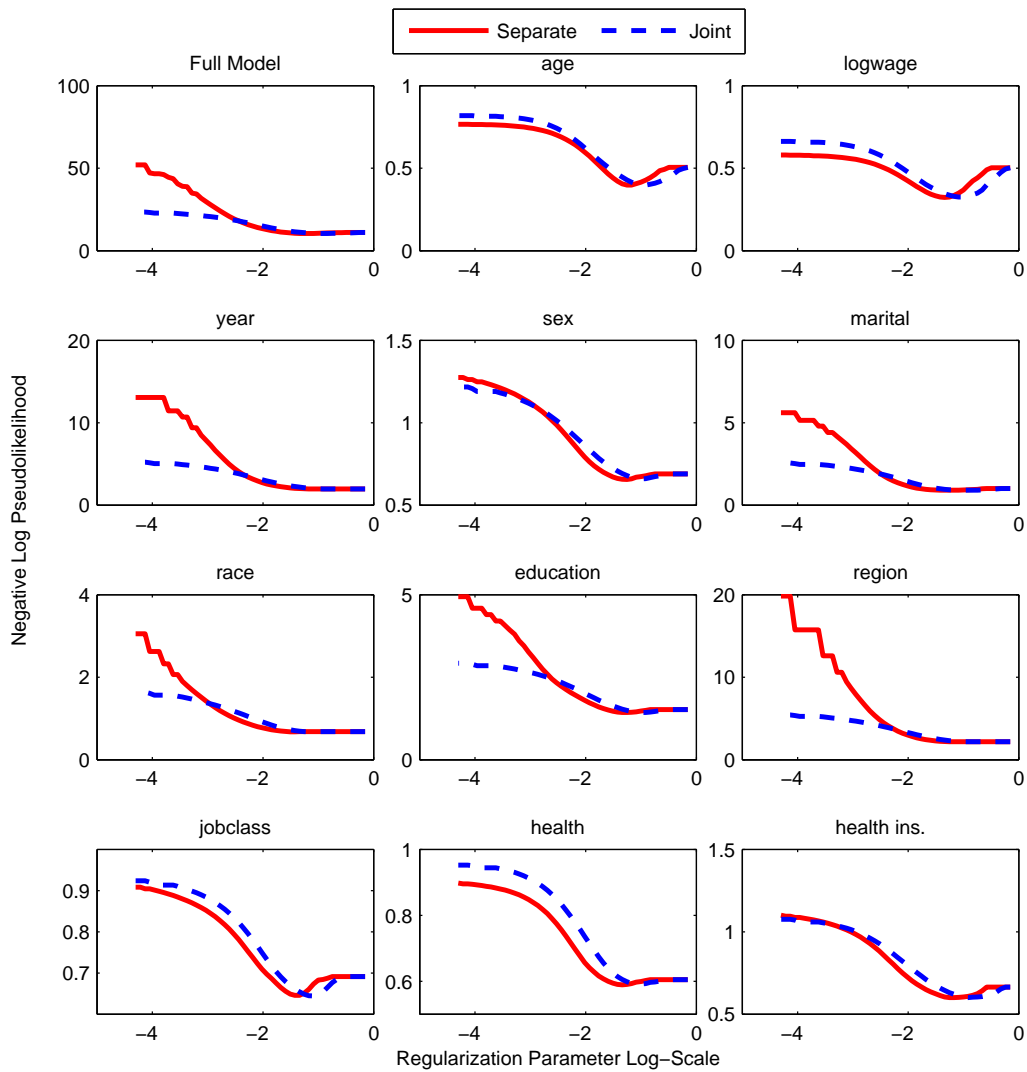}
\caption{\small\em Separate Regression vs Pseudolikelihood $n=100$. $y$-axis is the appropriate regression loss for the response variable. For low levels of regularization and at small training sizes, the pseudolikelihood seems to overfit less; this may be due to a global regularization effect from fitting the joint distribution as opposed to separate regressions. }
\label{fig:sepvspln100}
\end{figure}

\begin{figure}
\centering
\includegraphics[width=.9\textwidth]{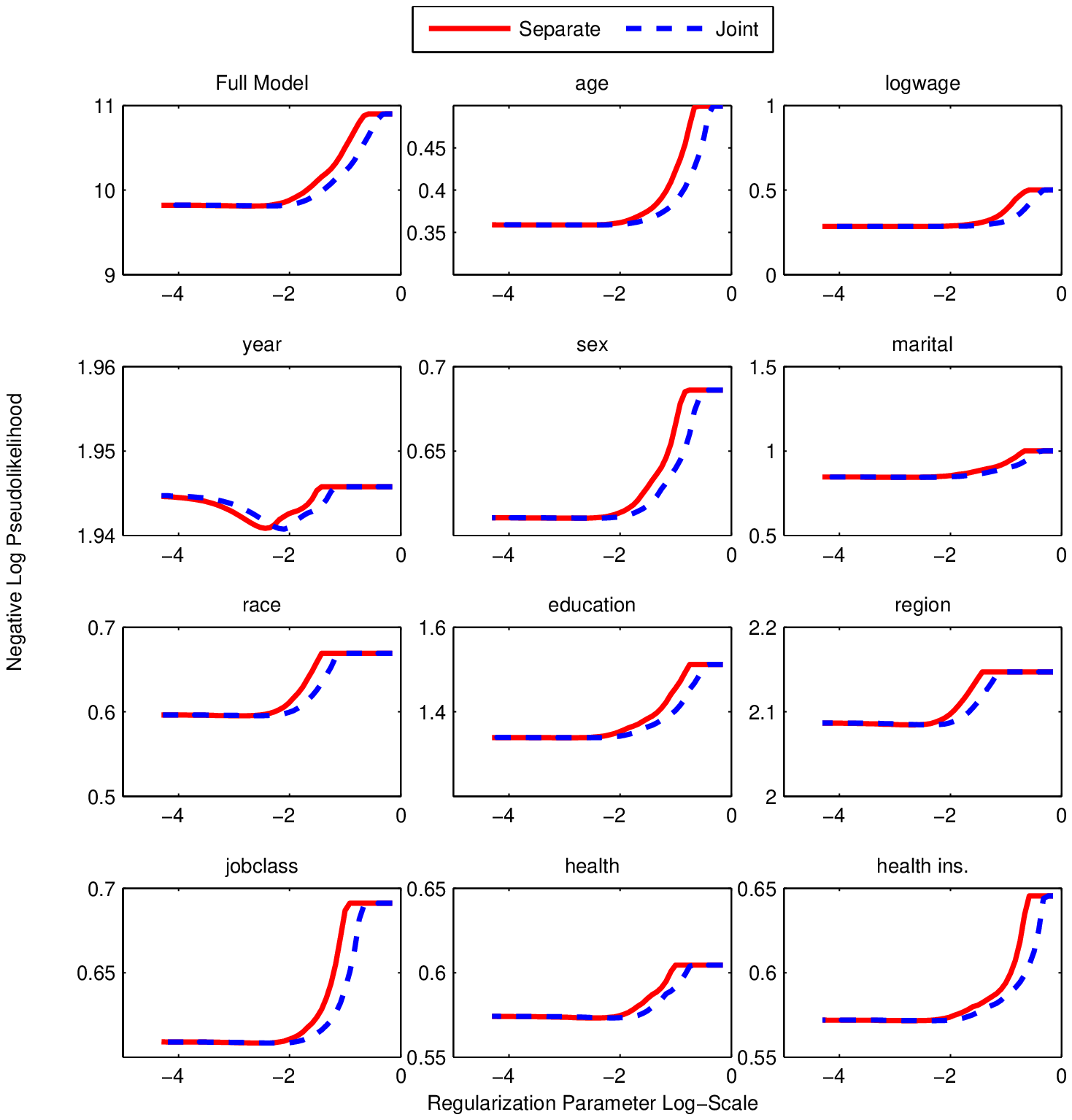}
\caption{\small\em Separate Regression vs Pseudolikelihood $n=10,000$. $y$-axis is the appropriate regression loss for the response variable. At large sample sizes, separate regressions and pseudolikelihood perform very similarly. This is expected since this is nearing the asymptotic regime.}
\label{fig:sepvspln10000}
\end{figure}

\subsubsection{Conditional Model}
Using the conditional model \eqref{eq:jointdensityfeat}, we model only the $3$ variables logwage, education($5$) and jobclass($2$). The other $8$ variables are only used as features. The conditional model is then trained using the pseudolikelihood. We compare against the generative model that learns a joint distribution on all $11$ variables. From Figure \ref{fig:condvsgen}, we see that the conditional model outperforms the generative model, except at small sample sizes. This is expected since the conditional distribution models less variables. At very small sample sizes and small $\lambda$, the generative model outperforms the conditional model. This is likely because generative models converge faster (with less samples) than discriminative models to its optimum. 
\begin{figure}

\centering
\includegraphics[width=.9\textwidth]{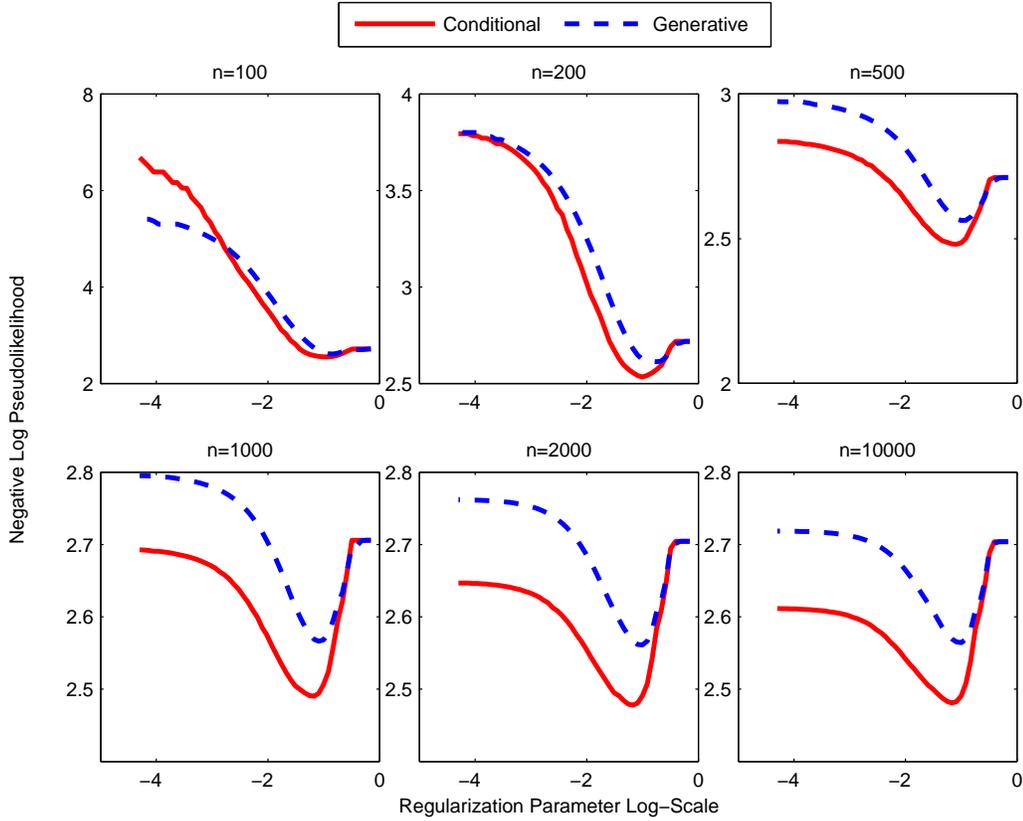}
\caption{\small\em Conditional Model vs Generative Model at various sample sizes. $y$-axis is test set performance is evaluated on negative log pseudolikelihood of the conditional model. The conditional model outperforms the full generative model at except the smallest sample size $n=100$.}
\label{fig:condvsgen}
\end{figure}
\subsubsection{Maximum Likelihood vs Pseudolikelihood}
The maximum likelihood estimates are computable for very small models such as the conditional model previously studied. The pseudolikelihood was originally motivated as an approximation to the likelihood that is computationally tractable. We compare the maximum likelihood and maximum pseudolikelihood on two different evaluation criteria: the negative log likelihood and negative log pseudolikelihood. In Figure \ref{fig:likvspl}, we find that the pseudolikelihood outperforms maximum likelihood under both the negative log likelihood and negative log pseudolikelihood. We would expect that the pseudolikelihood trained model does better on the pseudolikelihood evaluation and maximum likelihood trained model does better on the likelihood evaluation. However, we found that the pseudolikelihood trained model outperformed the maximum likelihood trained model on both evaluation criteria. Although asymptotic theory suggests that maximum likelihood is more efficient than the pseudolikelihood, this analysis is inapplicable because of the finite sample regime and misspecified model. See \citet{liang2008asymptotic} for asymptotic analysis of pseudolikelihood and maximum likelihood under a well-specified model. We also observed the pseudolikelihood slightly outperforming the maximum likelihood in the synthetic experiment of Figure \ref{fig:syntheticplot}.
\begin{figure}
\centering
\includegraphics[width=.9\textwidth]{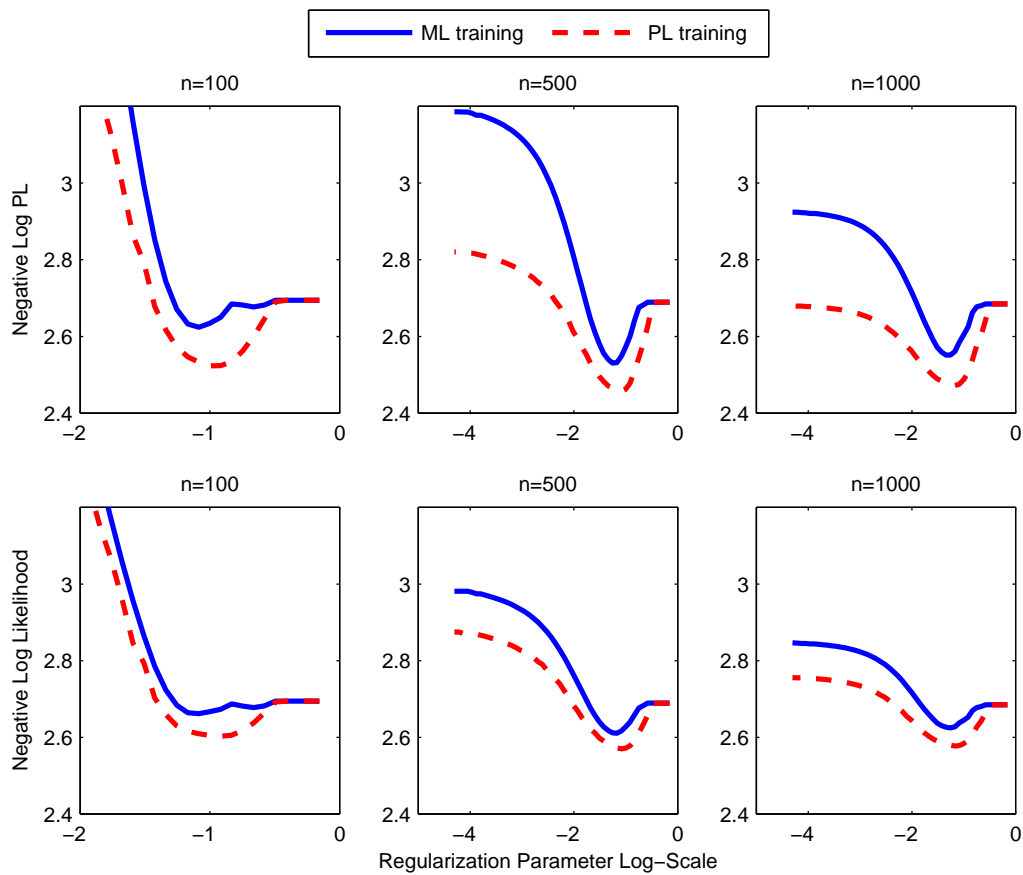}
\caption{\small\em Maximum Likelihood vs Pseudolikelihood. $y$-axis for top row is the negative log pseudolikelihood. $y$-axis for bottom row is the negative log likelihood. Pseudolikelihood outperforms maximum likelihood across all the experiments.}
\label{fig:likvspl}
\end{figure}
\section{Conclusion}
This work proposes a new pairwise mixed graphical model, which combines the Gaussian graphical model and discrete graphical model. Due to the introduction of discrete variables, the maximum likelihood estimator is computationally intractable, so we investigated the pseudolikelihood estimator. To learn the structure of this model, we use the appropriate group sparsity penalties with a calibrated weighing scheme. Model selection consistency results are shown for the mixed model using the maximum likelihood and pseudolikelihood estimators. The extension to a conditional model is discussed, since these are frequently used in practice.

We proposed two efficient algorithms for the purpose of estimating the parameters of this model, the proximal Newton and the proximal gradient algorithms. The proximal Newton algorithm is shown to scale to graphical models with $2000$ variables on a standard desktop. The model is evaluated on synthetic and the current population survey data, which demonstrates the pseudolikelihood performs well compared to maximum likelihood and nodewise regression.

For future work, it would be interesting to incorporate other discrete variables such as poisson or binomial variables and non-Gaussian continuous variables. This would broaden the scope of applications that mixed models could be used for. Our work is a first step in that direction.


\newpage

\section*{Supplementary Materials}
\subsection{Proof of Convexity}
\label{app:Proofs}
\noindent {\bf Proposition~\ref{prop:cvx}.}{\em
The negative log pseudolikelihood in  \eqref{eq:negpl} is jointly convex in all the parameters $\{\beta_{ss},\beta_{st}, \alpha_{s}, \phi_{rj}, \rho_{sj}\}$ over the region $\beta_{ss}>0$. 
}

\medskip
\begin{proof}
To verify the convexity of $\tilde{\ell}(\Theta|x,y)$, it suffices to check that each term is convex. 

\noindent$-\log{p(y_r| y_{\backslash r,}, x;\Theta)}$ is jointly convex in $\rho$ and $\phi$ since it is a multiclass logistic regression.
We now check that $-\log{p(x_s | x_{\backslash s}, y;\Theta)}$ is convex. $-\frac{1}{2} \log{\beta_{ss}} $ is a convex function. To establish that $$\frac{\beta_{ss}}{2} \left(\frac{\alpha_s}{\beta_{ss}}+\sum_{j}\frac{\rho_{sj}(y_j)}{\beta_{ss}} - \sum_{t\neq s} \frac{\beta_{st}}{\beta_{ss}} x_{t} - x_{s}\right)^2$$ is convex, we use the fact that $f(u,v)= \frac{v}{2} (\frac{u}{v} -c)^2$ is convex. Let $v=\beta_{ss}$,  $u= \alpha_s + \sum_{j} \rho_{sj} (y_j) - \sum_{t\neq s} \beta_{st} x_{t}$, and $c= x_s$. Notice that $x_s$, $\alpha_s$, $y_j$, and $x_t$ are fixed quantities and $u$ is affinely related to $\beta_{st}$ and $\rho_{sj}$. A convex function composed with an affine map is still convex, thus $\frac{\beta_{ss}}{2} \left(\frac{\alpha_s}{\beta_{ss}}+\sum_{j}\frac{\rho_{sj}(y_j)}{\beta_{ss}} - \sum_{t\neq s} \frac{\beta_{st}}{\beta_{ss}} x_{t} - x_{s}\right)^2$ is convex.

To finish the proof, we verify that $f(u,v)= \frac{v}{2} (\frac{u}{v} -c)^2 = \frac{1}{2} \frac{(u-cv)^2}{v}$ is convex over $v>0$. The epigraph of a convex function is a convex set iff the function is convex. Thus we establish that the set $C= \{ (u,v,t) |  \frac{1}{2} \frac{(u-cv)^2}{v}\le t, v>0\}$ is convex. Let $
A = \begin{bmatrix}
v&u-cv\\
u-cv&t
\end{bmatrix}.
$
The Schur complement criterion of positive definiteness says $A \succ 0$ iff $v>0$ and $t>\frac{(u-cv)^2}{v}$. The condition $A \succ 0$ is a linear matrix inequality and thus convex in the entries of $A$. The entries of $A$ are linearly related to $u$ and $v$, so $A\succ 0$ is also convex in $u$ and $v$. Therefore $v>0$ and $t>\frac{(u-cv)^2}{v}$ is a convex set.
\end{proof}
\subsection{Sampling From The Joint Distribution}
\label{app:sampling}
In this section we discuss how to draw samples $(x,y) \thicksim p(x,y)$. Using the property that $p(x,y)=p(y) p(x|y)$, we see that if $y \thicksim p(y) $ and $ x \thicksim p(x|y)$ then $(x,y) \thicksim p(x,y)$. We have that 
\begin{align}
p(y) & \propto \exp{(\sum_{r,j} \phi_{rj} (y_r, y_j) +\frac{1}{2}  \rho(y)^{T} B^{-1} \rho(y))}\\
(\rho(y))_s&= \sum_{j} \rho_{sj}(y_j)\\
p(x|y)&= No( B^{-1} (\alpha+\rho(y) ), B^{-1})
\end{align}
The difficult part is to sample $y\thicksim p(y)$ since this involves the partition function of the discrete MRF. This can be done with MCMC for larger models and junction tree algorithm or exact sampling for small models.

\subsection{Maximum Likelihood}
\label{app:mle}
The difficulty in MLE is that in each gradient step we have to compute $\hat{T}(x,y) -E_{p(\Theta)}\left[T(x,y)\right]$, the difference between the empirical sufficient statistic $\hat{T}(x,y)$  and the expected sufficient statistic. In both continuous and discrete graphical models the computationally expensive step is evaluating $E_{p(\Theta)}\left[T(x,y)\right]$. In discrete problems, this involves a sum over the discrete state space and in continuous problem, this requires matrix inversion. For both discrete and continuous models, there has been much work on addressing these difficulties. For discrete models, the junction tree algorithm is an exact method for evaluating marginals and is suitable for models with low tree width. Variational methods such as belief propagation and tree reweighted belief propagation work by optimizing a surrogate likelihood function by approximating the partition function $Z(\Theta)$ by a tractable surrogate $\widetilde{Z}(\Theta)$ \cite{wainwright2008}. In the case of a large discrete state space, these methods can be used to approximate $p(y)$ and do approximate maximum likelihood estimation for the discrete model. Approximate maximum likelihood estimation can also be done via Monte Carlo estimates of the gradients $\hat{T}(x,y) -E_{p(\Theta)}(T(x,y))$.  For continuous Gaussian graphical models, efficient algorithms based on block coordinate descent \cite{glasso,banerjee2008}  have been developed, that do not require matrix inversion.

The joint distribution and loglikelihood are:
\begin{align*}
p(x,y;\Theta)&= \exp{(-\frac{1}{2} x^{T} B x +(\alpha+\rho(y))^{T} x+\sum_{(r,j)}\phi_{rj}(y_r,y_j))}/Z(\Theta)\\
\ell(\Theta)&=\left(\frac{1}{2} x^{T} B x -(\alpha+\rho(y))^{T} x-\sum_{(r,j)}\phi_{rj}(y_r,y_j)\right)\\
&+\log( \sum_{y'} \int{dx \exp{(-\frac{1}{2}x^{T} B x +(\alpha+\rho(y'))^{T} x )}} \exp(\sum_{(r,j)} \phi_{rj}(y'_r,y'_j)) )\\
\end{align*}
The derivative is
\begin{align*}
\frac{\partial \ell}{\partial B} &= \frac{1}{2} x x^{T}+ \frac{\int dx( \sum_{y'} -\frac{1}{2} xx^{T}  \exp(-\frac{1}{2} x^{T} B x +(\alpha+\rho(y))^{T} x +\sum_{(r,j)} \phi_{rj}(y'_r,y'_j)))}{Z(\Theta)}\\
&=\frac{1}{2}xx^{T} + \int \sum_{y'} (-\frac{1}{2} xx^{T} p(x,y';\Theta))\\
&=\frac{1}{2}xx^{T} + \sum_{y'}\int -\frac{1}{2} xx^{T} p(x|y';\Theta) p(y') \\
&=\frac{1}{2}xx^{T} + \sum_{y'}\int -\frac{1}{2} \left(B^{-1} + B^{-1} (\alpha+\rho(y') )(\alpha+ \rho(y')^{T}) B^{-1}\right) p(y') 
\end{align*}
The primary cost is to compute $B^{-1}$ and the sum over the discrete states $y$.
\newline
The computation for the derivatives of $\ell(\Theta)$ with respect to $\rho_{sj}$ and $\phi_{rj}$ are similar.
\begin{align*}
\frac{\partial \ell}{\phi_{rj}(a,b) }&= -1(y_r =a,y_j =b)+\sum_{y'}\int dx 1(y'_r=a,y'_j=b) p(x,y';\Theta)\\
&= -1(y_r =a,y_j =b)+\sum_{y'} 1(y'_r=a,y'_j=b) p(y')
\end{align*}
The gradient requires summing over all discrete states.
\newline
Similarly for $\rho_{sj}(a)$:
\begin{align*}
\frac{\partial \ell}{\rho_{sj}(a) }= -1(y_j = a)x_{s}+\sum_{y'}\int dx (1(y'_j=a) x_s ) p(x',y';\Theta)\\
=-1(y_j = a)x_{s}+\int dx \sum_{y_{\backslash j}'} x_s p(x|y'_{\backslash j},y'_j=a)p(y'_{\backslash j},y'_j = a)
\end{align*}
MLE estimation requires summing over the discrete states to compute the expected sufficient statistics. This may be approximated using using samples $(x,y) \thicksim p(x,y;\Theta)$. The method in the previous section shows that sampling is efficient if $y \thicksim p(y)$ is efficient. This allows us to use  MCMC methods developed for discrete MRF's such as Gibbs sampling.
\subsection{Choosing the Weights}
\label{app:weights}
We first show how to compute $w_{sj}$. 
The gradient of the pseudo-likelihood with respect to a parameter $\rho_{sj} (a) $ is given below
\begin{align}
\frac{\partial \tilde{\ell}}{\partial \rho_{sj}(a)}&= \sum_{i=1}^n  -2\times \indicator{y_{j}^{i}=a} x_{s}^{i}+E_{p_{F}} ( \indicator{y_{j}=a } x_{s}|y_{\backslash j}^{i}, x^{i})  + E_{p_{F}} ( \indicator{y_{j}=a } x_{s}|x^{i}_{\backslash s}, y^{i}) \nonumber \\
&=\sum_{i=1}^n  -2\times \indicator{y_{j}^{i}=a} x_{s}^{i}+x^{i}_{s} p(y_{j}=a)   + \indicator{y_{j}^{i} =a} \mu_{s} \nonumber \\
&= \sum_{i=1}^n \indicator{y_{j}^{i} =a} \left(\hat{\mu}_{s} - x^{i}_{s}\right) +x_{s}^{i} \left( \hat{p}(y_{j} =a) - \indicator {y_{j}^{i}=a} \right) \nonumber\\
&=\sum_{i=1}^n \left(\indicator{y_{j}^{i} =a}-\hat{p}(y_{j}=a) \right) \left(\hat{\mu}_{s} - x^{i}_{s}\right) +\left(x_{s}^{i}-\hat{\mu}_{s}\right) \left( \hat{p}(y_{j} =a) - \indicator {y_{j}^{i}=a} \right)\\
&= \sum_{i=1}^n 2\left(\indicator{y_{j}^{i} =a}-\hat{p}(y_{j}=a) \right) \left(\hat{\mu}_{s} - x^{i}_{s}\right)
\end{align}
Since the subgradient condition includes a variable if $\norm{\frac{\partial \tilde{\ell}}{\partial \rho_{sj}}} > \lambda $, we compute $E\norm{\frac{\partial \tilde{\ell}}{\partial \rho_{sj}}}^{2}$. By independence,
\begin{align}
&E_{p_{F}}\left(\norm{\sum_{i=1}^n 2\left(\indicator{y_{j}^{i} =a}-\hat{p}(y_{j}=a) \right) \left(\hat{\mu}_{s} - x^{i}_{s}\right)}^2 \right )\\
&= 4n E_{p_{F}}\left(\norm{\indicator{y_{j}^{i} =a}-\hat{p}(y_{j}=a)}^2\right) E_{p_{F}}\left( \norm{\hat{\mu}_{s} -x^{i}_{s}}^2\right)\\
&= 4(n-1) p(y_{j}=a) (1-p(y_{j}=a)) \sigma_{s}^{2}
\end{align}
The last line is an equality if we replace the sample means $\hat{p}$ and $\hat{\mu}$ with the true values $p$ and $\mu$. Thus for the entire vector $\rho_{sj}$ we have $E_{p_{F}}\norm{\frac{\partial \tilde{\ell}}{\partial \rho_{sj}}}^{2} =4(n-1) \left(\sum_{a} p(y_{j}=a) (1-p(y_{j}=a)\right) \sigma_{s}^2$. If we let the vector $z$ be the indicator vector of the categorical variable $y_{j}$, and let the vector $p=p(y_{j}=a)$, then $E_{p_{F}}\norm{\frac{\partial \tilde{\ell}}{\partial \rho_{sj}}}^{2} =4(n-1) \sum_{a} p_{a} (1-p_{a}) \sigma^2 = 4(n-1) \mathbf{tr}(\mathbf{cov}(z)) \mathbf{var}(x)$ and $w_{sj} = \sqrt{\sum_{a} p_a (1-p_a ) \sigma_{s}^2}$.

We repeat the computation for $\beta_{st}$. 
\begin{align*}
\frac{\partial \ell}{\partial \beta_{st}} &= \sum_{i=1}^{n}-2 x^{i}_{s} x_{t} +E_{p_F} (x^{i}_{s} x^{i}_{t}|x_{\backslash s},y) +E_{p_F} (x^{i}_{s} x^{i}_{t} |x_{\backslash t},y)\\
&=\sum_{i=1}^{n}-2x^{i}_{s}x^{i}_{t} +\hat{\mu}_{s} x^{i}_{t} +\hat{\mu}_{t} x^{i}_{s}\\
&=\sum_{i=1}^{n} x^{i}_{t}( \hat{\mu_{s}} -x^{i}_{s}) +x^{i}_{s} (\hat{\mu}_{t} -x^{i}_{t})\\
&= \sum_{i=1}^{n}(x^{i}_{t}-\hat{\mu}_{t})(\hat{\mu_{s}} -x^{i}_{s}) +(x^{i}_{s} - \hat{\mu_{s}})(\hat{\mu_{t}} -x^{i}_{t} )\\
&=\sum_{i=1}^{n}2 (x^{i}_{t}-\hat{\mu}_{t})(\hat{\mu_{s}} -x^{i}_{s})
\end{align*}
Thus
\begin{align*}
E&\left( \norm{\sum_{i=1}^{n}2 (x^{i}_{t}-\hat{\mu}_{t})(\hat{\mu_{s}} -x^{i}_{s})}^{2} \right)\\
&=4n E_{p_F} \norm{x_{t} -\hat{\mu_{t}}}^2 E_{p_F} \norm{x_{s} - \hat{\mu}_{s}}^2\\
&= 4(n-1) \sigma_{s}^2 \sigma_{t}^2
\end{align*}
Thus $E_{p_F} \norm{\frac{\partial \ell}{\partial \beta_{st}}}^2 =4(n-1) \sigma_{s}^2 \sigma_{t}^2$ and taking square-roots gives us $w_{st}=\sigma_{s} \sigma_{t}$.\\
We repeat the same computation for $\phi_{rj}$. Let $p_{a} = Pr (y_r =a)$ and $q_{b} =Pr( y_j =b)$. 
\begin{align*}
\frac{\partial \tilde{\ell}}{\partial \phi_{rj}(a,b)} &=
\sum_{i=1}^n -\indicator{y^{i}_{r}=a}\indicator{y^{i}_{j}=b} +E\left(\indicator{y_{r}=a}\indicator{y_{j}=b}|y_{\backslash r},x\right)\\&+E\left(\indicator{y_{r}=a}\indicator{y_{j}=b}|y_{\backslash j}, x\right)\\
&=\sum_{i=1}^n -\indicator{y^{i}_{r}=a}\indicator{y^{i}_{j}=b} +\hat{p}_{a}\indicator{y^{i}_{j}=b}+\hat{q}_{b} \indicator{y^{i}_{r}=a}\\
&=\sum_{i=1}^n\indicator{y^{i}_{j}=b}(\hat{p}_a - \indicator{y^{i}_r =a}) +\indicator{y^{i}_r =a }( \hat{q}_b - \indicator{y^{i}_j =b})\\
&=\sum_{i=1}^n(\indicator{y^{i}_{j}=b}-\hat{q}_b )(\hat{p}_a - \indicator{y^{i}_r =a}) +(\indicator{y^{i}_r =a }-\hat{p}_a )( \hat{q}_b - \indicator{y^{i}_j =b})\\
&= \sum_{i=1}^n 2(\indicator{y^{i}_{j}=b}-\hat{q}_b )(\hat{p}_a - \indicator{y^{i}_r =a})
\end{align*}
Thus we compute 
\begin{align*}
E_{p_F } \norm{\frac{\partial \tilde{\ell}}{\partial \phi_{rj}(a,b)}}^2&=E\left( \norm{\sum_{i=1}^{n} 2(\indicator{y^{i}_{j}=b}-\hat{q}_b )(\hat{p}_a - \indicator{y^{i}_r =a})}^{2} \right)\\
&=4n E_{p_F} \norm{\hat{q}_b - \indicator{y_{j}=b}}^2 E_{p_F} \norm{\hat{p}_a - \indicator{y_{r}=a}}^2\\
&= 4(n-1) q_b (1-q_b ) p_a (1-p_a )
\end{align*}
From this, we see that $E_{p_F } \norm{\frac{\partial \tilde{\ell}}{\partial \phi_{rj}}}^2 = \sum_{a=1}^{L_r} \sum_{b=1}^{L_j} 4(n-1) q_b (1- q_b ) p_{a} (1- p_a )$ and

\noindent$w_{rj}=\sqrt{\sum_{a=1}^{L_r} \sum_{b=1}^{L_j}  q_b (1- q_b ) p_{a} (1- p_a )}$.

\subsection{Model Selection Consistency}
\label{sec:MSC}
One of the difficulties in establishing consistency results for the problem in Equation \eqref{eq:generic_estimator} is due to the non-identifiability of the parameters. $\ell(\Theta)$ is constant with respect to the change of variables $\rho'_{sj} (y_j )=\rho_{sj} (y_j) +c$ and similarly for $\phi$, so we cannot hope to recover $\Theta^\star$. A popular fix for this issue is to drop the last level of $\rho$ and $\phi$, so they are only indicators over $L-1$ levels instead of $L$ levels. This allows for the model to be identifiable, but it results in an asymmetric formulation that treats the last level differently from other levels. Instead, we will maintain the symmetric formulation by introducing constraints. Consider the problem
\begin{equation}
\begin{aligned}
\minimize_{\Theta}\,\ \ell(\Theta) + \lambda \sum_{g \in G} \norm{\Theta_g }_2\\
\text{subject to } C\Theta=0. 
\label{eq:generic_estimator_constrained}
\end{aligned}
\end{equation}
The matrix $C$ constrains the optimization variables such that
\begin{align*}
\sum_{y_j} \rho_{sj}(y_j) = 0 \\
\sum_{y_j } \phi_{rj} ( y_r,y_j ) =0.
\end{align*}
The group regularizer implicitly enforces the same set of constraints, so the optimization problems of Equation \eqref{eq:generic_estimator_constrained} and Equation \eqref{eq:generic_estimator} have the same solutions. For our theoretical results, we will use the constrained formulation of Equation \eqref{eq:generic_estimator_constrained}, since it is identifiable.

We first state some definitions and two assumptions from \cite{lee2013model} that are necessary to present the model selection consistency results. Let $A$ and $I$ represent the active and inactive groups in $\Theta$, so $\Theta^\star _g \neq 0 $ for any $g \in A$ and $\Theta_g ^\star =0$ for any $g \in I$. The sets associated with the active and inactive groups are defined as
\begin{align*}
\mathcal{A} &= \{\Theta\in\mathbb{R}^d : \max_{g\in G}\norm{\Theta_g}_2 \le 1\;\textnormal{and }\norm{\Theta_g}_2 = 0,\,g\in I\} \\
\mathcal{I} &= \{\Theta\in\mathbb{R}^d : \max_{g\in G}\norm{\Theta_g}_2 \le 1\;\textnormal{and }\norm{\Theta_g}_2 = 0,\,g\in A\}.
\end{align*}
Let $M= span(\mathcal{I})^\perp \cap Null(C)$ and $P_M$ be the orthogonal projector onto the subspace $M$. The two assumptions are
\begin{enumerate}
\item Restricted Strong Convexity. We assume that 
\begin{align}
\sup_{v\in M }\frac{v^T \nabla^2 \ell(\Theta) v }{v^T v} \ge m
\label{eq:rsc}
\end{align}
for all $\norm{\Theta-\Theta^\star}_2 \le r$. Since $\nabla^2 \ell(\Theta)$ is lipschitz continuous, the existence of a constant $m$ that satisfies \eqref{eq:rsc} is implied by the pointwise restricted convexity  
$$
\sup_{v\in M }\frac{v^T \nabla^2 \ell(\Theta^\star) v }{v^T v} \ge \tilde{m}.
$$
For convenience, we will use the former.
\item Irrepresentable condition. There exist $\tau\in (0,1)$ such that 
\begin{align}
\sup_{z \in \mathcal{A}}\,V(P_{M^\perp}(\nabla^2 \ell (\Theta^\star)P_M(P_M \nabla^2 \ell (\Theta^\star)P_M)^\dagger P_M z - z)) < 1-\tau,
\label{eq:irrepresentable-condition}
\end{align}
where $V$ is the \emph{infimal convolution} of the gauge $\rho_I$, $\mathcal{I}$ $\rho_I (x) =\inf \{t: t>0, tx \in \mathcal{I}\}$, and $\indicator{Null(C)^\perp}$:
$$
V(z) = \inf_{z=u_1+u_2}\,\{\rho_I(u_1) + \indicator{Null(C)^\perp}(u_2)\}.
$$
\end{enumerate}
Restricted strong convexity  is a standard assumption that ensures the parameter $\Theta$ is uniquely determined by the value of the likelihood function. Without this, there is no hope of accurately estimating $\Theta^\star$. It is only stated over a subspace $M$ which can be much smaller than $\mathbb{R}^d$. The Irrepresentable condition is a more stringent condition. Intuitively, it requires that the active parameter groups not be overly dependent on the inactive parameter groups. Although the exact form of the condition is not enlightening, it is known to be  necessary for model selection consistency in lasso-type problems \citep{zhao2006model,lee2013model} and a common assumption in other works that establish model selection consistency \citep{ravikumar2010,jalali2011,peng2009}.
We also define the constants that appear in the theorem:
\begin{enumerate}
\item Lipschitz constants $L_1$ and $L_2$. Let $\Lambda(\Theta)$ be the log-partition function. $\Lambda(\Theta)$ and $\ell(\Theta)$ are twice continuously differentiable functions, so their gradient and hessian are locally Lipschitz continuous in a ball of radius $r$ around $\Theta^\star$:
\begin{align*}
\norm{\nabla \Lambda(\Theta_1) -\nabla\Lambda(\Theta_2) }_2 \le L_1 \norm{\Theta_1 - \Theta_2}_2,\ \Theta_1, \Theta_2 \in B_r (\Theta^\star)\\
\norm{\nabla^2 \ell(\Theta_1) -\nabla^2\ell(\Theta_2) }_2 \le L_2 \norm{\Theta_1 - \Theta_2}_2,\ \Theta_1, \Theta_2 \in B_r (\Theta^\star)
\end{align*}
\item Let $\bar{\tau}$ satisfy
$$
\sup_{z \in \mathcal{A}\cup \mathcal{I}}\,V(P_{M^\perp}(\nabla^2 \ell (\Theta^\star)P_M(P_M \nabla^2 \ell (\Theta^\star)P_M)^\dagger P_M z - z)) < \bar{\tau}.
$$
$V$ is a continuous function of $z$, so a finite $\bar{\tau}$ exists.
\end{enumerate}

\begin{theorem}
\label{cor:penalized-mle-consistent}
Suppose we are given samples $x^{(1)},\dots,x^{(n)}$ from the mixed model with unknown parameters $\Theta^\star$. If we select 
$$
\lambda = \frac{2\sqrt{256L_1}\bar{\tau}}{\tau}\sqrt{\frac{(\max_{g\in G}|g|)\log|G|}{n}}
$$
and the sample size $n$ is larger than
$$
\max\,\begin{cases} \frac{4096L_1L_2^2\bar{\tau}^2}{m^4\tau^4}\left(2 + \frac{\tau}{\bar{\tau}}\right)^4(\max_{g\in G}|g|)|A|^2\log|G| \\
\frac{2048L_1}{m^2r^2}(2 + \frac{\tau}{\bar{\tau}})^2(\max_{g\in G}|g|)|A|\log|G|,
\end{cases}
$$
then, with probability at least $1-2\big(\max_{g\in G}|g|\big)\exp(-c\lambda^2n)$, the optimal solution to \eqref{eq:generic_estimator} is unique and model selection consistent, 
\begin{enumerate}
\item $\|\hat{\Theta} - \Theta^\star\|_2 \le \frac{4}{m}\left(\frac{\bar{\tau}+1}{2\tau}\right)\sqrt{\frac{256L_1 |A|(\max_{g\in G}|g|)\log|G|}{n}},$
\item $\hat{\Theta}_g = 0,\,g\in I$ and $\hat{\Theta}_g \ne 0\;\text{if }\norm{\Theta^\star_g}_2 > \frac{1}{m}\left(1 + \frac{\tau}{2\bar{\tau}}\right)\sqrt{|A|}\lambda$.
\end{enumerate}

\end{theorem}
\begin{remark}
The same theorem applies to both the maximum likelihood and pseudolikelihood estimators. For the maximum likelihood, the constants can be tightened; everywhere $L_1$ appears can be replaced by $L_1 /128$ and the theorem remains true. However, the values of $\tau, \bar{\tau}, m, L_1,L_2$ are different for the two methods. For the maximum likelihood, the gradient of the log-partition $\nabla\Lambda(\Theta)$ and hessian of the log-likelihood $\nabla^2 \ell(\Theta)$ do not depend on the samples. Thus the constants $\tau, \bar{\tau}, m, L_1,L_2$ are completely determined by $\Theta^\star$ and the likelihood. For the pseudolikelihood, the values of $\tau, \bar{\tau}, m,L_2$ depend on the samples, and the theorem only applies if the assumptions are made on sample quantities; thus, the theorem is less useful in practice when applied to the pseudolikelihood.  This is similar to the situation in \cite{yang2013graphical}, where assumptions are made on sample quantities. 
\end{remark}

    \appendix
\bibliographystyle{apalike}
\bibliography{HowLongLasso,struct_graphmodel}


\end{document}